\documentclass{article}

\usepackage[margin=1in]{geometry} %
\usepackage{natbib}               %

\usepackage[utf8]{inputenc} %
\usepackage[T1]{fontenc}    %
\usepackage{hyperref}       %
\usepackage{url}            %
\usepackage{booktabs}       %
\usepackage{amsfonts}       %
\usepackage{nicefrac}       %
\usepackage{microtype}      %
\usepackage[table]{xcolor}  %

\ifdefined
\else
\newtheorem{theorem}{\bf{Theorem}}
\newtheorem{lemma}{\bf{Lemma}}
\newtheorem{proposition}{\bf{Proposition}}

\newtheorem{definition}{\bf{Definition}}

\def\QED{~\rule[-1pt]{5pt}{5pt}\par\medskip}
\newenvironment{proof}{\emph{Proof.}}{\hfill\QED}

\fi

\newenvironment{theorem*}[1][]{%
    \par\addvspace{\topsep}%
    \noindent{\bfseries Theorem\ifx&#1&\relax\else\ (#1)\fi\quad}\itshape%
    \ignorespaces%
}{%
    \par\addvspace{\topsep}%
}

\newenvironment{lemma*}[1][]{%
    \par\addvspace{\topsep}%
    \noindent{\bfseries Lemma\ifx&#1&\relax\else\ (#1)\fi\quad}\itshape%
    \ignorespaces%
}{%
    \par\addvspace{\topsep}%
}

\newenvironment{proposition*}[1][]{%
    \par\addvspace{\topsep}%
    \noindent{\bfseries Proposition\ifx&#1&\relax\else\ (#1)\fi\quad}\itshape%
    \ignorespaces%
}{%
    \par\addvspace{\topsep}%
}

\newenvironment{corollary*}[1][]{%
    \par\addvspace{\topsep}%
    \noindent{\bfseries Corollary\ifx&#1&\relax\else\ (#1)\fi\quad}\itshape%
    \ignorespaces%
    }{%
    \par\addvspace{\topsep}%
}

\newenvironment{assumption*}[1][]{%
    \par\addvspace{\topsep}%
    \noindent{\bfseries Assumption\ifx&#1&\relax\else\ (#1)\fi\quad}\itshape%
    \ignorespaces%
}{%
    \par\addvspace{\topsep}%
}

\newenvironment{remark*}[1][]{%
    \par\addvspace{\topsep}%
    \noindent{\bfseries Remark\ifx&#1&\relax\else\ (#1)\fi\quad}\itshape%
    \ignorespaces%
}{%
    \par\addvspace{\topsep}%
}

\DeclareMathAlphabet{\altmathcal}{OMS}{cmsy}{m}{n} %

\newcommand{\xmath}[1] {\ensuremath{#1}\xspace}
\newcommand{\blmath}[1] {\xmath{\bm{#1}}}

\newcommand{\vw}{\blmath{w}}

\newcommand{\mc}{\mathcal}

\newcommand{\bb}{\mathbb}

\newcommand{\indicator}[1]{\mathbbm 1_{#1}}

\newcommand{\expect}[2]{\bb E_{#1}\left[ #2 \right]}

\newcommand{\half}{\mbox{$\frac12$}}

\renewcommand{\mathbf}{\boldsymbol}

\usepackage{amsmath}
\usepackage{enumitem}
\usepackage{multirow}
\usepackage{tablefootnote}
\usepackage{bm}
\usepackage{bbm}
\usepackage{amssymb}
\usepackage{xspace}
\usepackage{tikz} %
\usetikzlibrary{arrows.meta, positioning}
\usepackage[colorinlistoftodos]{todonotes} %
\usepackage{pgfplots} %
\usepgfplotslibrary{groupplots}
\pgfplotsset{compat=1.17}
\usepackage{graphicx}
\usepackage{subcaption}
\usepackage{graphicx}
\usepackage{cleveref}
\usepackage{float}
\usepackage{wrapfig}
\usepackage{tcolorbox}

\usepackage{arydshln}   %
\usepackage{booktabs}

\usepackage{pgfplotstable}

\usepackage{mathtools} %
\newlength{\expfigw}\newlength{\expfigh}
\AtBeginDocument{\setlength{\expfigw}{0.6\textwidth}\setlength{\expfigh}{0.45\textwidth}}

\definecolor{greenLight}{RGB}{231, 243, 231}  %
\definecolor{greenDark}{RGB}{204, 229, 204}   %

\definecolor{lb}{RGB}{240, 245, 255}
\definecolor{db}{RGB}{220, 230, 255}
\definecolor{lr}{RGB}{255, 240, 220}
\definecolor{dr}{RGB}{255, 230, 200}
\definecolor{ld}{RGB}{240, 220, 255}
\definecolor{dd}{RGB}{220, 200, 255}

\title{
    \textbf{Retrieval-Based In-Context Learning: \\ A Domain Adaptation Framework}
    }

\author{
  Yilun Zhu\thanks{University of Michigan. \texttt{allanzhu@umich.edu}} \quad
  Naihao Deng\thanks{University of Michigan. \texttt{dnaihao@umich.edu}} \quad
  Yingcong Li\thanks{New Jersey Institute of Technology. \texttt{yingcong.li@njit.edu}} \quad
Naichen Shi\thanks{Northwestern University. \texttt{naichen.shi@northwestern.edu}} \quad
  Clayton Scott\thanks{University of Michigan. \texttt{clayscot@umich.edu}}
}
\date{}

\begin{document}

\maketitle

\begin{abstract}

In-context retrieval (ICR) is a retrieval-based form of in-context
learning (ICL) in which demonstrations are retrieved from a source database based on similarity to the query,
rather than sampled independently.
In this work, we formulate ICR as a type of domain adaptation problem, where the source distribution $P$ of the database may differ from the target distribution $Q$ of the test query-label pair.
We investigate the performance of ICR under a flexible class of distributional shifts that substantially extends prior work \citep{li2024fine,guo2025retrieval}, and establish theoretical guarantees that quantify the benefits and pitfalls of this learning paradigm.
Our theory is verified by experiments on synthetic and language tasks.

\end{abstract}

\section{Introduction}

In-context learning (ICL) is the ability of a pretrained model to predict the label of a query point using input-output demonstrations provided in its context, without updating the model's parameters. In most theoretical models of ICL, these demonstrations are modeled as i.i.d. draws from the same distribution that  governs the query-label pair~\citep{garg2022can,guo2024how}. Throughout this paper, ICL refers to this standard setting with i.i.d. demonstrations.

In this work, we consider retrieval-based ICL, in which the input-output demonstrations are selected as a function of the query point. In particular, given access to a large database of demonstrations, a retrieval algorithm selects examples whose feature representations are similar to that of the query point. It has been observed empirically that such retrieved demonstrations can substantially improve performance over randomly sampled demonstrations~\citep{liu2022makes,rubin2022learning,izacard2022atlasfewshotlearningretrieval,huang2024ravenincontextlearningretrievalaugmented}. Following recent work, we refer to this setting as \emph{in-context retrieval} (ICR)~\citep{luo2024context,guo2025retrieval}.

Prior work has begun to explain the benefits of ICR theoretically, but several important regimes remain unexplored. In particular, existing statistical models~\citep{li2024fine,guo2025retrieval} are restricted to linear-Gaussian settings. Moreover, while these models allow the retrieved demonstrations and the query-label pair to be governed by different marginal distributions, the regression function remains the same: in other words, they model \emph{covariate shift}, but not a \emph{posterior drift} that changes the optimal predictor.

We introduce a nonparametric statistical framework for binary classification that casts ICR as a domain adaptation problem admitting both covariate shift and posterior drift. The source domain governs the corpus of demonstrations, the target domain governs the query-label pair, and the predictor operates on a query-dependent subset of retrieved source examples. Throughout, we treat the pretrained model as frozen and study the statistical effect of the demonstrations selected for its context, abstracting away from
transformer-level mechanisms for implementing in-context learning.

Within this framework, we characterize the performance of ICR in terms of four key quantities: the context size $n$, the corpus size $m$, the degree of posterior drift, and distributional properties of the domains (in particular, the smoothness of the class posterior). Under a nearest-neighbor model of retrieval and prediction, we establish a finite-$(m,n)$ excess risk bound for ICR.
This bound yields two main conclusions. First, the \emph{benefit of retrieval}: given a sufficiently large corpus size $m$,
the ICR bound decays faster in $n$ than the ICL rate.
Second, the \emph{price of posterior drift}: when the source and target posteriors disagree, the ICR excess risk contains an irreducible bias that no learning rule with access only to the source corpus can uniformly avoid. Consequently, ICR can be asymptotically inconsistent, whereas ICL remains consistent.

Thus, ICR is preferable under a few-shot regime (small $n$) with a sufficiently large corpus and limited posterior drift; otherwise, using i.i.d.\ target demonstrations is the safer choice.
We support these findings with experiments on synthetic data and a sentiment-analysis task with annotator disagreement.

\section{Related work}

\paragraph{In-context learning.}

In-context learning (ICL) became prominent with GPT-3~\citep{brown2020language} and is
now a standard mechanism for adapting language models at inference time
~\citep{dong-etal-2024-survey,agarwal2024many}.
Beyond natural language, this success has also extended to other modalities, such as tabular data, where tabular foundation models \citep{hollmann2022tabpfn,hollmann2025accurate,qu2025tabicl,ma2025tabdpt} represent a compelling application of ICL.
A growing theoretical literature
analyzes ICL through the behavior of transformers, implicit inference, and task generalization~\citep{ahn2023transformers,zhang2024trained,xie2021explanation,li2023transformers,yang2024context,bai2023transformers}, and
related work has also examined more challenging regimes such as
out-of-distribution tasks, test-time training, and noisy in-context demonstrations
~\citep{kwon2025out,gozeten2025test,cheng2025exploring}.
These studies provide useful tools for understanding ICL with i.i.d. demonstrations, but they do not study the retrieval-based setup.

\paragraph{Retrieval-augmented generation.}
Retrieval-augmented generation (RAG) has become a central paradigm for enriching
language models with external knowledge by conditioning generation on retrieved
text~\citep{lewis2020retrieval,izacard2021leveraging,borgeaud2022improving, es-etal-2024-ragas, gupta2024survey}, and has seen applications beyond language-specific tasks \citep{chen2025empowering}.
Recent work has improved many components of RAG systems~\citep{Chen2024BGEMM, li2024llama2vec, xu2024recomp, asai2023selfrag, li2024chainofknowledge, xiang2024certifiably}.
Most of the literature on RAG is empirical and method-driven.
In this work, we provide theoretical insight into RAG by analyzing ICR, where query-dependent in-context demonstrations are retrieved.

\paragraph{In-context retrieval.}

The basic motivation for in-context retrieval (ICR) is that the empirical performance of ICL depends heavily on the choice of in-context examples~\citep{liu2022makes}. In particular, retrieval-based prompt selection has been shown to outperform randomly sampled demonstrations~\citep{rubin2022learning}.
This idea has been broadly validated across a range of settings \citep{luo2024context}, including %
in-domain labeled data~\citep{izacard2022atlasfewshotlearningretrieval,huang2024ravenincontextlearningretrievalaugmented,ye2023compositional},
cross-domain or cross-lingual retrieval~\citep{cheng2023uprise,shi2022xricl},
and LLM-generated demonstrations~\citep{zhang2022automatic,li2023mot}.

Most prior work on ICR is empirical, with limited theoretical understanding of when and why retrieval helps. Two recent works begin to fill this gap: \citet{li2024fine} model demonstrations whose features are $\alpha$-correlated with the query and show that higher correlation yields improved ICR performance, while \citet{guo2025retrieval} analyze the covariate shift between i.i.d.\ training-time demonstrations and retrieved test-time demonstrations. Both works study linear-regression settings where source and target share the same regression function.
In this work, in contrast, we develop a nonparametric framework for binary classification that admits general domain shift and characterizes when retrieval helps versus hurts.

Additional connections to the literature are made below, including to domain adaptation theory.

\section{A statistical framework for in-context retrieval}
\label{sec:icr_framework}

This section develops the statistical models underlying our analysis. We cast standard in-context learning as conventional classification, and in-context retrieval as domain adaptation with a retrieval-based learner.
We then revisit the two closest prior theory works \citep{li2024fine,guo2025retrieval} within this general framework.

\subsection{
    ICL
    as conventional classification}

We begin by recalling the standard statistical learning setup for binary classification, with feature space $\mc{X}\subseteq\mathbb{R}^d$ and label space $\mc{Y}=\{0,1\}$. Let $D$ be a probability distribution on $\mc{X}\times\mc{Y}$, referred to as a \emph{domain}. The $X$-marginal is denoted by $D_X$ and the class posterior distribution $Y|X$ is fully characterized by the
\emph{posterior probability function} (or \emph{posterior})
\begin{align*}
    \eta_D(x) := D(Y=1\mid X=x).
\end{align*}
A \emph{classifier} is a function $f:\mc{X}\to\{0,1\}$, whose \emph{risk} (with respect to $D$) is
\begin{align*}
    R_D(f) := \mathbb{P}_{(X,Y)\sim D}\bigl(f(X)\neq Y\bigr).
\end{align*}
A Bayes classifier $f_D^*$ is one that minimizes $R_D$, and the canonical Bayes classifier is
\begin{align*}
    f_D^*(x) := \indicator{\eta_D(x)\ge\half}.
\end{align*}
The \emph{excess risk} of $f$ is $\mc{E}_D(f):=R_D(f)-R_D^*$,
where $R_D^*:=R_D(f_D^*)$ is the Bayes risk.

A \emph{learner} is a function $\hat f:\mc{X}\times(\mc{X}\times\mc{Y})^n\to\{0,1\}$ that takes as input a
query
point $x$ and labeled training data $S_n = \{ (X_i,Y_i) \}_{i=1}^n$ and outputs a label. For fixed $S_n$, the map $f_{S_n} := \hat f(\cdot, S_n)$ is a conventional classifier: the \emph{classifier induced by} $S_n$. The performance of the learner is measured by the expected excess risk
\begin{equation}
    \label{eqn:expexcessrisk}
    \mc{E}_D(\hat{f}) := \expect{S_n\sim D^{\otimes n}}{R_D(f_{S_n}) - R_D^*},
\end{equation}
where
$S_n\sim D^{\otimes n}$ indicates that $(X_i,Y_i) \overset{\mathrm{i.i.d.}}{\sim} D$.

Assuming a frozen, pretrained model, we view standard ICL as conventional classification where the learner $\hat f$ is the frozen model and the training data $S_n$ are the in-context demonstrations. Throughout, the domain governing the query and the ICL demonstrations is called the \emph{target} domain and denoted $Q$, and we denote the test query-label pair by $(X_q,Y_q)$; the data-generation process of ICL is:
\begin{align*}
    \text{Demonstrations:} \ \
        & S_n=\{(X_i,Y_i)\}_{i=1}^n,
        \ (X_i,Y_i)\overset{\mathrm{i.i.d.}}{\sim} Q, \\
    \text{Query-label:} \ \
        & (X_q,Y_q)\sim Q,\quad\text{independent of }S_n.
\end{align*}
In particular, the performance of $\hat f$ is measured by \Cref{eqn:expexcessrisk} with $D=Q$, and we write $\mc{E}_Q^{\mathrm{ICL}}(\hat{f}) := \mc{E}_Q(\hat{f})$ for the expected excess risk of ICL.

\subsection{
    ICR
    as retrieval-based domain adaptation}

\emph{Domain adaptation} is a variation on standard classification in which there are two probability distributions $P$ and $Q$ on $\mc{X}\times\mc{Y}$, referred to as the \emph{source} and \emph{target} domains. In domain adaptation, the training data $S_n$ come from $P$, but the test points follow $Q$. A learner is still a function $\hat f:\mc{X}\times(\mc{X}\times\mc{Y})^n\to\{0,1\}$, but now its performance is measured with respect to
\begin{align*}
    \mc{E}_{P,Q}(\hat{f}) := \expect{S_n\sim P^{\otimes n}}{R_Q(f_{S_n}) - R_Q^*} .
\end{align*}
In our setting, the target domain $Q$ governs the test query-label pair $(X_q,Y_q)$, while $P$ governs the source corpus from which demonstrations are retrieved.

A \emph{domain shift} occurs when $P\neq Q$. We focus on two particular types of domain shift, which can be expressed in terms of the feature marginals $P_X$ and $Q_X$ and class posterior distributions $P_{Y|X}$ and $Q_{Y|X}$ of the source and target. In \emph{covariate shift}, $P_X\neq Q_X$ \citep{sugiyama2007covariate}, while in \emph{posterior drift}, $P_{Y\mid X}\neq Q_{Y\mid X}$ \citep{scott2019generalized,cai2021transfer,maity2021linear, zhu2024label, kirev2025pac}.
\footnote{Note that posterior drift is similar to the notion of \emph{concept drift}, although a concept typically connotes a deterministic classifier. The term \emph{conditional shift} has also been applied, although we find posterior drift more descriptive.}
These two types of domain shift may co-occur.

We view ICR as domain adaptation with a retrieval-based learner. Given $m \ge n$, an \emph{$(m,n)$-retriever} is a function $T_n: \mc{X} \times (\mc{X}\times\mc{Y})^m \to (\mc{X}\times\mc{Y})^n$ that takes as input a query point $x_q$ and a source corpus of demonstrations $\mathcal{D}_m = \{ (X_i,Y_i) \}_{i=1}^m$, and selects the $n$ demonstrations $S_n(x_q) := T_n(x_q,\mathcal{D}_m)$ in $\mathcal{D}_m$ whose feature vectors are most similar to $x_q$ (e.g., in $\ell_2$ distance). A \emph{retrieval-based learner} $\hat f_T$ is then a function of the form
\begin{align*}
    \hat f_T(x, \mathcal{D}_m) = \hat f\bigl(x, T_n(x,\mathcal{D}_m)\bigr) = \hat f(x, S_n(x)),
\end{align*}
where $\hat f$ is any conventional learner as described above. Note that for a fixed corpus $\mathcal{D}_m$, the map $f_{\mathcal{D}_m} := \hat f_T(\cdot, \mathcal{D}_m)$ is the classifier induced by the corpus, analogous to $f_{S_n}$ above. Although the corpus has size $m$, this classifier uses only $n \ll m$ demonstrations as context for any given query. In summary, the data-generation process of ICR is:
\begin{align*}
    \text{Source corpus:} \quad
        & \mathcal{D}_m=\{(X_i,Y_i)\}_{i=1}^m,\quad
          (X_i,Y_i)\overset{\mathrm{i.i.d.}}{\sim} P, \\
    \text{Query-label:} \quad
        & (X_q,Y_q)\sim Q,\quad\text{independent of }\mathcal{D}_m, \\
    \text{Retrieve:} \quad
        & S_n(X_q) = T_n(X_q,\mathcal{D}_m)\subseteq\mathcal{D}_m.
\end{align*}

A retrieval-based learner is itself a learner, but its ``training data'' is the source corpus $\mathcal{D}_m$ rather than $S_n$. Its expected excess risk therefore integrates out $\mathcal{D}_m$:
\begin{align*}
    \mc{E}_Q^{\mathrm{ICR}}(\hat{f})
    := \mc{E}_{P,Q}(\hat{f}_T)
    = \expect{\mathcal{D}_m\sim P^{\otimes m}}{R_Q(f_{\mathcal{D}_m}) - R_Q^*}.
\end{align*}

In summary,
ICL and ICR differ in where the demonstrations come from ($Q$ versus $P$) and how they are selected (i.i.d.\ sampling versus retrieval). Figure~\ref{fig:icl_icr} contrasts the two learning paradigms of ICL and ICR.

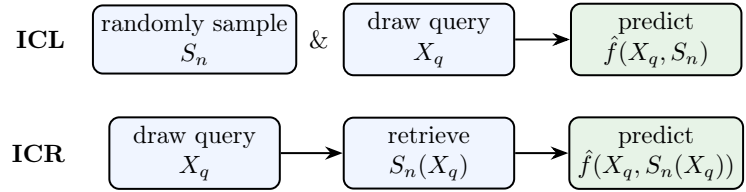
\begin{figure}[htbp]
\centering
\resizebox{0.6\linewidth}{!}{%
\begin{tikzpicture}[
  every node/.style = {font=\normalsize},
  flowbox/.style = {draw, rounded corners=4pt, minimum width=2.4cm,
                    minimum height=0.65cm, align=center, thick, fill=lb},
  pred/.style = {draw, rounded corners=4pt, minimum width=2.4cm,
                 minimum height=0.65cm, align=center, thick, fill=greenLight},
  arr/.style = {-{Stealth[length=2.5mm]}, thick},
]

\node[font=\normalsize\bfseries, anchor=east] at (-0.3, 0.8) {ICL};
\node[flowbox] (icl_s) at (1.4, 0.8) {randomly sample \\ $S_n$};
\node[flowbox] (icl_x) at (4.7, 0.8) {draw query\\$X_q$};
\node[pred] (icl_f) at (7.9, 0.8) {predict\\$\hat{f}(X_q,S_n)$};
\node[font=\normalsize] at (3.15, 0.8) {\&};
\draw[arr] (icl_x) -- (icl_f);

\node[font=\normalsize\bfseries, anchor=east] at (-0.3, -0.8) {ICR};
\node[flowbox] (icr_x) at (1.4, -0.8) {draw query\\$X_q$};
\node[flowbox] (icr_s) at (4.7, -0.8) {retrieve \\ $S_n(X_q)$};
\node[pred] (icr_f) at (7.9, -0.8) {predict\\$\hat{f}(X_q,S_n(X_q))$};
\draw[arr] (icr_x) -- (icr_s);
\draw[arr] (icr_s) -- (icr_f);

\end{tikzpicture}
}%
\caption{
Standard ICL vs ICR.
  }
\label{fig:icl_icr}
\end{figure}

\subsection{Connections to prior ICR and domain adaptation theory}

The two prior theoretical works on ICR most closely related to ours are \citet{li2024fine} and
\citet{guo2025retrieval}.
However, both consider Gaussian linear regression ($\mc{Y}=\mathbb{R}$) as their data model, whereas our analysis
targets binary classification ($\mc{Y}=\{0,1\}$).
Despite the difference in statistical model, each can be situated within our
$(P,Q)$ framework.

\paragraph{\citet{li2024fine}.}
Their data model is Gaussian linear regression with a single latent weight vector
$\vw$ shared between the query and all demonstrations:
\begin{gather*}
    X_q \sim \mathcal{N}(0, I_d), \qquad
    \vw \sim \mathcal{N}(0, I_d), \qquad
    Y_q = \vw^\top X_q + \varepsilon_q, \\
    X_i \mid X_q \sim \mathcal{N}(\alpha X_q, (1-\alpha^2) I_d), \qquad
    Y_i = \vw^\top X_i + \varepsilon_i,
    \qquad i=1,\dots,n,
\end{gather*}
where $\alpha \in [0,1]$ controls retrieval similarity and $\varepsilon_q,\varepsilon_i$'s are random noise.
In our $(P,Q)$ language:
$Q = \mathcal{N}(0,I_d)\otimes\mathcal{N}(\vw^\top\cdot,\sigma^2)$
(i.e.\ $Q_X=\mathcal{N}(0,I_d)$ and $Q_{Y\mid X=x}=\mathcal{N}(\vw^\top x,\sigma^2)$).
Marginalizing $X_i$ over $X_q$ gives $P_X=\mathcal{N}(0,I_d)=Q_X$, and the
posterior is the same: $P_{Y\mid X}=Q_{Y\mid X}$.
Thus $P=Q$;
so their model captures the effect of retrieval, but not a domain shift.

\paragraph{\citet{guo2025retrieval}.}
They also use a Gaussian linear model with a shared latent weight vector, but focus on
the covariate shift induced by retrieval offsets:
\begin{gather*}
    X_q \sim \mathcal{N}(0, I_d), \qquad
    \vw \sim \mathcal{N}(0, I_d), \qquad
    Y_q = \vw^\top X_q + \varepsilon_q, \\
    X_i = X_q + r_i, \qquad
    Y_i = \vw^\top X_i + \varepsilon_i,
    \qquad i=1,\dots,n,
\end{gather*}
where $r_i$ is the retrieval residual and $\varepsilon_q,\varepsilon_i$'s are random noise.
In our language: $Q_X=\mathcal{N}(0,I_d)$, while the marginal of $X_i$ is shifted
by $r_i$, so $P_X\neq Q_X$ in general.
Nevertheless, the regression function is shared,
$\mathbb{E}[Y\mid X=x]=\vw^\top x$ under $P$ and $Q$; the retrieved labels may carry a different noise variance, which leaves the optimal predictor unchanged.
The effect they study is therefore \emph{covariate shift} ($P_X\neq Q_X$) together with label noise,
not a change in the optimal predictor.

\paragraph{Limitations of prior ICR theory.}
Although both papers study ICR, they focus primarily on the relevance between the features $X_i$ and $X_q$, and assume that demonstrations and query share the same regression function.
In domain adaptation terms, there is \emph{no posterior drift} between $P$ and $Q$ that alters the optimal predictor.
In practice, however,
the source corpus may come from a different distribution with
$P_{Y\mid X}\neq Q_{Y\mid X}$ (where $f^*_P \neq f^*_Q$).
In addition, both papers adopt isotropic Gaussian models for $P$ and $Q$, whereas our focus is on nonparametric classification.

\paragraph{Relation to existing domain adaptation theory.}
The primary technical tool we use to model domain adaptation is the notion of relative signal strength (RSS) introduced by \citet{zhu2024label}. RSS extends the prior model of \citet{cai2021transfer} by allowing the source and target Bayes classifiers to differ on a set of positive measure. Similarly, \citet{fan2025robust} quantify source unreliability through an \emph{ambiguity level} that weights the target mass of unreliable regions by the target margin. As seen below in Lemma~\ref{lem:fundamental-reduction}, integrating the pointwise risk inequality implied by RSS gives a risk-level inequality in the spirit of the transfer exponent of \citet{hanneke2019value} and the modulus of transfer of \citet{hanneke2025adaptive}.
The RSS-based framework differs substantially from discrepancy-based approaches to domain adaptation \citep{ben2006analysis,mansour2009domain}, which characterize domain relatedness through global distributional discrepancies rather than pointwise relationships between the source and target posteriors. Finally, none of the aforementioned works consider \emph{retrieval-based learners}, the defining feature of ICR.

\section{Performance guarantee for in-context retrieval}
Our analysis proceeds in two steps. First, we convert the target-domain excess risk to a source-domain excess risk and establish a lower bound showing that the bias induced by posterior drift is unavoidable; these results are independent of any particular learning algorithm.
Second, we derive rates of convergence for a nearest-neighbor type retrieval-based learner for ICR, and contrast with the rates of a comparable nearest-neighbor learner for ICL.

\subsection{Relative signal strength and CSPD}
\label{subsec:rss-model-class}

The central challenge of ICR is that retrieved demonstrations come from $P$,
whose posterior $\eta_P$ may differ from the target $\eta_Q$.
We quantify this mismatch via the concept of relative signal strength of \citet{zhu2024label}.
For $D \in \{P,Q\}$, let the support of $D_X$ be denoted $\Omega_D$.
\vspace{-0.5em}

\begin{definition}[Relative Signal Strength (RSS)]
\label{def:rss}
For domains $P$ and $Q$,
define the
\emph{relative signal strength} at $x\in\mc{X}$ as
\begin{equation}
\label{eqn:rss}
    \mc{M}(x;P,Q)
    = \begin{cases}
        \max\!\left\{\frac{\eta_P(x)-\half}{\eta_Q(x)-\half},\,0\right\},
        & x \in \Omega_Q \cap \Omega_P,\\[4pt]
        0, & \text{otherwise}.
      \end{cases}
\end{equation}
On $\Omega_Q \cap \Omega_P$, we use the convention $\mc{M}(x;P,Q):=+\infty$ whenever $\eta_Q(x)=\half$.

For $\kappa>0$, define the \emph{high-RSS region}
$\mc{A}_\kappa(P, Q):=\{x\in\mc{X}:\mc{M}(x;P,Q)\ge\kappa\}$.
\end{definition}
$\mc{A}_\kappa(P,Q)$ is a subset of the feature space $\mc{X}$ where source data is reliable for predicting target labels; larger $\kappa$ means greater reliability.
If $(\eta_P(x)-\half)(\eta_Q(x)-\half)<0$, i.e., source and target favor opposite labels, then $\mc{M}(x;P,Q)=0$, so $x\in\mc{A}_\kappa^c$ for every $\kappa>0$.
Outside the support of $Q_{X}$, the RSS is defined to be zero. This is a natural convention in the context of ICR, where a large, diverse source database may cover regions beyond the target domain.

We combine RSS with a covariate shift assumption: $Q_X\ll P_X$ and
\begin{align}
    \lambda^{-1} \;\leq\; \frac{dQ_X}{dP_X}(x) \;\leq\; \lambda,
    \qquad \text{for } P_X\text{-a.e. } x\in\Omega_Q,
    \label{eq:density-ratio}
\end{align}
for some $\lambda\ge 1$.
That is, the source and target $X$-marginals are within a factor of $\lambda$ of each other at every point of the target support $\Omega_Q$, except possibly on a set that the source never samples from.
A sufficient condition for this to hold is that both $P$ and $Q$ have Lebesgue densities that are bounded above and below, and the support of the source distribution covers that of the target.

\begin{definition}[CSPD distribution class]
\label{def:model-class}

For $\varepsilon\in[0,1]$, $\kappa>0$, and $\lambda\ge 1$, let
\[
    \Pi(\varepsilon,\kappa,\lambda)
    :=
    \Bigl\{(P,Q) :\;
      Q_X\!\bigl(\mc{A}_\kappa(P,Q)\bigr)\ge 1-\varepsilon
      \;\text{ and }\;
      Q_X\ll P_X,\;\;\lambda^{-1}\le\tfrac{dQ_X}{dP_X}(x)\le\lambda
      \;\text{ for } P_X\text{-a.e. } x\in\Omega_Q\Bigr\}.
\]
\end{definition}
We refer to $\Pi(\varepsilon,\kappa,\lambda)$ as the \emph{Covariate Shift with Posterior Drift} (CSPD) distribution class.

Here $\varepsilon$ upper-bounds the target mass of the low-RSS region (where source labels may be unreliable); $\kappa$ is the RSS threshold; $\lambda$ bounds the covariate shift between $P_X$ and $Q_X$.

The following lemma relates the target excess risk to the source excess risk.

\begin{lemma}[Risk transfer]
\label{lem:fundamental-reduction}
If $(P,Q)\in\Pi(\varepsilon,\kappa,\lambda)$, then for any classifier $f:\mc{X}\to\{0,1\}$,
\begin{align*}
    R_Q(f) - R_Q^*
    \;\leq\;
    \varepsilon + \frac{\lambda}{\kappa}\,\bigl(R_P(f) - R_P^*\bigr).
\end{align*}
\end{lemma}

The full proof is in Appendix~\ref{proof:fundamental-reduction}. Lemma~\ref{lem:fundamental-reduction} shows that the target excess risk is controlled by the source excess risk, up to a factor $\lambda/\kappa$ and an additive error $\varepsilon$.
This fundamental reduction allows us to bound  the expected risk of ICR
in terms of the $P$-excess risk of the learner:
\[
\mc{E}_Q^{\mathrm{ICR}}(\hat{f})
    = \expect{\mathcal{D}_m\sim P^{\otimes m}}{R_Q(f_{\mathcal{D}_m}) - R_Q^*}
    \leq \varepsilon + \frac{\lambda}{\kappa}\,\expect{\mathcal{D}_m\sim P^{\otimes m}}{R_P(f_{\mathcal{D}_m}) - R_P^*}.
\]

We next show that an error proportional to $\varepsilon$ cannot be uniformly avoided under posterior drift.

\begin{proposition}[Lower bound for ICR]
\label{prop:icr-lower-bound-simple}
Assume $\mathcal X$ contains at least two distinct points. For any $\varepsilon\in[0,1]$, $\kappa>0$, $\lambda\ge 1$, and any learning rule $\hat{f}$ that observes a corpus $\mc{D}_m\sim P^{\otimes m}$ and produces a classifier $\hat{f}(\cdot)$,
\[
    \inf_{\hat{f}}\;\sup_{(P,Q)\in\Pi(\varepsilon,\kappa,\lambda)}
    \mc{E}_Q^{\mathrm{ICR}}(\hat{f})
    \;\ge\; \frac{\varepsilon}{2}.
\]
\end{proposition}

Full proof in Appendix~\ref{proof:icr-lower-bound-simple}. Proposition~\ref{prop:icr-lower-bound-simple} establishes the cost of posterior drift: no ICR learner can uniformly eliminate an additive bias of order $\varepsilon$ over $\Pi(\varepsilon,\kappa,\lambda)$. The next subsection turns to the benefit of retrieval, showing that it allows the source-side excess risk to decay faster with respect to the context size $n$.

\subsection{Finite sample analysis of ICR}
\label{subsec:rates}

The results of the previous subsection are learner-agnostic: Lemma~\ref{lem:fundamental-reduction} holds for any classifier, and Proposition~\ref{prop:icr-lower-bound-simple} for any learner. To get a precise finite-sample analysis, we must further specify (1) distributional assumptions on $(P,Q)$ and (2) the functional form of the in-context learner. For the latter, we adopt a nearest-neighbor majority-vote model, for two reasons.
First, predictions are formed directly from the demonstrations in the context without parameter updates, and transformer architectures are sufficiently expressive to emulate nearest-neighbor prediction ~\citep{li2024one,agarwal2024many}.
Second, nearest-neighbor prediction mirrors the structure of ICR itself: it is precisely a ``retrieve-then-vote'' procedure, making it a natural theoretical model for retrieval-based prediction.

Let $Y_{(j)}(x, s)$ denote the label of the $j$-th nearest neighbor of $x$
in dataset $s$, where distances are measured by the metric $\rho$ on $\mc{X}$. Ties are broken by the original sample order, independently of labels. Recall that for ICL, $S_n=\{ (X_i,Y_i) \}_{i=1}^n \sim Q^{\otimes n}$ contains i.i.d. demonstrations from the target; and for ICR, $S_n$ is retrieved from a source corpus $\mathcal{D}_m \sim P^{\otimes m}$. We define the corresponding learners as follows.
\begin{align*}
    \textbf{ICL learner:}\quad &
    \hat{f}_{\mathrm{ICL}}(X_q, S_n)
    =
    \begin{cases}
        1, & %
        \frac{1}{k}\sum_{j=1}^{k} Y_{(j)}(X_q,S_n) \geq \tfrac{1}{2}, \\[4pt]
        0, & \text{otherwise},
    \end{cases}\\[6pt]
    \textbf{ICR learner:}\quad &
    \hat{f}_{\mathrm{ICR}}(X_q, S_n(X_q))
    =
    \begin{cases}
        1, & %
        \frac{1}{n}\sum_{j=1}^{n} Y_{(j)}(X_q,\mc{D}_m) \geq \tfrac{1}{2}, \\[4pt]
        0, & \text{otherwise},
    \end{cases}
\end{align*}
Specifically, the ICL learner applies a $k$-nearest-neighbor majority vote ($k \leq n$) among $S_n$, with the rate-optimal choice $k \asymp n^{2\alpha/(2\alpha+1)}$ \citep{chaudhuri2014rates}. The ICR learner instead averages over all $n$ demonstrations: the retriever has already selected the $n$ corpus examples closest to $X_q$.

We now bound the excess risks of both learners under
distributional assumptions on $P$ and $Q$, and combine with Lemma~\ref{lem:fundamental-reduction}
to obtain the ICL and ICR rates.

\noindent\textbf{Assumptions on the target $Q$. }
We impose the following conditions on $Q$.
\begin{itemize}
  \item \emph{$(\alpha,L)$-smoothness}: for some $\alpha,L>0$,
    \[
        \bigl|\eta_Q(B(x,r))-\eta_Q(x)\bigr|
        \;\leq\; L\,Q_X\!\bigl(B^o(x,r)\bigr)^\alpha
    \]
    for all $x\in\mathrm{supp}(Q_X)$ and $r>0$,
    where $\eta_Q(A):=Q_X(A)^{-1}\!\int_A\eta_Q\,dQ_X$, and $B(x,r):=\{x'\in\mc{X}:\rho(x,x')\leq r\}$ and $B^o(x,r):=\{x'\in\mc{X}:\rho(x,x')< r\}$ are the closed and open balls of radius $r$ around point $x$.
  \item \emph{$\beta$-margin condition}: for some $\beta\ge 0$ and $C_\beta>0$,
  \begin{align*}
    Q_X\!\bigl(|\eta_Q(X)-\half|<t\bigr)\leq C_\beta\,t^\beta \quad \text{for all } t>0.
  \end{align*}
\end{itemize}

The $(\alpha, L)$-smoothness condition \citep{chaudhuri2014rates} characterizes how fast $\eta_Q$ varies over local neighborhoods, with the variation controlled by the probability mass $Q_X(B^o(x,r))$. The larger the $\alpha$, the smoother. It generalizes the H\"{o}lder smoothness condition: if $\eta_Q$ is $\alpha_H$-H\"{o}lder continuous and $Q_X$ has a density bounded below on a bounded support with regular boundary (e.g., a cube), then $\eta_Q$ is $(\alpha_H/d, L')$-smooth for some constant $L'$ \citep{chaudhuri2014rates}.
The $\beta$-margin condition \citep{audibert2007fast} controls how much probability mass lies near the decision boundary $\eta_Q(x)=\tfrac{1}{2}$, with larger $\beta$ implying a smaller ambiguous region and hence an easier classification problem.

\noindent\textbf{Assumption on the source $P$. }
For the finite-corpus analysis, we further assume $\eta_P$ is $(\alpha_P,L_P)$-smooth.

The following lemma derives the
margin condition on source distribution $P$ from the RSS condition.

\begin{lemma}[Margin transfer via RSS]
\label{lem:margin-transfer}
If $(P,Q)\in\Pi(\varepsilon,\kappa,\lambda)$ and $\eta_Q$ satisfies the
$\beta$-margin condition with constant $C_\beta$, then
\begin{equation}
    P_X\!\bigl( \{|\eta_P(X)-\half|<t\}\cap A_\kappa \bigr)
    \;\leq\;
    \lambda C_\beta\kappa^{-\beta}\,t^\beta
    \quad\forall\,t>0.
    \label{eq:margin-transfer}
\end{equation}
where $A_\kappa:=\mc{A}_\kappa(P,Q)$ is the high-RSS region.
\end{lemma}

The full proof is in Appendix~\ref{proof:margin-transfer}.
We now state the main result: excess risk bounds of ICL and ICR.

\newpage

\begin{theorem}[Excess risks of ICL and ICR]
\label{thm:icl-icr-rates}
\label{thm:icl-bound}
\label{thm:icr-bound}
Suppose $\eta_Q$ satisfies $(\alpha,L)$-smoothness and the $\beta$-margin condition with constant $C_\beta$. Then  $\forall n \geq 1$,
\begin{enumerate}
    \item[(i)] The ICL excess risk satisfies
    \[
        \mc{E}_Q^{\mathrm{ICL}}(\hat{f}_{\mathrm{ICL}})
        \;\leq\;
        C_o\,n^{-\frac{\beta+1}{2}\cdot\frac{2\alpha}{2\alpha+1}}.
    \]

    \item[(ii)] If additionally $(P,Q)\in\Pi(\varepsilon,\kappa,\lambda)$
    and $\eta_P$ satisfies $(\alpha_P,L_P)$-smoothness, then
    $\forall m\geq 2n$ and $n\geq n_0$,
    \begin{equation}
        \mc{E}_Q^{\mathrm{ICR}}(\hat{f}_{\mathrm{ICR}})
        \;\leq\; \varepsilon
        + \frac{3C_\beta\lambda^2}{\kappa^{\beta+1}}
        \max\!\left\{
            2L_P\!\left(\tfrac{2n}{m}\right)^{\alpha_P},
            \sqrt{\tfrac{8(\beta+2)}{n}}
        \right\}^{\beta+1}.
        \label{eq:icr-finite-m-main}
    \end{equation}
    In particular, in the large-corpus regime where
    \begin{equation}
        m \;\geq\; m_\star(n)
        :=2n\left(
            2L_P\sqrt{\frac{n}{8(\beta+2)}}
        \right)^{1/\alpha_P}
        = C_\star n^{1+1/(2\alpha_P)},
        \label{eq:icr-saturation}
    \end{equation}
    then
    \[
        \mc{E}_Q^{\mathrm{ICR}}(\hat{f}_{\mathrm{ICR}})
        \;\leq\; \varepsilon
        +\frac{C_o'\lambda^2}{\kappa^{\beta+1}}
        n^{-\frac{\beta+1}{2}}.
    \]
\end{enumerate}
Here $C_o=C_o(\alpha,\beta,L,C_\beta)$,
$C_o'=3C_\beta\bigl(8(\beta+2)\bigr)^{(\beta+1)/2}$,
$C_\star=2(2L_P/\sqrt{8(\beta+2)})^{1/\alpha_P}$, and
$n_0=n_0(\beta,C_\beta,\kappa,\lambda)$.
\end{theorem}

This is the main result of the paper; full proofs are deferred to Appendices~\ref{proof:icl-bound} and~\ref{proof:icr-bound}.
Furthermore, we show in Appendix~\ref{proof:icr-lower-bound-simple} that the lower bound of Proposition \ref{prop:icr-lower-bound-simple} also holds in the setting of Theorem \ref{thm:icl-icr-rates}. Thus, this result establishes a fundamental bias-variance trade-off: provided the source corpus is sufficiently large, the ICR bound decays strictly faster in $n$ than the ICL bound, but over the CSPD class, no ICR learner can uniformly avoid an
irreducible bias of order $\varepsilon$, while ICL remains consistent.

\paragraph{Comparison of ICL and ICR rates.}
The two rates in Theorem~\ref{thm:icl-icr-rates} reflect a trade-off between sample efficiency and label correctness. In ICL, the labels are drawn from the correct posterior $\eta_Q$, but only demonstrations near the query are informative, so the vote uses just the $k \asymp n^{2\alpha/(2\alpha+1)}\ll n$ nearest ones.
In ICR, the retriever selects from a corpus of size $m\gg n$, so all $n$ demonstrations in the context are close to the query and all $n$ labels are used in the vote, yielding the faster rate.
The price is that these labels are drawn from the source posterior $\eta_P$ rather than $\eta_Q$, and the irreducible bias $\varepsilon$ quantifies how wrong that posterior may be.
The following special cases make the effect of domain shift transparent; throughout, suppose $n\ge n_0$ and $m\ge\max\{2n,m_\star(n)\}$.

\begin{itemize}
    \item \textbf{No shift} ($P=Q$, hence $\varepsilon=0,\kappa=\lambda=1$):
    $$\mc{E}_Q^{\mathrm{ICR}}(\hat{f}_{\mathrm{ICR}})
    \le
    C_o'n^{-\frac{\beta+1}{2}},$$
    This isolates the statistical benefit of retrieval: the ICR guarantee is faster than the ICL guarantee $C_o\,n^{-\alpha(\beta+1)/(2\alpha+1)}$ for every $\alpha>0$. The advantage is most pronounced for non-smooth $Q$: the ICL exponent approaches zero as $\alpha\to0$, whereas the ICR exponent is independent of $\alpha$.
    \item \textbf{Covariate shift only} ($P_X \neq Q_X, \ \eta_P=\eta_Q$):
    $$\mc{E}_Q^{\mathrm{ICR}}(\hat{f}_{\mathrm{ICR}})
    \le
    C_o'\lambda^2 n^{-\frac{\beta+1}{2}},$$
    Here, ICR is consistent and converges faster than ICL;
    covariate shift only inflates the constant.
    \item \textbf{Posterior drift only} ($P_X = Q_X, \ \eta_P \neq \eta_Q$):
    $$\mc{E}_Q^{\mathrm{ICR}}(\hat{f}_{\mathrm{ICR}})
    \le \varepsilon
    +\frac{C_o'}{\kappa^{\beta+1}}n^{-\frac{\beta+1}{2}},$$
    The $n$-related term retains the faster decay, but posterior drift introduces an additive error $\varepsilon$ and a factor $\kappa^{-(\beta+1)}$ in the constant.
\end{itemize}

\section{Experiments}
\label{sec:experiments}

Our synthetic experiments validate each component of the theory:
\emph{(i)} with no domain shift, ICR converges faster in $n$ than ICL
\emph{(ii)} under posterior drift, ICR incurs a bias growing with $\varepsilon$
\emph{(iii)} ICL degrades as the target posterior becomes rougher, as quantified by $\alpha$, while ICR is insensitive to $\alpha$
\emph{(iv)} the ICR error saturates once the corpus size exceeds $m_\star(n)$.

We then test the effect of posterior drift on a real-world, sentiment classification task with annotator disagreement.
Full setups and
detailed results
are in Appendix~\ref{app:exp-details}.

\subsection{Synthetic experiments}
\label{subsec:exp-synthetic}

\textbf{Setup.}
Our experiments use the one-dimensional Weierstrass family (see Appendix~\ref{app:weierstrass}): $X\sim\mathrm{Uniform}[-1,1]$ and $\eta_\alpha(x)=\tfrac12+A\sin(2\pi f x)\bigl(1+b\,w_\alpha(x)\bigr)$, where $w_\alpha$ is a truncated Weierstrass series whose roughness is controlled by $\alpha$
($\alpha$ represents the H\"older exponent as the truncation level $J\to\infty$).
This construction varies the smoothness $\alpha$ while holding the margin exponent, the Bayes risk ($R_Q^*\approx 0.347$), and the Bayes classifier fixed. Unless stated otherwise, $P=Q$, the retrieval corpus has $m=10^5$ points, and each point on the figure is the mean over 5 seeds. Posterior drift is induced by flipping the source posterior, $\eta_P=1-\eta_Q$, on a region $B$ with $Q_X(B)=\varepsilon$.

We compare two learners: \textbf{TabPFN}~\citep{hollmann2025accurate}, a tabular foundation model that predicts from in-context demonstrations, and the \textbf{$k$-NN} style learners for ICL/ICR defined in Section~\ref{subsec:rates}.

\textbf{Experiment 1: convergence in $n$ (Figure~\ref{fig:exp-n}).}
With $P=Q$ and a large corpus, ICR outperforms ICL at every $n$, and on log-log axes the ICR curves decay visibly faster, consistent with the rate comparison of Theorem~\ref{thm:icl-icr-rates}.
TabPFN and $k$-NN track each other closely, supporting nearest-neighbor estimators as a reasonable model of the actual ICL/ICR learner.

\begin{figure}[htbp]
\centering
\begin{minipage}[t]{0.48\textwidth}
\centering
\begin{tikzpicture}
\begin{axis}[
    width=0.8*\expfigw, height=0.8*\expfigh,
    xmode=log, ymode=log,
    log basis x=2,
    xtick={4,16,64,256,1024},
    xticklabels={4,16,64,256,1024},
    xlabel={$n$ (demonstrations)},
    ylabel={Target excess risk},
    grid=major, grid style={gray!25},
    tick label style={font=\small}, label style={font=\small},
    legend columns=2,
    legend style={at={(0.5,-0.32)}, anchor=north, font=\small,
                  column sep=0.8em, legend cell align=left, draw=none},
]
\addplot[blue, solid, thick, mark=*, mark size=1.8pt]
    coordinates {(4,0.1512) (16,0.1508) (32,0.1338) (64,0.1032) (128,0.0593) (256,0.0360) (512,0.0191) (1024,0.0083)};
\addlegendentry{TabPFN-ICL}
\addplot[orange!90!black, solid, thick, mark=*, mark size=1.8pt]
    coordinates {(4,0.0664) (16,0.0398) (32,0.0237) (64,0.0180) (128,0.0116) (256,0.0071) (512,0.0031) (1024,0.0013)};
\addlegendentry{TabPFN-ICR}
\addplot[blue, dashed, thick, mark=square, mark size=1.8pt]
    coordinates {(4,0.1472) (16,0.1285) (32,0.0919) (64,0.0579) (128,0.0288) (256,0.0204) (512,0.0178) (1024,0.0129)};
\addlegendentry{$k$-NN-ICL}
\addplot[orange!90!black, dashed, thick, mark=square, mark size=1.8pt]
    coordinates {(4,0.0581) (16,0.0231) (32,0.0140) (64,0.0084) (128,0.0048) (256,0.0022) (512,0.0010) (1024,0.0004)};
\addlegendentry{$k$-NN-ICR}
\end{axis}
\end{tikzpicture}
\caption{Target excess risk vs.\ context size $n$ ($P=Q$, $\alpha=1.0$, $m=10^5$).
}
\label{fig:exp-n}
\end{minipage}\hfill
\begin{minipage}[t]{0.48\textwidth}
\centering
\begin{tikzpicture}
\begin{axis}[
    width=0.8*\expfigw, height=0.8*\expfigh,
    xmin=-0.02, xmax=1.02,
    xtick={0,0.2,0.4,0.6,0.8,1.0},
    ymin=0,
    xlabel={Posterior drift $\varepsilon$},
    ylabel={Target excess risk},
    grid=major, grid style={gray!25},
    tick label style={font=\small}, label style={font=\small},
    legend columns=2,
    legend style={at={(0.5,-0.32)}, anchor=north, font=\small,
                  column sep=0.8em, legend cell align=left, draw=none},
]
\addplot[blue, solid, thick, mark=none]
    coordinates {(-0.02,0.1510) (1.02,0.1510)};
\addlegendentry{TabPFN-ICL}
\addplot[orange!90!black, solid, thick, mark=*, mark size=1.8pt]
    coordinates {(0,0.0398) (0.1,0.0629) (0.2,0.0829) (0.3,0.1052) (0.4,0.1289) (0.5,0.1494) (0.6,0.1723) (0.7,0.1935) (0.8,0.2171) (0.9,0.2404) (1,0.2643)};
\addlegendentry{TabPFN-ICR}
\addplot[blue, dashed, thick, mark=none]
    coordinates {(-0.02,0.1285) (1.02,0.1285)};
\addlegendentry{$k$-NN-ICL}
\addplot[orange!90!black, dashed, thick, mark=square, mark size=1.8pt]
    coordinates {(0,0.0231) (0.1,0.0494) (0.2,0.0730) (0.3,0.0984) (0.4,0.1243) (0.5,0.1510) (0.6,0.1766) (0.7,0.2019) (0.8,0.2274) (0.9,0.2536) (1,0.2805)};
\addlegendentry{$k$-NN-ICR}
\end{axis}
\end{tikzpicture}
\caption{Target excess risk vs.\ posterior drift $\varepsilon$ ($n=16$, $\alpha=1.0$, $m=10^5$). ICL (blue) is flat; the ICR error (orange) grows almost linearly in $\varepsilon$.}
\label{fig:exp-eps}
\end{minipage}
\end{figure}

\textbf{Experiment 2: posterior drift (Figure~\ref{fig:exp-eps}).}
Fixing $n=16$ and sweeping $\varepsilon$, the ICR excess risk grows almost linearly in $\varepsilon$, matching the irreducible-bias term of Lemma~\ref{lem:fundamental-reduction} and Proposition~\ref{prop:icr-lower-bound-simple}, while ICL is unaffected by the drift. In this few-shot regime, ICR remains preferable up to $\varepsilon\approx 0.4$,
after which the drift bias outweighs ICR's variance reduction.
At larger context sizes the crossover occurs much earlier (Appendix~\ref{app:exp-drift-weier}).

\textbf{Experiment 3: smoothness
(Figure~\ref{fig:exp-alpha}).}
Sweeping \(\alpha\) at \(n=256\), both ICL curves deteriorate as \(\eta_\alpha\) becomes rougher, whereas ICR is essentially flat in \(\alpha\). This demonstrates empirical insensitivity of ICR to this roughness manipulation at the chosen corpus size. Theorem~\ref{thm:icl-icr-rates} guarantees an \(\alpha\)-independent ICR rate under a large-corpus condition whose threshold itself depends on \(\alpha_P\).

\begin{figure}[htbp]
\centering
\begin{minipage}[t]{0.48\textwidth}
\centering
\begin{tikzpicture}
\begin{axis}[
    width=0.8*\expfigw, height=0.8*\expfigh,
    xmode=log,
    xmin=0.0005, xmax=2,
    ymin=0, ymax=0.12,
    ytick={0,0.05,0.1},
    yticklabels={0,0.05,0.1},
    xtick={0.001,0.01,0.1,1},
    xticklabels={$10^{-3}$,$10^{-2}$,$10^{-1}$,$10^{0}$},
    x dir=reverse,
    xlabel={Smoothness $\alpha$ (rougher $\rightarrow$)},
    ylabel={Target excess risk},
    grid=major, grid style={gray!25},
    tick label style={font=\small}, label style={font=\small},
    legend columns=2,
    legend style={at={(0.5,-0.35)}, anchor=north, font=\small,
                  column sep=0.8em, legend cell align=left, draw=none},
]
\addplot[blue, solid, thick, mark=*, mark size=1.8pt]
    coordinates {(1,0.0360) (0.1,0.0377) (0.01,0.0455) (0.001,0.0804)};
\addlegendentry{TabPFN-ICL}
\addplot[orange!90!black, solid, thick, mark=*, mark size=1.8pt]
    coordinates {(1,0.0071) (0.1,0.0041) (0.01,0.0054) (0.001,0.0042)};
\addlegendentry{TabPFN-ICR}
\addplot[blue, dashed, thick, mark=square, mark size=1.8pt]
    coordinates {(1,0.0204) (0.1,0.0799) (0.01,0.0889) (0.001,0.1042)};
\addlegendentry{$k$-NN-ICL}
\addplot[orange!90!black, dashed, thick, mark=square, mark size=1.8pt]
    coordinates {(1,0.0022) (0.1,0.0009) (0.01,0.0014) (0.001,0.0013)};
\addlegendentry{$k$-NN-ICR}
\end{axis}
\end{tikzpicture}
\caption{Target excess risk vs.\ smoothness $\alpha$ ($n=256$, $P=Q$, $m=10^5$).
ICL (blue) deteriorates as $\alpha$ decreases; ICR (orange) is flat.}
\label{fig:exp-alpha}
\end{minipage}\hfill
\begin{minipage}[t]{0.48\textwidth}
\centering
\begin{tikzpicture}
\begin{axis}[
    width=0.8*\expfigw, height=0.8*\expfigh,
    xmode=log, ymode=log,
    log basis x=2,
    xtick={64,256,1024,4096,16384},
    xticklabels={64,256,1024,4096,16384},
    xlabel={Corpus size $m$},
    ylabel={Target excess risk},
    grid=major, grid style={gray!25},
    tick label style={font=\small}, label style={font=\small},
    legend columns=2,
    legend style={at={(0.5,-0.32)}, anchor=north, font=\small,
                  column sep=0.8em, legend cell align=left, draw=none},
]
\addplot[orange!90!black, solid, thick, mark=*, mark size=1.8pt]
    coordinates {(64,0.1333) (128,0.0732) (256,0.0379) (512,0.0203) (1024,0.0206) (2048,0.0209) (4096,0.0173) (8192,0.0165) (16384,0.0178)};
\addlegendentry{TabPFN-ICR}
\addplot[orange!90!black, dashed, thick, mark=square, mark size=1.8pt]
    coordinates {(64,0.1448) (128,0.1396) (256,0.1089) (512,0.0185) (1024,0.0086) (2048,0.0079) (4096,0.0081) (8192,0.0071) (16384,0.0082)};
\addlegendentry{$k$-NN-ICR}
\addplot[black!60, dashed, thick, mark=none]
    coordinates {(512,0.004)(512,0.3)};
\addlegendentry{$m=n^{3/2}$}
\end{axis}
\end{tikzpicture}
\caption{Target excess risk of ICR vs.\ corpus size $m$ ($n=64$, $P=Q$, $\alpha_P=1.0$). The dashed line marks the reference scale $m=n^{3/2}=512$.
ICL does not depend on $m$ and is omitted for clarity.
}
\label{fig:exp-m}
\end{minipage}
\end{figure}

\textbf{Experiment 4: corpus size (Figure~\ref{fig:exp-m}).}
Fixing $n=64$ and sweeping the corpus size $m$, the ICR error drops sharply and then plateaus near the predicted saturation scale $m_\star(n)\asymp n^{1+1/(2\alpha_P)}$ of \Cref{eq:icr-saturation} (here $\alpha_P=1$; the dashed vertical line marks the reference scale $m=n^{3/2}=512$).


\subsection{Language task}
\label{subsec:exp-sentiment}

\textbf{Setup.}
We use the sentiment analysis dataset of \citet{diaz2018addressing}, in which each sentence is rated by several annotators with $\approx 40\%$ label disagreement (training: 59{,}235 annotated examples / 1{,}481 annotators; test: 1{,}419 annotated examples / 878 annotators). Each test annotator defines a target distribution $Q$. The in-context learner is Qwen3.5-4B~\citep{qwen3.5}; sentences are embedded with BAAI/bge-base-en-v1.5~\citep{bge_embedding} and retrieved by Euclidean distance. ICL draws $n$ demonstrations from the annotator's own training data. We sweep four source-corpus designs with increasing level of posterior drift: no drift ($P=Q$), mild (\emph{self $+$ 100 other annotators}), severe (\emph{all 1{,}481 annotators}), and extreme (\emph{others only}). Full details are in Appendix~\ref{app:exp-details}.

\textbf{Results (Figure~\ref{fig:exp-sentiment}).}
ICR strongly outperforms ICL when $P=Q$ and remains preferable under mild drift for $n\le 16$; as the amount of posterior drift increases, ICR's performance degrades monotonically, and ICL eventually becomes preferable.
In Appendix~\ref{app:exp-details} we repeat the experiment with the stronger Qwen3.5-27B model (Figure~\ref{fig:exp-sentiment-scale}),
where we see
model capacity does not remove the bias, consistent with Proposition~\ref{prop:icr-lower-bound-simple}.

\begin{figure}[!htbp]
\centering
\begin{tikzpicture}
\begin{axis}[
    width=0.7*\expfigw, height=0.75*\expfigh,
    xmode=log, log basis x=2,
    xtick={4,8,16,32}, xticklabels={4,8,16,32},
    ymin=0, ymax=0.7,
    ytick={0,0.2,0.4,0.6},
    xlabel={$n$ (demonstrations)},
    ylabel={Test error},
    grid=major, grid style={gray!25},
    tick label style={font=\small}, label style={font=\small},
    legend columns=1,
    legend style={at={(1.04,0.5)}, anchor=west, font=\small,
                  legend cell align=left, draw=none},
]
\addplot[blue, solid, thick, mark=*, mark size=1.8pt]
    coordinates {(4,0.5383) (8,0.4806) (16,0.3480) (32,0.1168)};
\addlegendentry{ICL}
\addplot[orange!50!yellow, solid, thick, mark=square*, mark size=1.8pt]
    coordinates {(4,0.2630) (8,0.1693) (16,0.0985) (32,0.0724)};
\addlegendentry{ICR, no drift ($P=Q$)}
\addplot[orange!90!black, solid, thick, mark=triangle*, mark size=2.2pt]
    coordinates {(4,0.2653) (8,0.2250) (16,0.1622) (32,0.1420)};
\addlegendentry{ICR, mild drift}
\addplot[red!80!black, solid, thick, mark=diamond*, mark size=2.2pt]
    coordinates {(4,0.5347) (8,0.5039) (16,0.4910) (32,0.4800)};
\addlegendentry{ICR, severe drift}
\addplot[red!40!black, solid, thick, mark=pentagon*, mark size=2.2pt]
    coordinates {(4,0.6334) (8,0.6277) (16,0.6149) (32,0.6190)};
\addlegendentry{ICR, extreme drift}
\end{axis}
\end{tikzpicture}
\caption{Test error vs.\ context size $n$ on the sentiment dataset (Qwen3.5-4B; 878 test annotators, mean over 5 seeds). \emph{Blue}: ICL with i.i.d.\ demonstrations from the annotator's own data. \emph{Yellow $\to$ red}: ICR under four corpus designs with increasing posterior drift.}
\label{fig:exp-sentiment}
\end{figure}
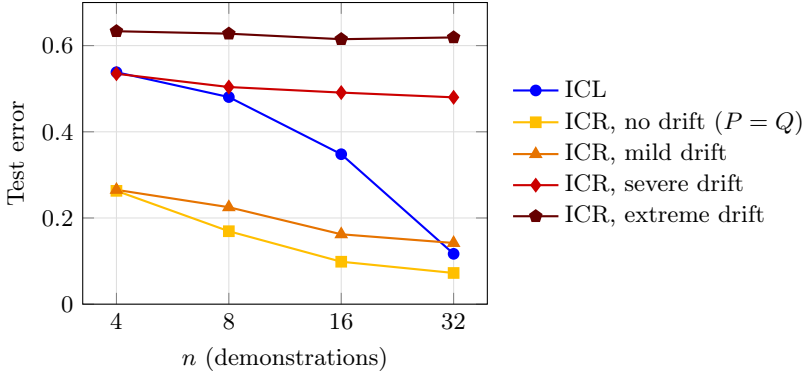

\newpage

\section{Conclusions \& Limitations}

We cast in-context retrieval as domain adaptation with a retrieval-based learner and derive finite-sample excess-risk bounds for both ICL and ICR. The ICR bound has a strictly faster dependence on $n$ than the ICL bound, but ICR also incurs an irreducible bias that prevents ICR from being asymptotically consistent uniformly under posterior drift.

Synthetic and sentiment-classification experiments confirm these predictions and carry a practical takeaway: ICR is not always preferable, na\"ively pooling all available source data into the retrieval corpus is not a safe default, and a smaller, well-aligned corpus can be better than a larger one.

The limitation of our work lies in that our analysis is specialized to binary classification and treats the in-context learners (induced by LLMs) as nearest-neighbor-style estimators.


\section*{Acknowledgements}
The authors thank Aditya Gangrade for helpful discussions and feedback.

\bibliography{refs}
\bibliographystyle{plainnat}

\newpage

\appendix

\begin{center}
{\Large\bfseries Appendix}
\end{center}

\section{Rate of convergence for nearest neighbor classification}
\label{sec:knn-rate}

This appendix reviews the excess-risk bound for the $k$-nearest-neighbor ($k$-NN) majority-vote classifier of \citet{chaudhuri2014rates}, restated using our paper's notation for a generic source distribution $P$. The bound is then specialized to the ICL setting (with $P=Q$ and training size $n$) and the ICR setting (with corpus size $m$ and neighbor count $n$) in the proof of
\Cref{thm:icr-bound}.

\paragraph{Setup.}
Let $P$ be a Borel probability distribution on $\mc{X}\times\{0,1\}$, where $(\mc{X},\rho)$ is a separable metric space. Let $\eta_P(x):=P(Y=1\mid X=x)$ be the class-probability function, $f_P^*(x):=\indicator{\eta_P(x)\ge\tfrac12}$ the Bayes classifier, $R_P(f):=P(f(X)\neq Y)$ the risk and $R_P^*:=R_P(f_P^*)$ the Bayes risk. Given an i.i.d.\ sample $\mc{D}_m=\{(X_i,Y_i)\}_{i=1}^m\sim P^{\otimes m}$, the $k$-NN majority-vote classifier is
\[
    \hat{f}^{\mathrm{kNN}}_{m,k}(x) \;:=\; \indicator{\frac{1}{k}\sum_{j=1}^{k}Y_{(j)}(x) \ge \tfrac12},
\]
where $Y_{(j)}(x)$ is the label of the $j$-th nearest neighbor of $x$ in $\mc{D}_m$ under $\rho$.

\paragraph{Assumptions on $P$.}
\begin{itemize}
    \item \emph{$(\alpha,L)$-smoothness:} for some $\alpha,L>0$ and all $x\in\mathrm{supp}(P_X)$, $r>0$,
    \[
        \bigl|\eta_P(B(x,r))-\eta_P(x)\bigr| \;\le\; L\,P_X\bigl(B^o(x,r)\bigr)^{\!\alpha},
    \]
    where $\eta_P(A):=P_X(A)^{-1}\!\int_A\eta_P\,dP_X$ is the average of $\eta_P$ over a set $A$, $B(x,r):=\{x':\rho(x,x')\le r\}$ is the closed ball, and $B^o(x,r):=\{x':\rho(x,x')< r\}$ is the open ball.
    \item \emph{$\beta$-margin condition:} for some $\beta\ge 0$ and $C_\beta>0$,
    \[
        P_X\!\bigl(|\eta_P(X)-\tfrac12|\le t\bigr) \;\le\; C_\beta\,t^\beta \qquad\text{for all }t>0.
    \]
\end{itemize}

The following lemma is restated from \citet{chaudhuri2014rates}, Lemma~20.

\begin{lemma}[Pointwise excess risk for $k$-NN]
\label{lem:pointwise-knn}
Let $m\ge2k\ge2$, and suppose $\eta_P$ is $(\alpha_P,L_P)$-smooth. Set $p:=2k/m$ and $\Delta_o:=L_P\,p^{\alpha_P}$. If $x\in\mathrm{supp}(P_X)$ satisfies $|\eta_P(x)-\half|>\Delta_o$, the $k$-NN majority-vote classifier $\hat{f}^{\mathrm{kNN}}_{m,k}$ on $\mc{D}_m\sim P^{\otimes m}$ satisfies
\[
    \expect{\mc{D}_m}{\Delta_P\!\bigl(x;\,\hat{f}^{\mathrm{kNN}}_{m,k}\bigr)}
    \;\le\;
    e^{-k/8}
    \;+\; 4\,\bigl|\eta_P(x)-\half\bigr|\;
        e^{-2k\,\bigl(|\eta_P(x)-\half|\,-\,\Delta_o\bigr)^2}.
\]
\end{lemma}

The following lemma is an extension of \citet{chaudhuri2014rates}, Lemma~21, to a subset of the input space, which is needed for our analysis.

\begin{lemma}[Excess risk on a subset]
\label{lem:localized-knn}
Let $A\subseteq\mc{X}$ be a measurable set, let $m\ge2k\ge2$, and let $\hat{f}^{\mathrm{kNN}}_{m,k}$ denote the $k$-NN majority-vote classifier on a corpus $\mc{D}_m\sim P^{\otimes m}$. Suppose $\eta_P$ is $(\alpha_P,L_P)$-smooth and, for some $\beta\ge0$ and $C_A>0$, satisfies the following local margin condition on $A$:
    \[
        P_X\!\bigl(\{x\in A:\,|\eta_P(x)-\half|<t\}\bigr)\;\le\;C_A\,t^\beta
        \qquad\text{for all }t>0.
    \]
Then
\[
    \expect{\mc{D}_m,\,X\sim P_X}{\Delta_P\!\bigl(X;\,\hat{f}^{\mathrm{kNN}}_{m,k}\bigr)\,\indicator{X\in A}}
    \;\le\;
    P_X(A)\,e^{-k/8}
    \;+\; \frac{5}{2}\,C_A\!\left[\max\!\left(2L_P\!\Bigl(\tfrac{2k}{m}\Bigr)^{\!\alpha_P},\;\sqrt{\tfrac{8(\beta+2)}{k}}\,\right)\right]^{\beta+1}.
\]
\end{lemma}

\begin{proof}
By Fubini and non-negativity, the expectation equals an integral over $A$ of the pointwise expected excess risk:
\[
    \expect{\mc{D}_m,\,X\sim P_X}{\Delta_P\!\bigl(X;\,\hat{f}^{\mathrm{kNN}}_{m,k}\bigr)\,\indicator{X\in A}}
    \;=\;
    \int_A \expect{\mc{D}_m}{\Delta_P\!\bigl(x;\,\hat{f}^{\mathrm{kNN}}_{m,k}\bigr)}\,P_X(dx).
\]
Set $p:=2k/m$, $\Delta_o:=L_P\,p^{\alpha_P}$, and
\[
    M:=\max\!\left(2\Delta_o,\sqrt{\frac{8(\beta+2)}{k}}\right).
\]
Writing $\Delta(x):=|\eta_P(x)-\half|$, decompose $A$ into
\[
    A_{\mathrm{near}}:=\{x\in A:\Delta(x)\le M\},\qquad
    A_{\mathrm{far}}:=\{x\in A:\Delta(x)>M\}.
\]

\textbf{Step 1: Near region.}
On $A_{\mathrm{near}}$, the trivial pointwise bound $\Delta_P(x;\,\hat f)\le 2\Delta(x)\le 2M$ together with the local margin condition gives
\[
    \int_{A_{\mathrm{near}}}\!\expect{\mc{D}_m}{\Delta_P\!\bigl(x;\,\hat{f}^{\mathrm{kNN}}_{m,k}\bigr)}\,P_X(dx)
    \;\le\; 2M\cdot P_X(A_{\mathrm{near}})
    \;\le\; 2\,C_A\,M^{\beta+1}.
\]

\textbf{Step 2: Far region.}
For $x\in A_{\mathrm{far}}$, global smoothness and $\Delta(x)>M\ge2\Delta_o$ allow us to apply Lemma~\ref{lem:pointwise-knn}:
\[
    \expect{\mc{D}_m}{\Delta_P\!\bigl(x;\,\hat{f}^{\mathrm{kNN}}_{m,k}\bigr)}
    \;\le\; e^{-k/8} \;+\; 4\,\Delta(x)\,e^{-2k\,(\Delta(x)-\Delta_o)^2}.
\]
Integrating over $A_{\mathrm{far}}$ and using $P_X(A_{\mathrm{far}})\le P_X(A)$ on the exponential term,
\begin{align*}
    \int_{A_{\mathrm{far}}}\!\expect{\mc{D}_m}{\Delta_P\!\bigl(x;\,\hat{f}^{\mathrm{kNN}}_{m,k}\bigr)}\,P_X(dx)
    \;\le\; P_X(A)\,e^{-k/8}
    + \int_{A_{\mathrm{far}}}\!4\Delta(x)\,e^{-2k(\Delta(x)-\Delta_o)^2}\,P_X(dx).
\end{align*}
Partition $A_{\mathrm{far}}$ into dyadic shells
$S_i:=\{x\in A:2^iM<\Delta(x)\le2^{i+1}M\}$ for $i\ge0$. Since $M\ge2\Delta_o$, on $S_i$ we have $\Delta(x)-\Delta_o\ge2^{i-1}M$. The local margin condition therefore gives
\[
    \int_{A_{\mathrm{far}}}\!4\Delta(x)e^{-2k(\Delta(x)-\Delta_o)^2}\,P_X(dx)
    \;\le\; C_A M^{\beta+1}\sum_{i\ge0}4\,2^{(i+1)(\beta+1)}e^{-(kM^2/2)4^i}.
\]

\textbf{Step 3: Geometric series.}
Because $kM^2\ge8(\beta+2)$, the ratio of consecutive summands in the last display is at most $2^{\beta+1}e^{-12(\beta+2)}\le\half$. Their sum is therefore at most twice its first term, which is bounded by
\[
    8\cdot2^{\beta+1}e^{-4(\beta+2)}
    =16e^{-8}(2e^{-4})^\beta\le\half.
\]
Thus the integral over the far region, apart from $P_X(A)e^{-k/8}$, is at most $\half C_A M^{\beta+1}$.

\textbf{Step 4: Combine.}
Adding Steps~1 and~3,
\[
    \int_A\expect{\mc{D}_m}{\Delta_P\!\bigl(x;\,\hat{f}^{\mathrm{kNN}}_{m,k}\bigr)}\,P_X(dx)
    \;\le\; P_X(A)\,e^{-k/8}
    \;+\; \frac{5}{2}\,C_A\,M^{\beta+1},
\]
which is the claimed bound.
\end{proof}

\begin{theorem}[\citealt{chaudhuri2014rates}, restated in our notation]
\label{thm:sanjoy_knn}
Suppose $\eta_P$ is $(\alpha,L)$-smooth and satisfies the $\beta$-margin condition with constant $C_\beta$. Then, for every $m\ge2$ and $1\le k\le m/2$,
\begin{equation}
\label{eq:sanjoy-knn-bound}
    \mathbb{E}_{\mc{D}_m\sim P^{\otimes m}}\!\left[R_P\!\bigl(\hat{f}^{\mathrm{kNN}}_{m,k}\bigr)-R_P^*\right]
    \;\le\;
    e^{-k/8}
    \;+\;
    6\,C_\beta\max\!\left(
        2L\!\left(\frac{2k}{m}\right)^{\!\alpha},\;
        \sqrt{\frac{8(\beta+2)}{k}}
    \right)^{\!\beta+1}.
\end{equation}
Consequently, choosing $k\asymp m^{2\alpha/(2\alpha+1)}$ (with integer rounding and truncation to $[1,m/2]$ for small $m$) yields
\begin{equation}
\label{eq:sanjoy-knn-rate}
    \mathbb{E}_{\mc{D}_m\sim P^{\otimes m}}\!\left[R_P\!\bigl(\hat{f}^{\mathrm{kNN}}_{m,k}\bigr)-R_P^*\right]
    \;\le\; C_o\,m^{-\alpha(\beta+1)/(2\alpha+1)},
\end{equation}
for an appropriate $C_o=C_o(\alpha,\beta,L,C_\beta)$.
\end{theorem}

\begin{proof}
The finite-sample inequality~\eqref{eq:sanjoy-knn-bound} is Lemma~21 of \citet{chaudhuri2014rates} restated in our notation. With $k\asymp m^{2\alpha/(2\alpha+1)}$, both arguments of the maximum are $O\bigl(m^{-\alpha/(2\alpha+1)}\bigr)$, while $e^{-k/8}=e^{-\Omega(m^{2\alpha/(2\alpha+1)})}$ decays faster than any polynomial in $m$. For all sufficiently large $m$, this choice also satisfies $k\le m/2$; enlarging $C_o$ covers the finitely many smaller values of $m$. Absorbing the exponential term into $C_o$ gives~\eqref{eq:sanjoy-knn-rate}, which is Theorem~4(b) of \citet{chaudhuri2014rates}.
\end{proof}

This matches the minimax rate for nonparametric classification under $\alpha_H$-H\"{o}lder smoothness and the $\beta$-margin condition provided \(\alpha_H\beta\le d\) \citep{audibert2007fast}, since, by Lemma~2 of \citet{chaudhuri2014rates}, $\alpha_H$-H\"{o}lder smoothness of $\eta_P$ implies $(\alpha,L)$-smoothness with $\alpha=\alpha_H/d$ when $P_X$ has a density bounded below on a bounded support with regular boundary (e.g., a cube); substituting $\alpha=\alpha_H/d$ into~\eqref{eq:sanjoy-knn-rate} recovers the H\"{o}lder rate $m^{-\alpha_H(\beta+1)/(2\alpha_H+d)}$.

\section{Proofs}
\label{sec:proofs}

This appendix collects the full proofs of the lemmas, theorems, and propositions stated in Sections~\ref{subsec:rss-model-class} and~\ref{subsec:rates}. Throughout, we use the pointwise excess risk
\[
    \Delta_D(x;\,f) \;:=\; 2\bigl|\eta_D(x)-\tfrac{1}{2}\bigr|\,\indicator{f(x)\neq f_D^*(x)},\qquad D\in\{P,Q\},
\]
which satisfies $R_D(f)-R_D^* = \mathbb{E}_{X\sim D_X}[\Delta_D(X;\,f)]$.

\subsection{Proof of Lemma~\ref{lem:fundamental-reduction} (Risk transfer)}
\label{proof:fundamental-reduction}

\begin{lemma*}[Risk transfer; restated from Lemma~\ref{lem:fundamental-reduction}]
If $(P,Q)\in\Pi(\varepsilon,\kappa,\lambda)$, then for any classifier $f:\mc{X}\to\{0,1\}$,
\begin{align*}
    R_Q(f) - R_Q^*
    \;\leq\;
    \varepsilon + \frac{\lambda}{\kappa}\,\bigl(R_P(f) - R_P^*\bigr).
\end{align*}
\end{lemma*}

\begin{proof}
Write $A_\kappa:=\mc{A}_\kappa(P,Q)$ and decompose:
\begin{align*}
    R_Q(f)-R_Q^*
    = \expect{X\sim Q_X}{\Delta_Q(X;\,f)}
    = \underbrace{\expect{X\sim Q_X}{\Delta_Q(X;\,f)\,\indicator{X\in A_\kappa}}}_{I_1}
    +\underbrace{\expect{X\sim Q_X}{\Delta_Q(X;\,f)\,\indicator{X\notin A_\kappa}}}_{I_2}.
\end{align*}

\textbf{Low-RSS term.}
Since $\Delta_Q\in[0,1]$, the complement of $A_\kappa$ contributes at most its $Q_X$-mass:
\[
    I_2 \;\le\; Q_X(A_\kappa^c) \;\le\; \varepsilon.
\]

\textbf{High-RSS term.}
If $\eta_Q(x)=\half$ then $\Delta_Q(x;\,f)=0$ and the bound below is trivial. Otherwise, for $x\in A_\kappa$, by definition $\eta_P$ and $\eta_Q$ lie on the same side of $\half$, so the Bayes classifiers agree, $f_P^*(x)=f_Q^*(x)$. Any prediction error of $f$ at $x$ is therefore an error under both $P$ and $Q$, and the target signal is weaker by at most $\kappa$:
\begin{align*}
    \Delta_Q(x;\,f)
    = 2\bigl|\eta_Q(x)-\tfrac{1}{2}\bigr|\,\indicator{f(x)\neq f_Q^*(x)}
    \;\le\;
    \frac{2\bigl|\eta_P(x)-\tfrac{1}{2}\bigr|}{\kappa}\,\indicator{f(x)\neq f_P^*(x)}
    = \frac{1}{\kappa}\,\Delta_P(x;\,f).
\end{align*}
Hence
\begin{align*}
    I_1
    \;\le\; \frac{1}{\kappa}\,\expect{X\sim Q_X}{\Delta_P(X;\,f)}
    \;\le\; \frac{\lambda}{\kappa}\,\expect{X\sim P_X}{\Delta_P(X;\,f)}
    = \frac{\lambda}{\kappa}\,\bigl(R_P(f)-R_P^*\bigr),
\end{align*}
where the second inequality uses the density ratio $dQ_X\le\lambda\,dP_X$. Adding the bounds on $I_1$ and $I_2$ gives the result.
\end{proof}

\subsection{Proof of Lemma~\ref{lem:margin-transfer} (margin transfer via RSS)}
\label{proof:margin-transfer}

\begin{lemma*}[Margin transfer via RSS; restated from Lemma~\ref{lem:margin-transfer}]
If $(P,Q)\in\Pi(\varepsilon,\kappa,\lambda)$ and $\eta_Q$ satisfies the $\beta$-margin condition with constant $C_\beta$, then for the high-RSS region $A_\kappa:=\mc{A}_\kappa(P,Q)$,
\begin{align*}
    P_X\!\bigl(\{|\eta_P(X)-\half|<t\}\cap A_\kappa\bigr)
    \;\leq\;
    \lambda C_\beta\kappa^{-\beta}\,t^\beta
    \quad\forall\,t>0.
\end{align*}
We denote $C_\beta':=\lambda C_\beta\kappa^{-\beta}$.
\end{lemma*}

\begin{proof}
Fix any $t>0$.

\textbf{Step 1: Subset inclusion with RSS.}
For $x\in A_\kappa$, by definition, $|\eta_P(x)-\half|\ge\kappa\,|\eta_Q(x)-\half|$. Therefore, if $x\in A_\kappa$ and $|\eta_P(x)-\half|<t$, then
\[
    \kappa\,|\eta_Q(x)-\half| \le |\eta_P(x)-\half| < t,
    \quad\text{i.e.,}\quad
    |\eta_Q(x)-\half| < \frac{t}{\kappa}.
\]
This gives the set inclusion
\[
    \bigl\{|\eta_P(X)-\half|<t\bigr\}\cap A_\kappa
    \;\subseteq\;
    \bigl\{|\eta_Q(X)-\half|<t/\kappa\bigr\}.
\]

\textbf{Step 2: Change measure from $P_X$ to $Q_X$.}
The set on the left lies in $A_\kappa\subseteq \Omega_Q$ by Definition~\ref{def:rss}, on which \eqref{eq:density-ratio} gives $dP_X\le\lambda\,dQ_X$. By monotonicity of probability and this bound,
\begin{align*}
    P_X\!\bigl(|\eta_P(X)-\half|<t,\;X\in A_\kappa\bigr)
    &\;\le\; \lambda\,Q_X\!\bigl(|\eta_P(X)-\half|<t,\;X\in A_\kappa\bigr) \\
    &\;\le\; \lambda\,Q_X\!\bigl(|\eta_Q(X)-\half|<t/\kappa\bigr).
\end{align*}

\textbf{Step 3: Apply the $\beta$-margin condition on $\eta_Q$.}
By assumption, $Q_X(|\eta_Q(X)-\half|<u)\le C_\beta\,u^\beta$ for all $u>0$. Setting $u=t/\kappa$,
\[
    \lambda\,Q_X\!\bigl(|\eta_Q(X)-\half|<t/\kappa\bigr)
    \;\le\; \lambda\,C_\beta\,(t/\kappa)^\beta
    \;=\; \lambda\,C_\beta\,\kappa^{-\beta}\,t^\beta.
\]
Chaining the three steps yields the final result \eqref{eq:margin-transfer}.
\end{proof}

\subsection{Proof of Theorem~\ref{thm:icl-bound} (ICL excess risk)}
\label{proof:icl-bound}

\begin{theorem*}[ICL excess risk; restated from Theorem~\ref{thm:icl-bound}]
Under $(\alpha,L)$-smoothness and $\beta$-margin with constant $C_\beta$ on $\eta_Q$,
\[
    \mc{E}_Q^{\mathrm{ICL}}(\hat{f}_{\mathrm{ICL}})
    \;\leq\;
    C_o\, n^{-\frac{\alpha(\beta+1)}{2\alpha+1}},
\]
where $C_o=C_o(\alpha,\beta,L,C_\beta)$ is the constant of Theorem~\ref{thm:sanjoy_knn}.
\end{theorem*}

\begin{proof}
For ICL, the demonstration set $S_n=((X_i,Y_i))_{i=1}^n$ consists of $n$ i.i.d.\ samples from $Q$, and the induced classifier $\hat{f}_{\mathrm{ICL}}$ is the $k$-NN majority vote with the rate-optimal inner tuning parameter $k\asymp n^{2\alpha/(2\alpha+1)}$. Specializing Theorem~\ref{thm:sanjoy_knn} to the target distribution ($Q$ in place of $P$, demonstration size $n$ in place of $m$), \eqref{eq:sanjoy-knn-rate} gives
\[
    \mc{E}_Q^{\mathrm{ICL}}(\hat{f}_{\mathrm{ICL}})
    \;=\; \mathbb{E}_{S_n\sim Q^{\otimes n}}\!\left[R_Q(f_{S_n})-R_Q^*\right]
    \;\le\; C_o\,n^{-\alpha(\beta+1)/(2\alpha+1)}.
\]
This proves the claim.
\end{proof}

\subsection{Proof of Theorem~\ref{thm:icr-bound} (ICR excess risk)}
\label{proof:icr-bound}

\begin{theorem*}[ICR excess risk; full version of Theorem~\ref{thm:icr-bound}]
Let $(P,Q)\in\Pi(\varepsilon,\kappa,\lambda)$, suppose $\eta_Q$ satisfies
the $\beta$-margin condition with constant $C_\beta$, and
$\eta_P$ satisfies $(\alpha_P,L_P)$-smoothness. Then there exists a finite $n_0=n_0(\beta,C_\beta,\kappa,\lambda)$ such that, for every $n\ge n_0$ and $m\ge2n$,
\begin{equation}
\label{eq:icr-finite-m-poly}
    \mc{E}_Q^{\mathrm{ICR}}(\hat{f}_{\mathrm{ICR}})
    \;\leq\;
    \varepsilon
    \;+\; \frac{3\,C_\beta\,\lambda^2}{\kappa^{\beta+1}}
    \max\!\left(
        2L_P\!\left(\frac{2n}{m}\right)^{\!\alpha_P},\;
        \sqrt{\frac{8(\beta+2)}{n}}
    \right)^{\!\beta+1}.
\end{equation}
Define
\[
    C_o' := 3C_\beta\bigl(8(\beta+2)\bigr)^{(\beta+1)/2}.
\]
If moreover the corpus exceeds the critical size $m_\star(n)$,
\begin{equation}
\label{eq:m-crit}
    m \;\ge\; m_\star(n)
    \;:=\; 2n\left(2L_P\sqrt{\frac{n}{8(\beta+2)}}\right)^{\!1/\alpha_P}
    \;=\; C_\star\,n^{1+\frac{1}{2\alpha_P}},
    \qquad
    C_\star := 2\left(\frac{2L_P}{\sqrt{8(\beta+2)}}\right)^{\!1/\alpha_P},
\end{equation}
then the maximum in~\eqref{eq:icr-finite-m-poly} is attained by its second argument, and
\begin{equation}
\label{eq:icr-saturated-full}
    \mc{E}_Q^{\mathrm{ICR}}(\hat{f}_{\mathrm{ICR}})
    \;\leq\;
    \varepsilon
    \;+\; \frac{C_o'\lambda^2}{\kappa^{\beta+1}}
    n^{-\frac{\beta+1}{2}}.
\end{equation}
\end{theorem*}

\begin{proof}
Throughout, write $A_\kappa:=\mc{A}_\kappa(P,Q)$ and recall the pointwise excess risk $\Delta_D(x;\,f)=2|\eta_D(x)-\half|\,\indicator{f(x)\neq f_D^*(x)}$, so that $R_D(f)-R_D^*=\expect{X\sim D_X}{\Delta_D(X;\,f)}$ for $D\in\{P,Q\}$. For a fixed corpus $\mc{D}_m$, let $f_{\mc{D}_m}(x):=\hat{f}_{\mathrm{ICR}}\bigl(x,T_n(x,\mc{D}_m)\bigr)$ be the induced classifier.

\textbf{Step 1: A localized risk transfer.}
Decompose the target excess risk over $A_\kappa$ and its complement, exactly as in the proof of Lemma~\ref{lem:fundamental-reduction}:
\[
    R_Q(f_{\mc{D}_m})-R_Q^*
    = \underbrace{\expect{X\sim Q_X}{\Delta_Q(X;\,f_{\mc{D}_m})\,\indicator{X\in A_\kappa}}}_{I_1}
    + \underbrace{\expect{X\sim Q_X}{\Delta_Q(X;\,f_{\mc{D}_m})\,\indicator{X\notin A_\kappa}}}_{I_2}.
\]
Since $\Delta_Q\in[0,1]$, the complement term satisfies $I_2\le Q_X(A_\kappa^c)\le\varepsilon$. On $A_\kappa$ we have $\Delta_Q(x;\,f)\le\kappa^{-1}\Delta_P(x;\,f)$ (as shown in that proof), and the density ratio $dQ_X\le\lambda\,dP_X$ converts the expectation from $Q_X$ to $P_X$. Unlike Lemma~\ref{lem:fundamental-reduction}, we now \emph{retain} the indicator $\indicator{X\in A_\kappa}$ rather than bounding it by $1$:
\[
    I_1 \;\le\; \frac{1}{\kappa}\,\expect{X\sim Q_X}{\Delta_P(X;\,f_{\mc{D}_m})\,\indicator{X\in A_\kappa}}
    \;\le\; \frac{\lambda}{\kappa}\,\expect{X\sim P_X}{\Delta_P(X;\,f_{\mc{D}_m})\,\indicator{X\in A_\kappa}}.
\]
Adding the two bounds and taking expectation over $\mc{D}_m\sim P^{\otimes m}$,
\begin{align}
    \mc{E}_Q^{\mathrm{ICR}}(\hat{f}_{\mathrm{ICR}})
    \;\le\; \varepsilon
    + \frac{\lambda}{\kappa}\,
    \underbrace{\expect{\mc{D}_m,\,X\sim P_X}{\Delta_P(X;\,f_{\mc{D}_m})\,\indicator{X\in A_\kappa}}}_{=:\,J}.
    \label{eq:icr-step1}
\end{align}

\textbf{Step 2: Localized $k$-NN rate on $A_\kappa$.}
By construction the ICR classifier is the majority vote over the $n$ nearest neighbors of $X$ in $\mc{D}_m$, i.e.\ $f_{\mc{D}_m}=\hat{f}^{\mathrm{kNN}}_{m,k}$ with neighbor count $k=n$. We apply Lemma~\ref{lem:localized-knn} with $A=A_\kappa$ and $k=n$. Its smoothness assumption holds by the assumed $(\alpha_P,L_P)$-smoothness of $\eta_P$, and $m\ge2n$ ensures that $2n/m\le1$. Its local margin assumption follows from Lemma~\ref{lem:margin-transfer}, which gives $P_X\bigl(\{|\eta_P(X)-\half|<t\}\cap A_\kappa\bigr)\le C_\beta'\,t^\beta$ for all $t>0$, with $C_\beta':=\lambda C_\beta\kappa^{-\beta}$.
Lemma~\ref{lem:localized-knn} then yields, with $M_{n,m}:=\max\bigl(2L_P(2n/m)^{\alpha_P},\,\sqrt{8(\beta+2)/n}\bigr)$,
\[
    J \;\le\; P_X(A_\kappa)\,e^{-n/8} + \frac{5}{2}\,C_\beta'\,M_{n,m}^{\,\beta+1}
    \;\le\; e^{-n/8} + \frac{5}{2}\,C_\beta'\,M_{n,m}^{\,\beta+1},
\]
where the last step uses $P_X(A_\kappa)\le 1$.

\textbf{Step 3: Combine (finite corpus).}
Substituting into~\eqref{eq:icr-step1} and using $C_\beta'=\lambda C_\beta\kappa^{-\beta}$,
\begin{align}
    \mc{E}_Q^{\mathrm{ICR}}(\hat{f}_{\mathrm{ICR}})
    \;\le\; \varepsilon
    \;+\; \frac{\lambda}{\kappa}\,e^{-n/8}
    \;+\; \frac{5\,C_\beta\,\lambda^2}{2\kappa^{\beta+1}}
    \max\!\left(
        2L_P\!\left(\frac{2n}{m}\right)^{\!\alpha_P},\;
        \sqrt{\frac{8(\beta+2)}{n}}
    \right)^{\!\beta+1},
    \label{eq:icr-finite-m}
\end{align}
which is the intermediate finite-corpus bound.

\textbf{Step 4: Dominating the exponential term.}
Because $e^{-n/8}=o\bigl(n^{-(\beta+1)/2}\bigr)$, there exists a finite $n_0=n_0(\beta,C_\beta,\kappa,\lambda)$ such that, for every $n\ge n_0$,
\[
    e^{-n/8}
    \;\le\;
    \frac{C_\beta'}{2}
    \left(\frac{8(\beta+2)}{n}\right)^{\!\frac{\beta+1}{2}}
    \;\le\; \frac{C_\beta'}{2}\,M_{n,m}^{\beta+1}.
\]
Multiplying by $\lambda/\kappa$ and combining with~\eqref{eq:icr-finite-m} increases the coefficient $5/2$ to $3$, giving~\eqref{eq:icr-finite-m-poly}.

\textbf{Step 5: Saturated corpus.}
The localization term $2L_P(2n/m)^{\alpha_P}$ is decreasing in $m$, and solving
\[
    2L_P\!\left(\frac{2n}{m}\right)^{\!\alpha_P} \;\le\; \sqrt{\frac{8(\beta+2)}{n}}
\]
for $m$ gives exactly $m\ge m_\star(n)$ with $m_\star(n)$ as in~\eqref{eq:m-crit}. Hence for $m\ge\max\{2n,\,m_\star(n)\}$ the maximum $M_{n,m}$ is attained by its second argument, and
\[
    M_{n,m}^{\,\beta+1}
    = \left(\frac{8(\beta+2)}{n}\right)^{\!(\beta+1)/2}
    = \bigl(8(\beta+2)\bigr)^{(\beta+1)/2}\,n^{-\frac{\beta+1}{2}}.
\]
Substituting into~\eqref{eq:icr-finite-m-poly} gives~\eqref{eq:icr-saturated-full}.
\end{proof}

\subsection{Proof of Proposition~\ref{prop:icr-lower-bound-simple} (lower bound) and its generalization}
\label{proof:icr-lower-bound-simple}

Proposition \ref{prop:icr-lower-bound-simple} follows immediately from the following. The result below shows that the lower bound of Proposition \ref{prop:icr-lower-bound-simple} holds under the assumptions of Theorem \ref{thm:icl-icr-rates}.

\begin{proposition}[Lower bound for ICR under regularity]
Assume $\mc{X}$ contains at least two distinct points. For any
$\varepsilon\in[0,1]$, $\kappa>0$, $\lambda\ge 1$, $\alpha,\alpha_P>0$, and
$\beta\ge 0$, there exist finite constants $L,L_P,C_\beta$ such that
\[
\inf_{\hat{f}}
\sup_{(P,Q)\in\Pi_{\rm reg}(\varepsilon,\kappa,\lambda)}
\mc{E}_Q^{\rm ICR}(\hat{f})
\ge \frac{\varepsilon}{2},
\]
where $\Pi_{\rm reg}(\varepsilon,\kappa,\lambda)$ denotes the subset of
$\Pi(\varepsilon,\kappa,\lambda)$ for which $\eta_Q$ is $(\alpha,L)$-smooth
and satisfies the $\beta$-margin condition with constant $C_\beta$, and
$\eta_P$ is $(\alpha_P,L_P)$-smooth.
\end{proposition}

\begin{proof}
The case $\varepsilon=0$ is immediate. Suppose $0<\varepsilon<1$, and choose
distinct $x_L,x_H\in\mc{X}$. Let $P_X=Q_X$ with
$Q_X(x_L)=\varepsilon$ and $Q_X(x_H)=1-\varepsilon$, so the density ratio is
identically one. Let $\gamma:=1/(4\max\{1,\kappa\})$. For
$b\in\{0,1\}$, define
\[
\eta_P(x_L)=\frac12,\quad \eta_Q(x_L)=b,
\qquad
\eta_P(x_H)=\frac12+\kappa\gamma,\quad
\eta_Q(x_H)=\frac12+\gamma.
\]
Thus the source distribution $P$ is the same for both choices of $b$.

At $x_L$, $\mc{M}(x_L;P,Q)=0$, while at $x_H$,
$\mc{M}(x_H;P,Q)=\kappa$. Hence $\mc{A}_\kappa(P,Q)=\{x_H\}$ on the support of
$Q_X$, and therefore $Q_X(A_\kappa)=1-\varepsilon$. Thus
$(P,Q)\in\Pi(\varepsilon,\kappa,\lambda)$.

Moreover, $|\eta_Q(x)-1/2|\ge \gamma$ $Q_X$-almost surely, so $Q$ satisfies
the $\beta$-margin condition for every $\beta\ge0$, for example with
$C_\beta=\gamma^{-\beta}$. Since $P_X=Q_X$ has finite support with positive
mass at both support points, $\eta_Q$ and $\eta_P$ satisfy the required
measure-based smoothness conditions for any $\alpha,\alpha_P>0$, with
suitable finite $L,L_P$. Hence both choices of $b$ belong to
$\Pi_{\rm reg}(\varepsilon,\kappa,\lambda)$.

Finally, because $P$ is independent of $b$, a corpus
$\mc{D}_m\sim P^{\otimes m}$ contains no information about $b$. At $x_L$, the
target Bayes label is $b$ and the target label is deterministic. Averaging
over $b\sim\operatorname{Unif}\{0,1\}$,
\[
\mathbb E_b\mathbb E_{\mc{D}_m}\!\left[\mc{E}_Q(\hat{f})\right]
\ge
\varepsilon\,
\mathbb E_b\!\left[\mathbf 1\{\hat{f}(x_L)\neq b\}\right]
=
\frac{\varepsilon}{2}.
\]
Therefore at least one choice of $b$ yields target excess risk at least
$\varepsilon/2$. Taking the supremum over $(P,Q)$ and then the infimum over
$\hat{f}$ proves the result. The case $\varepsilon=1$ follows from the
same construction with all target mass on $x_L$.
\end{proof}

\section{Experimental details}
\label{app:exp-details}

\paragraph{Hardware and runtime.}
TabPFN inference and the Qwen3.5-4B sentiment experiments were run on NVIDIA A40, A100 or V100 GPUs; the Qwen3.5-27B sentiment experiments were run on the same GPU cluster. %
Each synthetic configuration and each Qwen3.5-4B configuration takes less than one hour of wall-clock time; the Qwen3.5-27B runs are correspondingly slower.

\providecommand{\addstd}[1]{%
    \pgfplotstablecreatecol[create col/expr={sqrt(((\thisrow{r1}-\thisrow{mean})^2+(\thisrow{r2}-\thisrow{mean})^2+(\thisrow{r3}-\thisrow{mean})^2+(\thisrow{r4}-\thisrow{mean})^2+(\thisrow{r5}-\thisrow{mean})^2)/4)}]{std}{#1}%
}
\providecommand{\meanstd}[2]{%
    \pgfplotstablegetelem{#2}{mean}\of#1\edef\msMean{\pgfplotsretval}%
    \pgfplotstablegetelem{#2}{std}\of#1\edef\msStd{\pgfplotsretval}%
    \pgfmathprintnumber[fixed,precision=3,fixed zerofill]{\msMean}\,$\pm$\,\pgfmathprintnumber[fixed,precision=3,fixed zerofill]{\msStd}%
}

\pgfplotstableread[col sep=space]{
n    r1      r2      r3      r4      r5
4    0.4500  0.3650  0.4450  0.4725  0.2425
16   0.2250  0.1175  0.3200  0.2475  0.2325
32   0.2175  0.1450  0.1400  0.1950  0.1275
128  0.0675  0.0825  0.0700  0.0950  0.0500
256  0.0350  0.0275  0.0475  0.0300  0.0500
}\dEightTabICL
\pgfplotstablecreatecol[create col/expr={(\thisrow{r1}+\thisrow{r2}+\thisrow{r3}+\thisrow{r4}+\thisrow{r5})/5}]{mean}{\dEightTabICL}

\pgfplotstableread[col sep=space]{
n    r1      r2      r3      r4      r5
4    0.0750  0.0975  0.0575  0.1000  0.0775
16   0.0425  0.0525  0.0400  0.0475  0.0550
32   0.0350  0.0550  0.0425  0.0525  0.0525
128  0.0300  0.0375  0.0300  0.0250  0.0300
256  0.0175  0.0075  0.0025  0.0100  0.0225
}\dEightTabICR
\pgfplotstablecreatecol[create col/expr={(\thisrow{r1}+\thisrow{r2}+\thisrow{r3}+\thisrow{r4}+\thisrow{r5})/5}]{mean}{\dEightTabICR}

\pgfplotstableread[col sep=space]{
n    r1      r2      r3      r4      r5
4    0.4200  0.3550  0.2850  0.4900  0.1975
16   0.2700  0.3025  0.2575  0.2375  0.2925
32   0.2950  0.2350  0.1600  0.2625  0.1975
128  0.1075  0.0875  0.1525  0.1500  0.1150
256  0.0875  0.1450  0.0950  0.0750  0.1075
}\dEightKNNICL
\pgfplotstablecreatecol[create col/expr={(\thisrow{r1}+\thisrow{r2}+\thisrow{r3}+\thisrow{r4}+\thisrow{r5})/5}]{mean}{\dEightKNNICL}

\pgfplotstableread[col sep=space]{
n    r1      r2      r3      r4      r5
4    0.0850  0.0875  0.0775  0.1000  0.0725
16   0.0325  0.0575  0.0325  0.0400  0.0425
32   0.0375  0.0575  0.0225  0.0400  0.0375
128  0.0300  0.0400  0.0200  0.0225  0.0200
256  0.0275  0.0400  0.0175  0.0200  0.0275
}\dEightKNNICR
\pgfplotstablecreatecol[create col/expr={(\thisrow{r1}+\thisrow{r2}+\thisrow{r3}+\thisrow{r4}+\thisrow{r5})/5}]{mean}{\dEightKNNICR}

\pgfplotstableread[col sep=space]{
n    r1      r2      r3      r4      r5
4    0.5200  0.5075  0.4825  0.4475  0.3850
16   0.4475  0.2750  0.3125  0.4750  0.4000
32   0.2750  0.2950  0.2550  0.2325  0.2700
128  0.0975  0.0900  0.1150  0.0850  0.1050
256  0.0650  0.0675  0.0675  0.0500  0.0825
}\dSixteenTabICL
\pgfplotstablecreatecol[create col/expr={(\thisrow{r1}+\thisrow{r2}+\thisrow{r3}+\thisrow{r4}+\thisrow{r5})/5}]{mean}{\dSixteenTabICL}

\pgfplotstableread[col sep=space]{
n    r1      r2      r3      r4      r5
4    0.1400  0.1450  0.1575  0.1500  0.1000
16   0.1050  0.0850  0.1150  0.0950  0.1050
32   0.1075  0.0725  0.1100  0.1125  0.0950
128  0.0950  0.0575  0.1025  0.0800  0.0625
256  0.0575  0.0550  0.0750  0.0700  0.0525
}\dSixteenTabICR
\pgfplotstablecreatecol[create col/expr={(\thisrow{r1}+\thisrow{r2}+\thisrow{r3}+\thisrow{r4}+\thisrow{r5})/5}]{mean}{\dSixteenTabICR}

\pgfplotstableread[col sep=space]{
n    r1      r2      r3      r4      r5
4    0.3625  0.4550  0.4825  0.3750  0.3100
16   0.4125  0.4175  0.3300  0.3600  0.3500
32   0.4150  0.3450  0.2825  0.2675  0.3250
128  0.2275  0.1575  0.2300  0.2150  0.1475
256  0.1175  0.1375  0.1250  0.1475  0.1600
}\dSixteenKNNICL
\pgfplotstablecreatecol[create col/expr={(\thisrow{r1}+\thisrow{r2}+\thisrow{r3}+\thisrow{r4}+\thisrow{r5})/5}]{mean}{\dSixteenKNNICL}

\pgfplotstableread[col sep=space]{
n    r1      r2      r3      r4      r5
4    0.1525  0.1600  0.1800  0.1400  0.1100
16   0.0825  0.1025  0.0950  0.0925  0.0725
32   0.0800  0.0550  0.0725  0.0700  0.0500
128  0.0575  0.0275  0.0425  0.0550  0.0400
256  0.0500  0.0300  0.0325  0.0500  0.0425
}\dSixteenKNNICR
\pgfplotstablecreatecol[create col/expr={(\thisrow{r1}+\thisrow{r2}+\thisrow{r3}+\thisrow{r4}+\thisrow{r5})/5}]{mean}{\dSixteenKNNICR}

\pgfplotstableread[col sep=space]{
epsilon  r1      r2      r3      r4      r5
0.0      0.4450  0.3650  0.4450  0.4725  0.2425
0.1      0.3175  0.4575  0.4475  0.3825  0.4500
0.2      0.4575  0.2975  0.3925  0.4200  0.4500
0.3      0.5225  0.4550  0.5350  0.4200  0.3550
0.4      0.5175  0.3600  0.4350  0.4200  0.3225
0.5      0.4600  0.3675  0.4800  0.4875  0.2725
0.6      0.3950  0.4675  0.4800  0.4950  0.4050
0.7      0.3950  0.4375  0.4650  0.3100  0.4800
0.8      0.2650  0.3350  0.4650  0.3800  0.4100
0.9      0.4150  0.4800  0.2500  0.2675  0.5100
1.0      0.4175  0.4500  0.4175  0.4725  0.4400
}\dEightDriftTabICL
\pgfplotstablecreatecol[create col/expr={(\thisrow{r1}+\thisrow{r2}+\thisrow{r3}+\thisrow{r4}+\thisrow{r5})/5}]{mean}{\dEightDriftTabICL}

\pgfplotstableread[col sep=space]{
epsilon  r1      r2      r3      r4      r5
0.0      0.0775  0.0975  0.0575  0.1000  0.0775
0.1      0.1100  0.1250  0.1100  0.1450  0.1450
0.2      0.2250  0.2200  0.1875  0.2650  0.2650
0.3      0.3375  0.2925  0.2800  0.3450  0.3850
0.4      0.4175  0.3750  0.3600  0.4100  0.4375
0.5      0.5325  0.4475  0.4650  0.4925  0.5400
0.6      0.6150  0.5475  0.5550  0.5625  0.6225
0.7      0.6900  0.6550  0.6525  0.6650  0.6900
0.8      0.7975  0.7300  0.7600  0.7475  0.7650
0.9      0.8650  0.8375  0.8725  0.8525  0.8650
1.0      0.9250  0.9025  0.9425  0.9000  0.9200
}\dEightDriftTabICR
\pgfplotstablecreatecol[create col/expr={(\thisrow{r1}+\thisrow{r2}+\thisrow{r3}+\thisrow{r4}+\thisrow{r5})/5}]{mean}{\dEightDriftTabICR}

\pgfplotstableread[col sep=space]{
epsilon  r1      r2      r3      r4      r5
0.0      0.4200  0.3550  0.2850  0.4900  0.1975
0.1      0.3425  0.2425  0.3525  0.4075  0.4225
0.2      0.3700  0.3025  0.4250  0.4025  0.3625
0.3      0.5225  0.2200  0.5350  0.3500  0.3150
0.4      0.3675  0.2775  0.3825  0.3725  0.2650
0.5      0.4225  0.3700  0.4475  0.4225  0.2375
0.6      0.2900  0.4050  0.4550  0.4775  0.3625
0.7      0.3950  0.2925  0.4650  0.3200  0.5075
0.8      0.2550  0.2475  0.4650  0.2500  0.4425
0.9      0.4125  0.3125  0.2775  0.3225  0.3850
1.0      0.3775  0.3900  0.3775  0.4600  0.3850
}\dEightDriftKNNICL
\pgfplotstablecreatecol[create col/expr={(\thisrow{r1}+\thisrow{r2}+\thisrow{r3}+\thisrow{r4}+\thisrow{r5})/5}]{mean}{\dEightDriftKNNICL}

\pgfplotstableread[col sep=space]{
epsilon  r1      r2      r3      r4      r5
0.0      0.0850  0.0875  0.0775  0.1000  0.0725
0.1      0.1300  0.1175  0.1050  0.1550  0.1300
0.2      0.2350  0.2000  0.1675  0.2825  0.2400
0.3      0.3350  0.2750  0.2675  0.3400  0.3725
0.4      0.4450  0.3750  0.3525  0.4200  0.4325
0.5      0.5600  0.4550  0.4700  0.5100  0.5350
0.6      0.6300  0.5575  0.5550  0.5600  0.6050
0.7      0.6950  0.6550  0.6725  0.6450  0.6975
0.8      0.8075  0.7575  0.7675  0.7425  0.7675
0.9      0.8825  0.8625  0.8825  0.8350  0.8675
1.0      0.9325  0.9025  0.9400  0.8875  0.9325
}\dEightDriftKNNICR
\pgfplotstablecreatecol[create col/expr={(\thisrow{r1}+\thisrow{r2}+\thisrow{r3}+\thisrow{r4}+\thisrow{r5})/5}]{mean}{\dEightDriftKNNICR}

\pgfplotstableread[col sep=space]{
epsilon  r1      r2      r3      r4      r5
0.0      0.5200  0.5075  0.4825  0.4475  0.3825
0.1      0.4775  0.4375  0.3425  0.5400  0.4375
0.2      0.4775  0.4650  0.3925  0.4600  0.5275
0.3      0.5125  0.4250  0.4775  0.4575  0.5125
0.4      0.4300  0.4650  0.5000  0.4600  0.4850
0.5      0.4950  0.4825  0.3825  0.5325  0.5075
0.6      0.5275  0.4550  0.5175  0.4475  0.4050
0.7      0.3825  0.4100  0.5050  0.4875  0.4050
0.8      0.4800  0.3750  0.4250  0.4425  0.4500
0.9      0.4425  0.2675  0.4975  0.5100  0.4875
1.0      0.5200  0.2900  0.3600  0.4850  0.3275
}\dSixteenDriftTabICL
\pgfplotstablecreatecol[create col/expr={(\thisrow{r1}+\thisrow{r2}+\thisrow{r3}+\thisrow{r4}+\thisrow{r5})/5}]{mean}{\dSixteenDriftTabICL}

\pgfplotstableread[col sep=space]{
epsilon  r1      r2      r3      r4      r5
0.0      0.1400  0.1475  0.1575  0.1500  0.1000
0.1      0.1900  0.1950  0.2200  0.1875  0.1325
0.2      0.2750  0.2850  0.2825  0.2725  0.2050
0.3      0.3500  0.3450  0.3350  0.3575  0.2925
0.4      0.4125  0.4325  0.4100  0.4350  0.3900
0.5      0.5300  0.5200  0.4575  0.5100  0.4625
0.6      0.5925  0.6000  0.5325  0.5925  0.5450
0.7      0.6700  0.6675  0.6200  0.6675  0.6500
0.8      0.7500  0.7250  0.7475  0.7375  0.7450
0.9      0.8150  0.8225  0.8100  0.8000  0.8475
1.0      0.8575  0.8525  0.8425  0.8500  0.9025
}\dSixteenDriftTabICR
\pgfplotstablecreatecol[create col/expr={(\thisrow{r1}+\thisrow{r2}+\thisrow{r3}+\thisrow{r4}+\thisrow{r5})/5}]{mean}{\dSixteenDriftTabICR}

\pgfplotstableread[col sep=space]{
epsilon  r1      r2      r3      r4      r5
0.0      0.3625  0.4550  0.4825  0.3750  0.3100
0.1      0.4100  0.4650  0.3775  0.5400  0.4075
0.2      0.4775  0.4650  0.4000  0.4600  0.4525
0.3      0.4875  0.3375  0.4150  0.3950  0.5125
0.4      0.4475  0.4500  0.3875  0.4600  0.3975
0.5      0.5150  0.4475  0.4300  0.4575  0.4900
0.6      0.4975  0.3375  0.5175  0.3550  0.4375
0.7      0.3875  0.4650  0.4875  0.4350  0.4050
0.8      0.4375  0.2875  0.3725  0.3550  0.3975
0.9      0.4150  0.3200  0.4700  0.4925  0.4100
1.0      0.4650  0.2750  0.3725  0.3675  0.3850
}\dSixteenDriftKNNICL
\pgfplotstablecreatecol[create col/expr={(\thisrow{r1}+\thisrow{r2}+\thisrow{r3}+\thisrow{r4}+\thisrow{r5})/5}]{mean}{\dSixteenDriftKNNICL}

\pgfplotstableread[col sep=space]{
epsilon  r1      r2      r3      r4      r5
0.0      0.1525  0.1600  0.1800  0.1400  0.1100
0.1      0.1925  0.2000  0.2000  0.1850  0.1425
0.2      0.2775  0.2850  0.2700  0.2800  0.2225
0.3      0.3575  0.3800  0.3350  0.3750  0.3175
0.4      0.4175  0.4600  0.4075  0.4200  0.4325
0.5      0.5200  0.4900  0.4550  0.5000  0.4875
0.6      0.5925  0.5800  0.5525  0.5975  0.5825
0.7      0.6750  0.6525  0.6575  0.6550  0.6825
0.8      0.7725  0.7300  0.7625  0.7525  0.7775
0.9      0.8300  0.8150  0.8175  0.8075  0.8575
1.0      0.8675  0.8500  0.8425  0.8475  0.9025
}\dSixteenDriftKNNICR
\pgfplotstablecreatecol[create col/expr={(\thisrow{r1}+\thisrow{r2}+\thisrow{r3}+\thisrow{r4}+\thisrow{r5})/5}]{mean}{\dSixteenDriftKNNICR}

\xdef\dEightDriftTabICLFlatSum{0}
\foreach \rownum in {0,...,10}{%
    \pgfplotstablegetelem{\rownum}{mean}\of\dEightDriftTabICL
    \edef\tmpval{\pgfplotsretval}%
    \pgfmathparse{\dEightDriftTabICLFlatSum + \tmpval}%
    \xdef\dEightDriftTabICLFlatSum{\pgfmathresult}%
}
\pgfmathsetmacro\dEightDriftTabICLFlat{\dEightDriftTabICLFlatSum / 11}

\xdef\dEightDriftKNNICLFlatSum{0}
\foreach \rownum in {0,...,10}{%
    \pgfplotstablegetelem{\rownum}{mean}\of\dEightDriftKNNICL
    \edef\tmpval{\pgfplotsretval}%
    \pgfmathparse{\dEightDriftKNNICLFlatSum + \tmpval}%
    \xdef\dEightDriftKNNICLFlatSum{\pgfmathresult}%
}
\pgfmathsetmacro\dEightDriftKNNICLFlat{\dEightDriftKNNICLFlatSum / 11}

\xdef\dSixteenDriftTabICLFlatSum{0}
\foreach \rownum in {0,...,10}{%
    \pgfplotstablegetelem{\rownum}{mean}\of\dSixteenDriftTabICL
    \edef\tmpval{\pgfplotsretval}%
    \pgfmathparse{\dSixteenDriftTabICLFlatSum + \tmpval}%
    \xdef\dSixteenDriftTabICLFlatSum{\pgfmathresult}%
}
\pgfmathsetmacro\dSixteenDriftTabICLFlat{\dSixteenDriftTabICLFlatSum / 11}

\xdef\dSixteenDriftKNNICLFlatSum{0}
\foreach \rownum in {0,...,10}{%
    \pgfplotstablegetelem{\rownum}{mean}\of\dSixteenDriftKNNICL
    \edef\tmpval{\pgfplotsretval}%
    \pgfmathparse{\dSixteenDriftKNNICLFlatSum + \tmpval}%
    \xdef\dSixteenDriftKNNICLFlatSum{\pgfmathresult}%
}
\pgfmathsetmacro\dSixteenDriftKNNICLFlat{\dSixteenDriftKNNICLFlatSum / 11}

\addstd{\dEightTabICL}\addstd{\dEightTabICR}\addstd{\dEightKNNICL}\addstd{\dEightKNNICR}
\addstd{\dSixteenTabICL}\addstd{\dSixteenTabICR}\addstd{\dSixteenKNNICL}\addstd{\dSixteenKNNICR}
\addstd{\dEightDriftTabICL}\addstd{\dEightDriftTabICR}\addstd{\dEightDriftKNNICL}\addstd{\dEightDriftKNNICR}
\addstd{\dSixteenDriftTabICL}\addstd{\dSixteenDriftTabICR}\addstd{\dSixteenDriftKNNICL}\addstd{\dSixteenDriftKNNICR}

\subsection{Synthetic experiments: linear}
\label{app:linear}

\paragraph{Data-generating process.}
Features are drawn from $X\sim\mathrm{Uniform}\{-500,\ldots,500\}^d$ (integer grid). The target class-probability function is $\eta_Q(x)=\indicator{\vw^\top x>0}$ with $\vw\sim\mc{N}(0,I_d)$ fixed per trial, so $R^*_Q=0$. Posterior drift is induced by a region $B=\{x:x_1\le\tau\}$, with $\tau$ chosen so that $Q_X(B)\approx\varepsilon$; inside $B$ we set $\eta_P=1-\eta_Q$ (labels flipped), and $\eta_P=\eta_Q$ outside $B$. The parameter $\varepsilon\in[0,1]$ is therefore the fraction of $Q_X$-mass on which source and target labels disagree. The ICR corpus has $m=20{,}000$ i.i.d.\ samples from $P$; the test set has $n_{\mathrm{test}}=400$ samples from $Q$. ICL samples $n$ i.i.d.\ demonstrations from $Q$, shared across all test queries; ICR retrieves the $n$ nearest corpus neighbors of each query $X_q$ by Euclidean distance.

\paragraph{Models.}
\textbf{TabPFN}~\citep{hollmann2025accurate} is a pre-trained tabular in-context learner that takes labeled demonstrations as context. \textbf{$k$}-NN uses cross-validation on the demonstrations to pick $k$ in the ICL setting; in the ICR setting it takes a full majority vote over the retrieved neighbors (no inner $k$). All synthetic results are averaged over 5 independent trials with different draws of $\vw$ and the data.

We conduct two experiments.
Experiment~1 fixes $P=Q$ and varies the number of demonstrations $n$ (for $d=8,16$), and Experiment~2 fixes $n=4$ and varies the posterior drift $\varepsilon$. For each we show the summary figure and report the complete per-cell results (mean$\pm$std over 5 trials, test error on the target $Q$; lower is better); the mean of each cell is the value plotted in the corresponding figure, and the std columns quantify the variation across the 5 trials.

\subsubsection{Experiment 1: no domain shift ($P=Q$), varying $n$ and $d$}

\Cref{fig:icl_icr_tabpfn_app} shows that with no distributional shift, ICR outperforms ICL at every $n$, especially in the few-shot regime. The gap narrows as $n$ grows, which aligns with our finite-sample theory. TabPFN and $k$-NN track each other closely, supporting $k$-NN as a good proxy for the in-context learner.

\begin{figure}[htbp]
\centering
\vspace{-7pt}
\begin{tikzpicture}
\begin{groupplot}[
    group style={
        group size=2 by 1,
        horizontal sep=2.2cm,
        group name=myplots,
    },
    width=0.4\textwidth,
    height=0.32\textwidth,
    xmode=log,
    log basis x=2,
    xtick={4,16,32,128,256},
    xticklabels={4,16,32,128,256},
    ymin=0, ymax=0.6,
    ytick={0,0.1,0.2,0.3,0.4,0.5,0.6},
    xlabel={$n$ (demonstrations)},
    ylabel={Test error on $Q$},
    grid=major,
    grid style={gray!25},
    tick label style={font=\small},
    label style={font=\small},
]

\nextgroupplot[title={$d=8$},
    legend to name=sharedlegendLinOne,
    legend columns=4,
    legend style={font=\small, column sep=0.8em, legend cell align=left},
]

\addplot[blue, solid, thick, mark=*, mark size=2pt]
    table[x=n, y=mean] {\dEightTabICL};
\addlegendentry{TabPFN-ICL}

\addplot[orange!90!black, solid, thick, mark=*, mark size=2pt]
    table[x=n, y=mean] {\dEightTabICR};
\addlegendentry{TabPFN-ICR}

\addplot[blue, dashed, thick, mark=square, mark size=2pt]
    table[x=n, y=mean] {\dEightKNNICL};
\addlegendentry{$k$-NN-ICL}

\addplot[orange!90!black, dashed, thick, mark=square, mark size=2pt]
    table[x=n, y=mean] {\dEightKNNICR};
\addlegendentry{$k$-NN-ICR}

\nextgroupplot[title={$d=16$}, ylabel={}]

\addplot[blue, solid, thick, mark=*, mark size=2pt]
    table[x=n, y=mean] {\dSixteenTabICL};

\addplot[orange!90!black, solid, thick, mark=*, mark size=2pt]
    table[x=n, y=mean] {\dSixteenTabICR};

\addplot[blue, dashed, thick, mark=square, mark size=2pt]
    table[x=n, y=mean] {\dSixteenKNNICL};

\addplot[orange!90!black, dashed, thick, mark=square, mark size=2pt]
    table[x=n, y=mean] {\dSixteenKNNICR};

\end{groupplot}

\node[anchor=north, yshift=-0.8cm]
    at ($(myplots c1r1.south)!0.5!(myplots c2r1.south)$)
    {\pgfplotslegendfromname{sharedlegendLinOne}};

\end{tikzpicture}
\caption{Test error on $Q$ vs.\ number of demonstrations $n$, with $P=Q$
(no domain shift). \emph{Color}: ICL (blue) vs.\ ICR (orange).
\emph{Line style}: TabPFN (solid, \textbullet) vs.\ $k$-NN (dashed, $\square$).
Corpus size $m=20{,}000$, $n_{\text{test}}=400$. Results averaged over 5 independent runs.}
\label{fig:icl_icr_tabpfn_app}
\vspace{-5pt}
\end{figure}
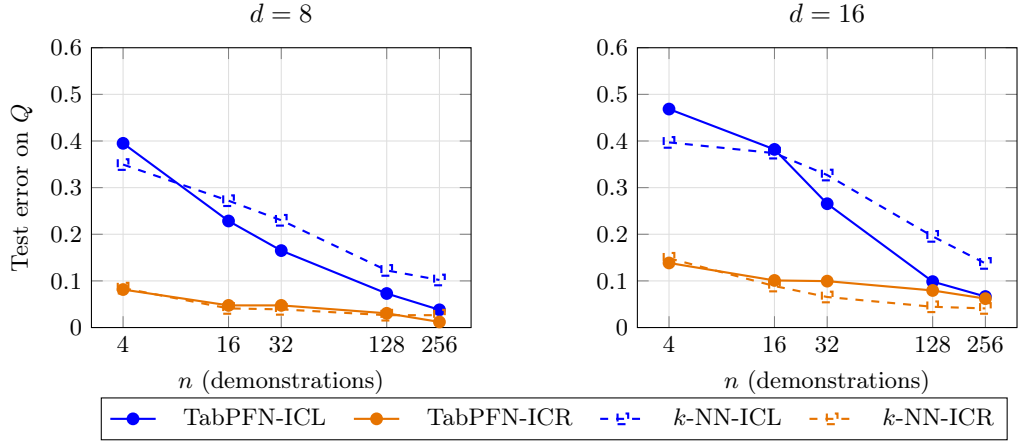

\begin{table}[htbp]
\caption{Experiment 1, $d=8$: test error on $Q$, we report mean$\pm$std over 5 trials.}
\label{tab:exp1-d8}
\vskip 0.1in
\centering\small\setlength{\tabcolsep}{4pt}
\begin{tabular}{rcccc}
\toprule
$n$ & TabPFN-ICL & TabPFN-ICR & $k$-NN-ICL & $k$-NN-ICR \\
\midrule
4   & \meanstd{\dEightTabICL}{0} & \meanstd{\dEightTabICR}{0} & \meanstd{\dEightKNNICL}{0} & \meanstd{\dEightKNNICR}{0} \\
16  & \meanstd{\dEightTabICL}{1} & \meanstd{\dEightTabICR}{1} & \meanstd{\dEightKNNICL}{1} & \meanstd{\dEightKNNICR}{1} \\
32  & \meanstd{\dEightTabICL}{2} & \meanstd{\dEightTabICR}{2} & \meanstd{\dEightKNNICL}{2} & \meanstd{\dEightKNNICR}{2} \\
128 & \meanstd{\dEightTabICL}{3} & \meanstd{\dEightTabICR}{3} & \meanstd{\dEightKNNICL}{3} & \meanstd{\dEightKNNICR}{3} \\
256 & \meanstd{\dEightTabICL}{4} & \meanstd{\dEightTabICR}{4} & \meanstd{\dEightKNNICL}{4} & \meanstd{\dEightKNNICR}{4} \\
\bottomrule
\end{tabular}
\end{table}

\begin{table}[htbp]
\caption{Experiment 1, $d=16$: test error on $Q$, we report mean$\pm$std over 5 trials.}
\label{tab:exp1-d16}
\centering\small\setlength{\tabcolsep}{4pt}
\vskip 0.1in
\begin{tabular}{rcccc}
\toprule
$n$ & TabPFN-ICL & TabPFN-ICR & $k$-NN-ICL & $k$-NN-ICR \\
\midrule
4   & \meanstd{\dSixteenTabICL}{0} & \meanstd{\dSixteenTabICR}{0} & \meanstd{\dSixteenKNNICL}{0} & \meanstd{\dSixteenKNNICR}{0} \\
16  & \meanstd{\dSixteenTabICL}{1} & \meanstd{\dSixteenTabICR}{1} & \meanstd{\dSixteenKNNICL}{1} & \meanstd{\dSixteenKNNICR}{1} \\
32  & \meanstd{\dSixteenTabICL}{2} & \meanstd{\dSixteenTabICR}{2} & \meanstd{\dSixteenKNNICL}{2} & \meanstd{\dSixteenKNNICR}{2} \\
128 & \meanstd{\dSixteenTabICL}{3} & \meanstd{\dSixteenTabICR}{3} & \meanstd{\dSixteenKNNICL}{3} & \meanstd{\dSixteenKNNICR}{3} \\
256 & \meanstd{\dSixteenTabICL}{4} & \meanstd{\dSixteenTabICR}{4} & \meanstd{\dSixteenKNNICL}{4} & \meanstd{\dSixteenKNNICR}{4} \\
\bottomrule
\end{tabular}
\end{table}

\subsubsection{Experiment 2: posterior drift ($\eta_P \neq \eta_Q$), $n=4$, varying $\varepsilon$}

\Cref{fig:icl_icr_drift_n4_app} shows that, fixing $n=4$ and sweeping $\varepsilon\in[0,1]$, ICR error increases monotonically from $\approx 8\%$ at $\varepsilon=0$ to $\approx 93\%$ at $\varepsilon=1$, while ICL error does not depend on $\varepsilon$. The crossover $\varepsilon^*\approx 0.3$ for $d=8$ marks the bias--variance tipping point predicted by the theory: below $\varepsilon^*$ ICR's variance reduction outweighs the bias from posterior drift; above $\varepsilon^*$, the irreducible error dominates.

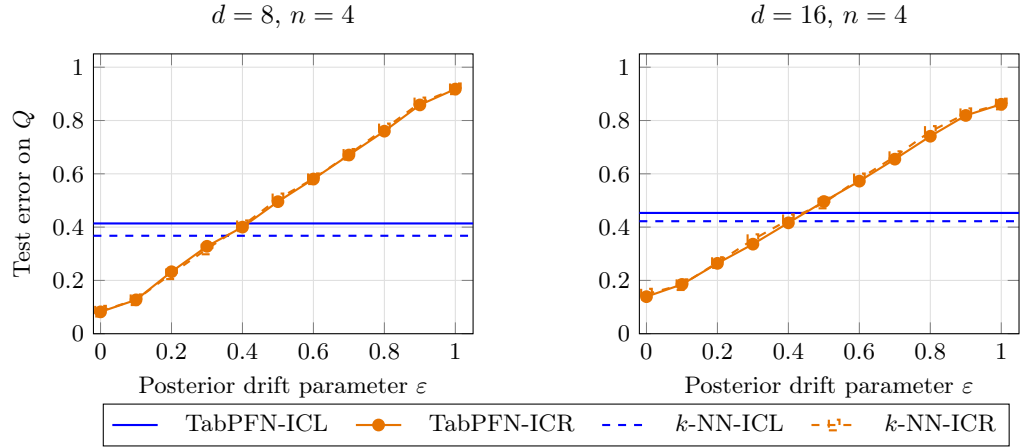
\begin{figure}[htbp]
\vspace{-7pt}
\centering
\begin{tikzpicture}
\begin{groupplot}[
    group style={
        group size=2 by 1,
        horizontal sep=2.2cm,
        group name=myplots,
    },
    width=0.4\textwidth,
    height=0.32\textwidth,
    xmin=-0.02, xmax=1.05,
    xtick={0,0.2,0.4,0.6,0.8,1.0},
    ymin=0, ymax=1.05,
    ytick={0,0.2,0.4,0.6,0.8,1.0},
    xlabel={Posterior drift parameter $\varepsilon$},
    ylabel={Test error on $Q$},
    grid=major,
    grid style={gray!25},
    tick label style={font=\small},
    label style={font=\small},
]

\nextgroupplot[title={$d=8$, $n=4$},
    legend to name=sharedlegendLinTwo,
    legend columns=4,
    legend style={font=\small, column sep=0.8em, legend cell align=left},
]

\addplot[blue, solid, thick, mark=none]
    coordinates {(-0.02,\dEightDriftTabICLFlat)(1.05,\dEightDriftTabICLFlat)};
\addlegendentry{TabPFN-ICL}

\addplot[orange!90!black, solid, thick, mark=*, mark size=2pt]
    table[x=epsilon, y=mean] {\dEightDriftTabICR};
\addlegendentry{TabPFN-ICR}

\addplot[blue, dashed, thick, mark=none]
    coordinates {(-0.02,\dEightDriftKNNICLFlat)(1.05,\dEightDriftKNNICLFlat)};
\addlegendentry{$k$-NN-ICL}

\addplot[orange!90!black, dashed, thick, mark=square, mark size=2pt]
    table[x=epsilon, y=mean] {\dEightDriftKNNICR};
\addlegendentry{$k$-NN-ICR}

\nextgroupplot[title={$d=16$, $n=4$}, ylabel={}]

\addplot[blue, solid, thick, mark=none]
    coordinates {(-0.02,\dSixteenDriftTabICLFlat)(1.05,\dSixteenDriftTabICLFlat)};

\addplot[orange!90!black, solid, thick, mark=*, mark size=2pt]
    table[x=epsilon, y=mean] {\dSixteenDriftTabICR};

\addplot[blue, dashed, thick, mark=none]
    coordinates {(-0.02,\dSixteenDriftKNNICLFlat)(1.05,\dSixteenDriftKNNICLFlat)};

\addplot[orange!90!black, dashed, thick, mark=square, mark size=2pt]
    table[x=epsilon, y=mean] {\dSixteenDriftKNNICR};

\end{groupplot}

\node[anchor=north, yshift=-0.8cm]
    at ($(myplots c1r1.south)!0.5!(myplots c2r1.south)$)
    {\pgfplotslegendfromname{sharedlegendLinTwo}};

\end{tikzpicture}
\caption{Test error on $Q$ vs.\ posterior drift parameter $\varepsilon$, with $n=4$ fixed. \emph{Color}: ICL (blue) vs.\ ICR (orange).
\emph{Line style}: TabPFN (solid) vs.\ $k$-NN (dashed).
ICL baselines (horizontal) are the means over 5 runs $\times$ 11 $\varepsilon$
values (since $\varepsilon$ only affects the ICR corpus; not affecting ICL). ICR mean error grows monotonically with $\varepsilon$. Results averaged over
5 independent runs.}
\label{fig:icl_icr_drift_n4_app}
\vspace{-14pt}
\end{figure}

\begin{table}[htbp]
\caption{Experiment 2, $d=8$, $n=4$: test error on $Q$ vs.\ posterior drift $\varepsilon$, we report mean$\pm$std over 5 trials.}
\label{tab:exp2-d8}
\vskip 0.1in
\centering\small\setlength{\tabcolsep}{4pt}
\begin{tabular}{ccccc}
\toprule
$\varepsilon$ & TabPFN-ICL & TabPFN-ICR & $k$-NN-ICL & $k$-NN-ICR \\
\midrule
0.0 & \meanstd{\dEightDriftTabICL}{0}  & \meanstd{\dEightDriftTabICR}{0}  & \meanstd{\dEightDriftKNNICL}{0}  & \meanstd{\dEightDriftKNNICR}{0}  \\
0.1 & \meanstd{\dEightDriftTabICL}{1}  & \meanstd{\dEightDriftTabICR}{1}  & \meanstd{\dEightDriftKNNICL}{1}  & \meanstd{\dEightDriftKNNICR}{1}  \\
0.2 & \meanstd{\dEightDriftTabICL}{2}  & \meanstd{\dEightDriftTabICR}{2}  & \meanstd{\dEightDriftKNNICL}{2}  & \meanstd{\dEightDriftKNNICR}{2}  \\
0.3 & \meanstd{\dEightDriftTabICL}{3}  & \meanstd{\dEightDriftTabICR}{3}  & \meanstd{\dEightDriftKNNICL}{3}  & \meanstd{\dEightDriftKNNICR}{3}  \\
0.4 & \meanstd{\dEightDriftTabICL}{4}  & \meanstd{\dEightDriftTabICR}{4}  & \meanstd{\dEightDriftKNNICL}{4}  & \meanstd{\dEightDriftKNNICR}{4}  \\
0.5 & \meanstd{\dEightDriftTabICL}{5}  & \meanstd{\dEightDriftTabICR}{5}  & \meanstd{\dEightDriftKNNICL}{5}  & \meanstd{\dEightDriftKNNICR}{5}  \\
0.6 & \meanstd{\dEightDriftTabICL}{6}  & \meanstd{\dEightDriftTabICR}{6}  & \meanstd{\dEightDriftKNNICL}{6}  & \meanstd{\dEightDriftKNNICR}{6}  \\
0.7 & \meanstd{\dEightDriftTabICL}{7}  & \meanstd{\dEightDriftTabICR}{7}  & \meanstd{\dEightDriftKNNICL}{7}  & \meanstd{\dEightDriftKNNICR}{7}  \\
0.8 & \meanstd{\dEightDriftTabICL}{8}  & \meanstd{\dEightDriftTabICR}{8}  & \meanstd{\dEightDriftKNNICL}{8}  & \meanstd{\dEightDriftKNNICR}{8}  \\
0.9 & \meanstd{\dEightDriftTabICL}{9}  & \meanstd{\dEightDriftTabICR}{9}  & \meanstd{\dEightDriftKNNICL}{9}  & \meanstd{\dEightDriftKNNICR}{9}  \\
1.0 & \meanstd{\dEightDriftTabICL}{10} & \meanstd{\dEightDriftTabICR}{10} & \meanstd{\dEightDriftKNNICL}{10} & \meanstd{\dEightDriftKNNICR}{10} \\
\bottomrule
\end{tabular}
\end{table}

\begin{table}[htbp]
\caption{Experiment 2, $d=16$, $n=4$: test error on $Q$ vs.\ posterior drift $\varepsilon$, we report mean$\pm$std over 5 trials.}
\label{tab:exp2-d16}
\vskip 0.1in
\centering\small\setlength{\tabcolsep}{4pt}
\begin{tabular}{ccccc}
\toprule
$\varepsilon$ & TabPFN-ICL & TabPFN-ICR & $k$-NN-ICL & $k$-NN-ICR \\
\midrule
0.0 & \meanstd{\dSixteenDriftTabICL}{0}  & \meanstd{\dSixteenDriftTabICR}{0}  & \meanstd{\dSixteenDriftKNNICL}{0}  & \meanstd{\dSixteenDriftKNNICR}{0}  \\
0.1 & \meanstd{\dSixteenDriftTabICL}{1}  & \meanstd{\dSixteenDriftTabICR}{1}  & \meanstd{\dSixteenDriftKNNICL}{1}  & \meanstd{\dSixteenDriftKNNICR}{1}  \\
0.2 & \meanstd{\dSixteenDriftTabICL}{2}  & \meanstd{\dSixteenDriftTabICR}{2}  & \meanstd{\dSixteenDriftKNNICL}{2}  & \meanstd{\dSixteenDriftKNNICR}{2}  \\
0.3 & \meanstd{\dSixteenDriftTabICL}{3}  & \meanstd{\dSixteenDriftTabICR}{3}  & \meanstd{\dSixteenDriftKNNICL}{3}  & \meanstd{\dSixteenDriftKNNICR}{3}  \\
0.4 & \meanstd{\dSixteenDriftTabICL}{4}  & \meanstd{\dSixteenDriftTabICR}{4}  & \meanstd{\dSixteenDriftKNNICL}{4}  & \meanstd{\dSixteenDriftKNNICR}{4}  \\
0.5 & \meanstd{\dSixteenDriftTabICL}{5}  & \meanstd{\dSixteenDriftTabICR}{5}  & \meanstd{\dSixteenDriftKNNICL}{5}  & \meanstd{\dSixteenDriftKNNICR}{5}  \\
0.6 & \meanstd{\dSixteenDriftTabICL}{6}  & \meanstd{\dSixteenDriftTabICR}{6}  & \meanstd{\dSixteenDriftKNNICL}{6}  & \meanstd{\dSixteenDriftKNNICR}{6}  \\
0.7 & \meanstd{\dSixteenDriftTabICL}{7}  & \meanstd{\dSixteenDriftTabICR}{7}  & \meanstd{\dSixteenDriftKNNICL}{7}  & \meanstd{\dSixteenDriftKNNICR}{7}  \\
0.8 & \meanstd{\dSixteenDriftTabICL}{8}  & \meanstd{\dSixteenDriftTabICR}{8}  & \meanstd{\dSixteenDriftKNNICL}{8}  & \meanstd{\dSixteenDriftKNNICR}{8}  \\
0.9 & \meanstd{\dSixteenDriftTabICL}{9}  & \meanstd{\dSixteenDriftTabICR}{9}  & \meanstd{\dSixteenDriftKNNICL}{9}  & \meanstd{\dSixteenDriftKNNICR}{9}  \\
1.0 & \meanstd{\dSixteenDriftTabICL}{10} & \meanstd{\dSixteenDriftTabICR}{10} & \meanstd{\dSixteenDriftKNNICL}{10} & \meanstd{\dSixteenDriftKNNICR}{10} \\
\bottomrule
\end{tabular}
\end{table}

\newpage

\subsection{Synthetic experiments: Weierstrass}
\label{app:weierstrass}

\paragraph{Overview.}
All experiments in this subsection replace the linear target with a Weierstrass-type class-probability function $\eta_\alpha$ of controlled smoothness $\alpha$ on $\mc{X}=[-1,1]$, which holds the margin exponent, the Bayes risk, and the Bayes classifier fixed while $\alpha$ varies; the full construction is given below. Experiments~1 and~2 sweep the number of demonstrations $n$ and the posterior drift $\varepsilon$ at fixed smoothness, while Experiments~3 and~4 sweep the smoothness $\alpha$ and the retrieval corpus size $m_{\mathrm{corpus}}$.

\paragraph{Data-generating process.}
We take $\mc{X} = [-1, 1]$ with $X \sim \mathrm{Uniform}[-1, 1]$, so $P_X = Q_X$ and, in one dimension, Chaudhuri--Dasgupta's measure-based smoothness exponent coincides with the pointwise H\"older exponent. The class-probability function is
\begin{align*}
    \eta_\alpha(x)
    &= \tfrac{1}{2} + A \sin(2\pi f x)\bigl(1 + b\, w_\alpha(x)\bigr), \qquad x \in [-1, 1],
\end{align*}
where the smoothness parameter $\alpha$ is encoded in the normalized Weierstrass-style odd sine series
\begin{align*}
    w_\alpha(x)
    &= \frac{\sum_{j=0}^{J} 2^{-j\alpha}\sin(2\pi\,2^{j} x)}{\displaystyle\max_{z \in [-1,1]}\bigl|\sum_{j=0}^{J} 2^{-j\alpha}\sin(2\pi\,2^{j} z)\bigr|}.
\end{align*}
The coefficients $2^{-j\alpha}$ on frequencies $2^j$ give $w_\alpha$ pointwise H\"older exponent $\alpha$ as $J \to \infty$; we truncate at $J = 12$ and normalize by the supremum so that $|w_\alpha(x)| \le 1$ for every $\alpha \in (0, 1]$.

We fix $A = 0.24$, $b = 0.85$, and $f = 2$ across $\alpha$; here $f$ denotes the frequency.

This construction of $\eta$ allows us to control the smoothness easily, see the figure below for examples of $\eta_\alpha$ for different $\alpha$. The Bayes risk is $R_Q^* = \frac{1}{2} - \frac{2A}{\pi} \approx 0.3472$ and the Bayes classifier is $f_Q^*(x) = \indicator{\eta_\alpha(x)\ge\half}$, which is independent of $\alpha$.

\begin{figure}[htbp]
    \centering
    \includegraphics[width=0.8\linewidth]{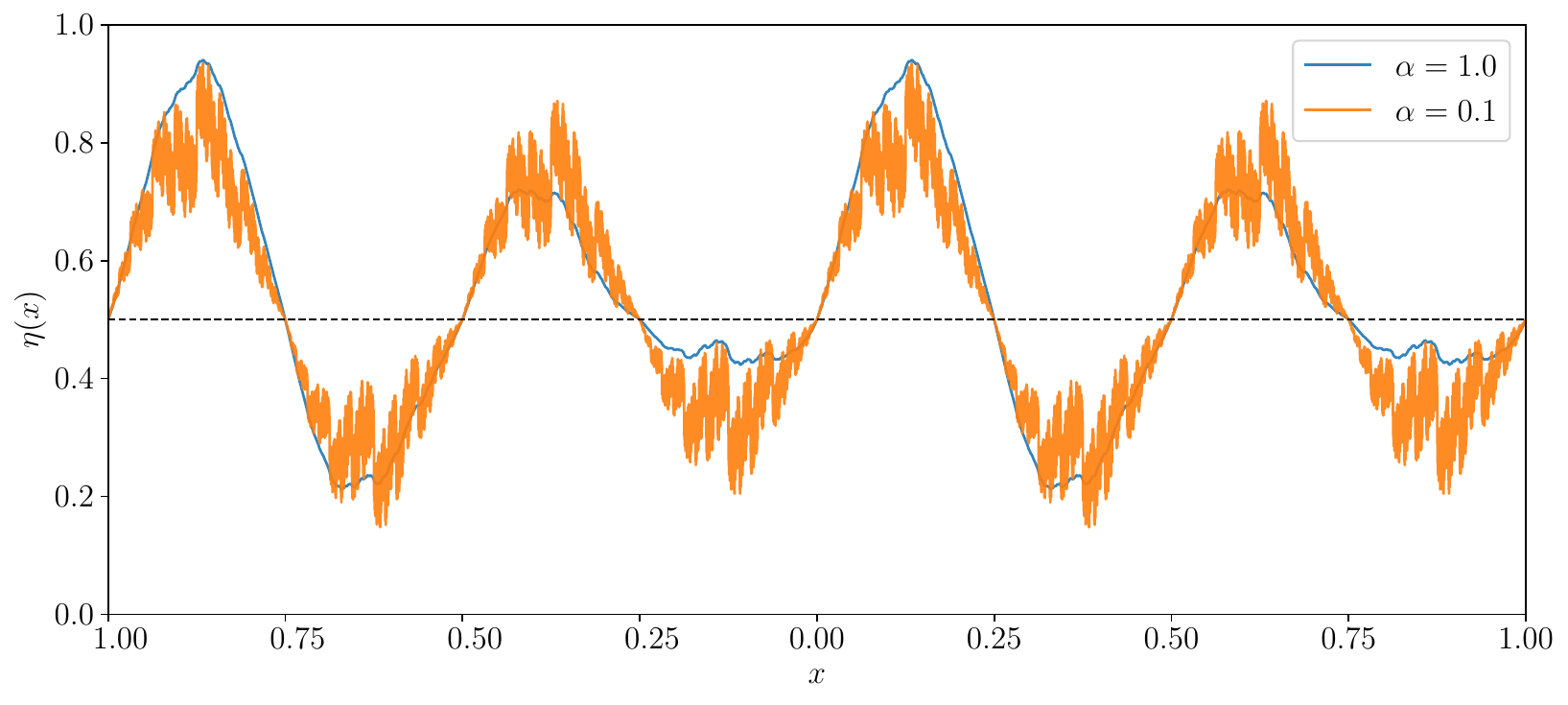}
    \caption{Examples of the Weierstrass class-probability function $\eta_\alpha(x)$ for different values of the smoothness parameter $\alpha$.}
\end{figure}

\paragraph{Models.}
TabPFN takes the labeled demonstrations in context; $k$-NN uses Euclidean distance with the theory-inspired neighborhood size
\begin{align*}
    k_\alpha &= \mathrm{round\_odd}\!\left(c_k \cdot n^{2\alpha/(2\alpha + 1)}\right), \qquad c_k = 0.30,
\end{align*}
which balances the bias/variance terms of \Cref{thm:sanjoy_knn}.
The retrieval corpus has $m_{\mathrm{corpus}} = 10^{5}$ i.i.d.\ samples from $P = Q$. The population excess risk over $Q$ is approximated by evaluating on a deterministic grid of $n_{\mathrm{test}} = 400$ equally spaced points on $[-1, 1]$, averaged over $5$ seeds.

\subsubsection{Experiment 1: no domain shift ($P=Q$), large $m$, varying $n$}
This experiment sweeps the number of demonstrations $n \in \{4,\dots,1024\}$ with no domain shift ($P=Q$), at two smoothness levels $\alpha = 1.0$ (smooth) and $\alpha = 0.5$ (rougher); all remaining parameters follow the common setup above ($R_Q^* = 0.3472$, $m_{\mathrm{corpus}} = 10^{5}$, $n_{\mathrm{test}} = 400$, oracle $k_\alpha$, 5 seeds 42--46). \Cref{fig:icl_icr_weier_n} shows that ICR outperforms ICL at every $n$ and both smoothness levels, and \Cref{tab:weier-n-alpha1,tab:weier-n-alpha05} report the per-cell values.

\providecommand{\addstd}[1]{%
    \pgfplotstablecreatecol[create col/expr={sqrt(((\thisrow{r1}-\thisrow{mean})^2+(\thisrow{r2}-\thisrow{mean})^2+(\thisrow{r3}-\thisrow{mean})^2+(\thisrow{r4}-\thisrow{mean})^2+(\thisrow{r5}-\thisrow{mean})^2)/4)}]{std}{#1}%
}
\providecommand{\meanstd}[2]{%
    \pgfplotstablegetelem{#2}{mean}\of#1\edef\msMean{\pgfplotsretval}%
    \pgfplotstablegetelem{#2}{std}\of#1\edef\msStd{\pgfplotsretval}%
    \pgfmathprintnumber[fixed,precision=4,fixed zerofill]{\msMean}\,$\pm$\,\pgfmathprintnumber[fixed,precision=4,fixed zerofill]{\msStd}%
}

\pgfplotstableread[col sep=space]{
n     r1      r2      r3      r4      r5
4     0.1907  0.1227  0.1813  0.0943  0.1468
16    0.1440  0.1079  0.1832  0.1469  0.0605
32    0.1173  0.0821  0.1066  0.0789  0.0746
64    0.0316  0.0857  0.0544  0.0443  0.0733
128   0.0148  0.0190  0.0259  0.0405  0.0436
256   0.0219  0.0144  0.0244  0.0054  0.0360
512   0.0137  0.0178  0.0304  0.0074  0.0196
1024  0.0131  0.0123  0.0071  0.0137  0.0184
}\WeierOneKNNICL
\pgfplotstablecreatecol[create col/expr={(\thisrow{r1}+\thisrow{r2}+\thisrow{r3}+\thisrow{r4}+\thisrow{r5})/5}]{mean}{\WeierOneKNNICL}
\addstd{\WeierOneKNNICL}

\pgfplotstableread[col sep=space]{
n     r1      r2      r3      r4      r5
4     0.0533  0.0589  0.0558  0.0566  0.0661
16    0.0205  0.0225  0.0254  0.0231  0.0240
32    0.0135  0.0143  0.0141  0.0154  0.0128
64    0.0066  0.0085  0.0086  0.0080  0.0105
128   0.0055  0.0035  0.0048  0.0043  0.0061
256   0.0030  0.0023  0.0022  0.0014  0.0022
512   0.0009  0.0008  0.0012  0.0011  0.0012
1024  0.0005  0.0006  0.0004  0.0003  0.0002
}\WeierOneKNNICR
\pgfplotstablecreatecol[create col/expr={(\thisrow{r1}+\thisrow{r2}+\thisrow{r3}+\thisrow{r4}+\thisrow{r5})/5}]{mean}{\WeierOneKNNICR}
\addstd{\WeierOneKNNICR}

\pgfplotstableread[col sep=space]{
n     r1      r2      r3      r4      r5
4     0.1907  0.1229  0.1866  0.1146  0.1410
16    0.1403  0.1416  0.1692  0.1428  0.1601
32    0.1357  0.1024  0.1624  0.1225  0.1460
64    0.1031  0.1124  0.1377  0.0444  0.1184
128   0.0487  0.0640  0.0722  0.0516  0.0599
256   0.0433  0.0162  0.0440  0.0145  0.0619
512   0.0150  0.0164  0.0421  0.0045  0.0173
1024  0.0054  0.0063  0.0099  0.0050  0.0147
}\WeierOneTabICL
\pgfplotstablecreatecol[create col/expr={(\thisrow{r1}+\thisrow{r2}+\thisrow{r3}+\thisrow{r4}+\thisrow{r5})/5}]{mean}{\WeierOneTabICL}
\addstd{\WeierOneTabICL}

\pgfplotstableread[col sep=space]{
n     r1      r2      r3      r4      r5
4     0.0702  0.0734  0.0670  0.0581  0.0635
16    0.0391  0.0418  0.0418  0.0315  0.0450
32    0.0193  0.0257  0.0216  0.0263  0.0255
64    0.0123  0.0170  0.0174  0.0218  0.0214
128   0.0123  0.0141  0.0084  0.0130  0.0100
256   0.0059  0.0079  0.0066  0.0072  0.0080
512   0.0037  0.0031  0.0026  0.0023  0.0036
1024  0.0010  0.0017  0.0016  0.0013  0.0011
}\WeierOneTabICR
\pgfplotstablecreatecol[create col/expr={(\thisrow{r1}+\thisrow{r2}+\thisrow{r3}+\thisrow{r4}+\thisrow{r5})/5}]{mean}{\WeierOneTabICR}
\addstd{\WeierOneTabICR}

\pgfplotstableread[col sep=space]{
n     r1      r2      r3      r4      r5
4     0.1889  0.1183  0.1742  0.1578  0.1460
16    0.1224  0.0201  0.1809  0.1421  0.1409
32    0.0974  0.0757  0.1204  0.1010  0.0856
64    0.0266  0.0687  0.1276  0.0719  0.0619
128   0.0558  0.0284  0.0774  0.0609  0.0517
256   0.0360  0.0316  0.0627  0.0253  0.0582
512   0.0400  0.0357  0.0484  0.0232  0.0400
1024  0.0271  0.0251  0.0317  0.0283  0.0240
}\WeierHalfKNNICL
\pgfplotstablecreatecol[create col/expr={(\thisrow{r1}+\thisrow{r2}+\thisrow{r3}+\thisrow{r4}+\thisrow{r5})/5}]{mean}{\WeierHalfKNNICL}
\addstd{\WeierHalfKNNICL}

\pgfplotstableread[col sep=space]{
n     r1      r2      r3      r4      r5
4     0.0570  0.0708  0.0649  0.0780  0.0662
16    0.0229  0.0246  0.0222  0.0224  0.0257
32    0.0112  0.0115  0.0148  0.0162  0.0151
64    0.0079  0.0081  0.0074  0.0068  0.0075
128   0.0046  0.0028  0.0034  0.0027  0.0038
256   0.0016  0.0020  0.0016  0.0015  0.0015
512   0.0006  0.0004  0.0010  0.0012  0.0009
1024  0.0003  0.0003  0.0003  0.0003  0.0002
}\WeierHalfKNNICR
\pgfplotstablecreatecol[create col/expr={(\thisrow{r1}+\thisrow{r2}+\thisrow{r3}+\thisrow{r4}+\thisrow{r5})/5}]{mean}{\WeierHalfKNNICR}
\addstd{\WeierHalfKNNICR}

\pgfplotstableread[col sep=space]{
n     r1      r2      r3      r4      r5
4     0.1889  0.1236  0.1828  0.1639  0.1413
16    0.1348  0.1414  0.1688  0.1671  0.1556
32    0.1314  0.0973  0.1644  0.1241  0.1158
64    0.0921  0.1134  0.1430  0.0562  0.1299
128   0.0498  0.0298  0.0842  0.0942  0.0863
256   0.0305  0.0138  0.0545  0.0162  0.0683
512   0.0236  0.0163  0.0399  0.0040  0.0144
1024  0.0042  0.0083  0.0105  0.0037  0.0093
}\WeierHalfTabICL
\pgfplotstablecreatecol[create col/expr={(\thisrow{r1}+\thisrow{r2}+\thisrow{r3}+\thisrow{r4}+\thisrow{r5})/5}]{mean}{\WeierHalfTabICL}
\addstd{\WeierHalfTabICL}

\pgfplotstableread[col sep=space]{
n     r1      r2      r3      r4      r5
4     0.0774  0.0828  0.0813  0.0679  0.0750
16    0.0431  0.0451  0.0426  0.0341  0.0515
32    0.0260  0.0245  0.0233  0.0324  0.0277
64    0.0151  0.0141  0.0212  0.0255  0.0200
128   0.0100  0.0142  0.0075  0.0121  0.0127
256   0.0042  0.0066  0.0068  0.0043  0.0058
512   0.0020  0.0012  0.0026  0.0023  0.0021
1024  0.0008  0.0010  0.0009  0.0015  0.0020
}\WeierHalfTabICR
\pgfplotstablecreatecol[create col/expr={(\thisrow{r1}+\thisrow{r2}+\thisrow{r3}+\thisrow{r4}+\thisrow{r5})/5}]{mean}{\WeierHalfTabICR}
\addstd{\WeierHalfTabICR}

\begin{figure}[htbp]
\centering
\vspace{-7pt}
\begin{tikzpicture}
\begin{groupplot}[
    group style={
        group size=2 by 1,
        horizontal sep=2.2cm,
        group name=weiernplots,
    },
    width=0.42\textwidth,
    height=0.34\textwidth,
    xmode=log,
    ymode=log,
    log basis x=2,
    xtick={4,16,32,64,128,256,512,1024},
    xticklabels={4,16,32,64,128,256,512,1024},
    x tick label style={font=\footnotesize, rotate=45, anchor=north east},
    ymin=1e-4, ymax=0.4,
    xlabel={$n$ (demonstrations)},
    ylabel={Target excess risk},
    grid=major,
    grid style={gray!25},
    tick label style={font=\small},
    label style={font=\small},
]

\nextgroupplot[title={$\alpha = 1.0$ (smoother $\eta$)},
    legend to name=sharedlegendWeier,
    legend columns=4,
    legend style={font=\small, column sep=0.8em, legend cell align=left},
]

\addplot[blue, solid, thick, mark=*, mark size=2pt]
    table[x=n, y=mean] {\WeierOneTabICL};
\addlegendentry{TabPFN-ICL}

\addplot[orange!90!black, solid, thick, mark=*, mark size=2pt]
    table[x=n, y=mean] {\WeierOneTabICR};
\addlegendentry{TabPFN-ICR}

\addplot[blue, dashed, thick, mark=square, mark size=2pt]
    table[x=n, y=mean] {\WeierOneKNNICL};
\addlegendentry{$k$-NN-ICL}

\addplot[orange!90!black, dashed, thick, mark=square, mark size=2pt]
    table[x=n, y=mean] {\WeierOneKNNICR};
\addlegendentry{$k$-NN-ICR}

\nextgroupplot[title={$\alpha = 0.5$ (rougher $\eta$)}, ylabel={}]

\addplot[blue, solid, thick, mark=*, mark size=2pt]
    table[x=n, y=mean] {\WeierHalfTabICL};

\addplot[orange!90!black, solid, thick, mark=*, mark size=2pt]
    table[x=n, y=mean] {\WeierHalfTabICR};

\addplot[blue, dashed, thick, mark=square, mark size=2pt]
    table[x=n, y=mean] {\WeierHalfKNNICL};

\addplot[orange!90!black, dashed, thick, mark=square, mark size=2pt]
    table[x=n, y=mean] {\WeierHalfKNNICR};

\end{groupplot}

\node[anchor=north, yshift=-1.5cm]
    at ($(weiernplots c1r1.south)!0.5!(weiernplots c2r1.south)$)
    {\pgfplotslegendfromname{sharedlegendWeier}};

\end{tikzpicture}
\caption{Target excess risk vs.\ number of demonstrations $n$ on the 1-D
Weierstrass family ($P = Q$; $\beta = 1$; oracle $k_\alpha$). \emph{Color}:
ICL (blue) vs.\ ICR (orange). \emph{Line style}: TabPFN (solid,
\textbullet) vs.\ $k$-NN (dashed, $\square$). Both axes are log-scaled.
Corpus size $m = 10^{5}$, $n_{\text{test}} = 400$; each curve is the mean
over 5 seeds (42--46).}
\label{fig:icl_icr_weier_n}
\vspace{-5pt}
\end{figure}
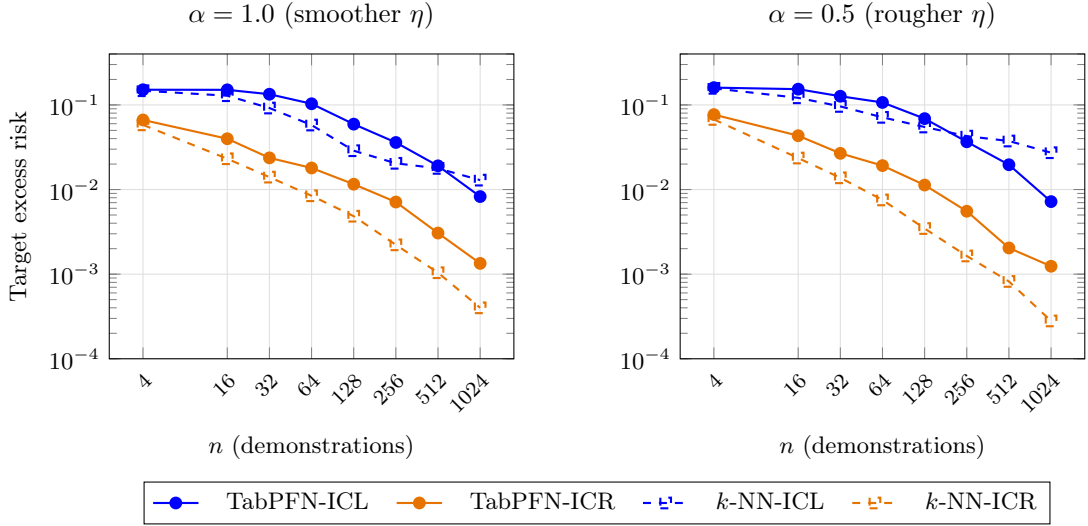

\begin{table}[htbp]
\caption{Weierstrass sample-size sweep, $\alpha = 1.0$: target excess risk, mean$\pm$std over 5 seeds.}
\label{tab:weier-n-alpha1}
\vskip 0.1in
\centering\small\setlength{\tabcolsep}{4pt}
\begin{tabular}{rcccc}
\toprule
$n$ & TabPFN-ICL & TabPFN-ICR & $k$-NN-ICL & $k$-NN-ICR \\
\midrule
4    & \meanstd{\WeierOneTabICL}{0} & \meanstd{\WeierOneTabICR}{0} & \meanstd{\WeierOneKNNICL}{0} & \meanstd{\WeierOneKNNICR}{0} \\
16   & \meanstd{\WeierOneTabICL}{1} & \meanstd{\WeierOneTabICR}{1} & \meanstd{\WeierOneKNNICL}{1} & \meanstd{\WeierOneKNNICR}{1} \\
32   & \meanstd{\WeierOneTabICL}{2} & \meanstd{\WeierOneTabICR}{2} & \meanstd{\WeierOneKNNICL}{2} & \meanstd{\WeierOneKNNICR}{2} \\
64   & \meanstd{\WeierOneTabICL}{3} & \meanstd{\WeierOneTabICR}{3} & \meanstd{\WeierOneKNNICL}{3} & \meanstd{\WeierOneKNNICR}{3} \\
128  & \meanstd{\WeierOneTabICL}{4} & \meanstd{\WeierOneTabICR}{4} & \meanstd{\WeierOneKNNICL}{4} & \meanstd{\WeierOneKNNICR}{4} \\
256  & \meanstd{\WeierOneTabICL}{5} & \meanstd{\WeierOneTabICR}{5} & \meanstd{\WeierOneKNNICL}{5} & \meanstd{\WeierOneKNNICR}{5} \\
512  & \meanstd{\WeierOneTabICL}{6} & \meanstd{\WeierOneTabICR}{6} & \meanstd{\WeierOneKNNICL}{6} & \meanstd{\WeierOneKNNICR}{6} \\
1024 & \meanstd{\WeierOneTabICL}{7} & \meanstd{\WeierOneTabICR}{7} & \meanstd{\WeierOneKNNICL}{7} & \meanstd{\WeierOneKNNICR}{7} \\
\bottomrule
\end{tabular}
\end{table}

\begin{table}[htbp]
\caption{Weierstrass sample-size sweep, $\alpha = 0.5$: target excess risk, mean$\pm$std over 5 seeds.}
\label{tab:weier-n-alpha05}
\vskip 0.1in
\centering\small\setlength{\tabcolsep}{4pt}
\begin{tabular}{rcccc}
\toprule
$n$ & TabPFN-ICL & TabPFN-ICR & $k$-NN-ICL & $k$-NN-ICR \\
\midrule
4    & \meanstd{\WeierHalfTabICL}{0} & \meanstd{\WeierHalfTabICR}{0} & \meanstd{\WeierHalfKNNICL}{0} & \meanstd{\WeierHalfKNNICR}{0} \\
16   & \meanstd{\WeierHalfTabICL}{1} & \meanstd{\WeierHalfTabICR}{1} & \meanstd{\WeierHalfKNNICL}{1} & \meanstd{\WeierHalfKNNICR}{1} \\
32   & \meanstd{\WeierHalfTabICL}{2} & \meanstd{\WeierHalfTabICR}{2} & \meanstd{\WeierHalfKNNICL}{2} & \meanstd{\WeierHalfKNNICR}{2} \\
64   & \meanstd{\WeierHalfTabICL}{3} & \meanstd{\WeierHalfTabICR}{3} & \meanstd{\WeierHalfKNNICL}{3} & \meanstd{\WeierHalfKNNICR}{3} \\
128  & \meanstd{\WeierHalfTabICL}{4} & \meanstd{\WeierHalfTabICR}{4} & \meanstd{\WeierHalfKNNICL}{4} & \meanstd{\WeierHalfKNNICR}{4} \\
256  & \meanstd{\WeierHalfTabICL}{5} & \meanstd{\WeierHalfTabICR}{5} & \meanstd{\WeierHalfKNNICL}{5} & \meanstd{\WeierHalfKNNICR}{5} \\
512  & \meanstd{\WeierHalfTabICL}{6} & \meanstd{\WeierHalfTabICR}{6} & \meanstd{\WeierHalfKNNICL}{6} & \meanstd{\WeierHalfKNNICR}{6} \\
1024 & \meanstd{\WeierHalfTabICL}{7} & \meanstd{\WeierHalfTabICR}{7} & \meanstd{\WeierHalfKNNICL}{7} & \meanstd{\WeierHalfKNNICR}{7} \\
\bottomrule
\end{tabular}
\end{table}

\subsubsection{Experiment 2: posterior drift ($\eta_P \neq \eta_Q$), varying $\varepsilon$}
\label{app:exp-drift-weier}

This experiment fixes the smoothness parameter at $\alpha=1.0$ and sweeps
the posterior drift $\varepsilon\in\{0,0.1,\ldots,1\}$ at two context sizes,
$n=16$ and $n=512$. The drift region $B$ is a union of $N=40$ periodic
subintervals, each covering the leftmost $\varepsilon$-fraction of one
width-$2/N$ interval of $[-1,1]$, so that $Q_X(B)=\varepsilon$ exactly.
We set $\eta_P=1-\eta_Q$ on $B$ and $\eta_P=\eta_Q$ elsewhere. This
construction generally introduces discontinuities in $\eta_P$, so the
source smoothness assumption of Theorem~\ref{thm:icl-icr-rates} does not hold. The experiment
illustrates the effect of posterior drift on ICR and the additive bias
term in Lemma~\ref{lem:fundamental-reduction}, which requires no smoothness assumption. ICL uses
$Q$-labels and is therefore unaffected by $\varepsilon$. \Cref{fig:icr_eps_sweep_weier} shows
the resulting bias--variance crossover, and \Cref{tab:exp2-eps-n16,tab:exp2-eps-n512} report the
per-cell values.

\pgfplotstableread[col sep=space]{
epsilon r1      r2      r3      r4      r5
0.0     0.1440  0.1079  0.1832  0.1469  0.0605
0.1     0.1440  0.1079  0.1832  0.1469  0.0605
0.2     0.1440  0.1079  0.1832  0.1469  0.0605
0.3     0.1440  0.1079  0.1832  0.1469  0.0605
0.4     0.1440  0.1079  0.1832  0.1469  0.0605
0.5     0.1440  0.1079  0.1832  0.1469  0.0605
0.6     0.1440  0.1079  0.1832  0.1469  0.0605
0.7     0.1440  0.1079  0.1832  0.1469  0.0605
0.8     0.1440  0.1079  0.1832  0.1469  0.0605
0.9     0.1440  0.1079  0.1832  0.1469  0.0605
1.0     0.1440  0.1079  0.1832  0.1469  0.0605
}\nSmKNNICL
\pgfplotstablecreatecol[create col/expr={(\thisrow{r1}+\thisrow{r2}+\thisrow{r3}+\thisrow{r4}+\thisrow{r5})/5}]{mean}{\nSmKNNICL}

\pgfplotstableread[col sep=space]{
epsilon r1      r2      r3      r4      r5
0.0     0.0205  0.0225  0.0254  0.0231  0.0240
0.1     0.0472  0.0483  0.0511  0.0510  0.0492
0.2     0.0716  0.0693  0.0741  0.0750  0.0752
0.3     0.0981  0.0953  0.0997  0.0996  0.0995
0.4     0.1239  0.1234  0.1230  0.1248  0.1262
0.5     0.1505  0.1525  0.1475  0.1527  0.1517
0.6     0.1764  0.1784  0.1705  0.1788  0.1791
0.7     0.2042  0.2038  0.1975  0.2047  0.1991
0.8     0.2309  0.2271  0.2246  0.2314  0.2228
0.9     0.2546  0.2547  0.2520  0.2558  0.2510
1.0     0.2828  0.2797  0.2806  0.2805  0.2788
}\nSmKNNICR
\pgfplotstablecreatecol[create col/expr={(\thisrow{r1}+\thisrow{r2}+\thisrow{r3}+\thisrow{r4}+\thisrow{r5})/5}]{mean}{\nSmKNNICR}

\pgfplotstableread[col sep=space]{
epsilon r1      r2      r3      r4      r5
0.0     0.1403  0.1416  0.1692  0.1439  0.1601
0.1     0.1403  0.1416  0.1692  0.1439  0.1601
0.2     0.1403  0.1416  0.1692  0.1439  0.1601
0.3     0.1403  0.1416  0.1692  0.1439  0.1601
0.4     0.1403  0.1416  0.1692  0.1439  0.1601
0.5     0.1403  0.1416  0.1692  0.1439  0.1601
0.6     0.1403  0.1416  0.1692  0.1439  0.1601
0.7     0.1403  0.1416  0.1692  0.1439  0.1601
0.8     0.1403  0.1416  0.1692  0.1439  0.1601
0.9     0.1403  0.1416  0.1692  0.1439  0.1601
1.0     0.1403  0.1416  0.1692  0.1439  0.1601
}\nSmTabICL
\pgfplotstablecreatecol[create col/expr={(\thisrow{r1}+\thisrow{r2}+\thisrow{r3}+\thisrow{r4}+\thisrow{r5})/5}]{mean}{\nSmTabICL}

\pgfplotstableread[col sep=space]{
epsilon r1      r2      r3      r4      r5
0.0     0.0391  0.0416  0.0416  0.0315  0.0450
0.1     0.0622  0.0660  0.0635  0.0591  0.0637
0.2     0.0804  0.0808  0.0875  0.0815  0.0845
0.3     0.1053  0.1040  0.1069  0.1075  0.1023
0.4     0.1325  0.1279  0.1258  0.1331  0.1253
0.5     0.1490  0.1549  0.1431  0.1578  0.1421
0.6     0.1726  0.1788  0.1649  0.1798  0.1656
0.7     0.1957  0.1991  0.1892  0.2029  0.1806
0.8     0.2210  0.2184  0.2111  0.2296  0.2053
0.9     0.2430  0.2416  0.2347  0.2499  0.2327
1.0     0.2654  0.2635  0.2615  0.2733  0.2578
}\nSmTabICR
\pgfplotstablecreatecol[create col/expr={(\thisrow{r1}+\thisrow{r2}+\thisrow{r3}+\thisrow{r4}+\thisrow{r5})/5}]{mean}{\nSmTabICR}

\pgfplotstableread[col sep=space]{
epsilon r1      r2      r3      r4      r5
0.0     0.0137  0.0178  0.0304  0.0074  0.0196
0.1     0.0137  0.0178  0.0304  0.0074  0.0196
0.2     0.0137  0.0178  0.0304  0.0074  0.0196
0.3     0.0137  0.0178  0.0304  0.0074  0.0196
0.4     0.0137  0.0178  0.0304  0.0074  0.0196
0.5     0.0137  0.0178  0.0304  0.0074  0.0196
0.6     0.0137  0.0178  0.0304  0.0074  0.0196
0.7     0.0137  0.0178  0.0304  0.0074  0.0196
0.8     0.0137  0.0178  0.0304  0.0074  0.0196
0.9     0.0137  0.0178  0.0304  0.0074  0.0196
1.0     0.0137  0.0178  0.0304  0.0074  0.0196
}\nLgKNNICL
\pgfplotstablecreatecol[create col/expr={(\thisrow{r1}+\thisrow{r2}+\thisrow{r3}+\thisrow{r4}+\thisrow{r5})/5}]{mean}{\nLgKNNICL}

\pgfplotstableread[col sep=space]{
epsilon r1      r2      r3      r4      r5
0.0     0.0009  0.0008  0.0012  0.0011  0.0012
0.1     0.0268  0.0154  0.0239  0.0194  0.0142
0.2     0.0604  0.0577  0.0592  0.0607  0.0596
0.3     0.0900  0.0884  0.0911  0.0895  0.0886
0.4     0.1182  0.1208  0.1221  0.1208  0.1194
0.5     0.1493  0.1482  0.1504  0.1514  0.1491
0.6     0.1803  0.1830  0.1810  0.1863  0.1788
0.7     0.2140  0.2112  0.2163  0.2129  0.2109
0.8     0.2437  0.2439  0.2464  0.2467  0.2407
0.9     0.2910  0.2876  0.2906  0.2915  0.2878
1.0     0.3039  0.3039  0.3036  0.3037  0.3036
}\nLgKNNICR
\pgfplotstablecreatecol[create col/expr={(\thisrow{r1}+\thisrow{r2}+\thisrow{r3}+\thisrow{r4}+\thisrow{r5})/5}]{mean}{\nLgKNNICR}

\pgfplotstableread[col sep=space]{
epsilon r1      r2      r3      r4      r5
0.0     0.0150  0.0164  0.0421  0.0045  0.0173
0.1     0.0150  0.0164  0.0421  0.0045  0.0173
0.2     0.0150  0.0164  0.0421  0.0045  0.0173
0.3     0.0150  0.0164  0.0421  0.0045  0.0173
0.4     0.0150  0.0164  0.0421  0.0045  0.0173
0.5     0.0150  0.0164  0.0421  0.0045  0.0173
0.6     0.0150  0.0164  0.0421  0.0045  0.0173
0.7     0.0150  0.0164  0.0421  0.0045  0.0173
0.8     0.0150  0.0164  0.0421  0.0045  0.0173
0.9     0.0150  0.0164  0.0421  0.0045  0.0173
1.0     0.0150  0.0164  0.0421  0.0045  0.0173
}\nLgTabICL
\pgfplotstablecreatecol[create col/expr={(\thisrow{r1}+\thisrow{r2}+\thisrow{r3}+\thisrow{r4}+\thisrow{r5})/5}]{mean}{\nLgTabICL}

\pgfplotstableread[col sep=space]{
epsilon r1      r2      r3      r4      r5
0.0     0.0040  0.0039  0.0026  0.0022  0.0036
0.1     0.0324  0.0322  0.0324  0.0314  0.0304
0.2     0.0616  0.0604  0.0598  0.0605  0.0599
0.3     0.0908  0.0919  0.0903  0.0903  0.0890
0.4     0.1203  0.1215  0.1213  0.1188  0.1201
0.5     0.1496  0.1529  0.1516  0.1497  0.1510
0.6     0.1799  0.1824  0.1812  0.1824  0.1816
0.7     0.2104  0.2137  0.2124  0.2128  0.2121
0.8     0.2412  0.2420  0.2432  0.2431  0.2432
0.9     0.2700  0.2734  0.2722  0.2721  0.2717
1.0     0.3008  0.3016  0.3024  0.3022  0.3012
}\nLgTabICR
\pgfplotstablecreatecol[create col/expr={(\thisrow{r1}+\thisrow{r2}+\thisrow{r3}+\thisrow{r4}+\thisrow{r5})/5}]{mean}{\nLgTabICR}

\addstd{\nSmKNNICL}\addstd{\nSmKNNICR}\addstd{\nSmTabICL}\addstd{\nSmTabICR}
\addstd{\nLgKNNICL}\addstd{\nLgKNNICR}\addstd{\nLgTabICL}\addstd{\nLgTabICR}

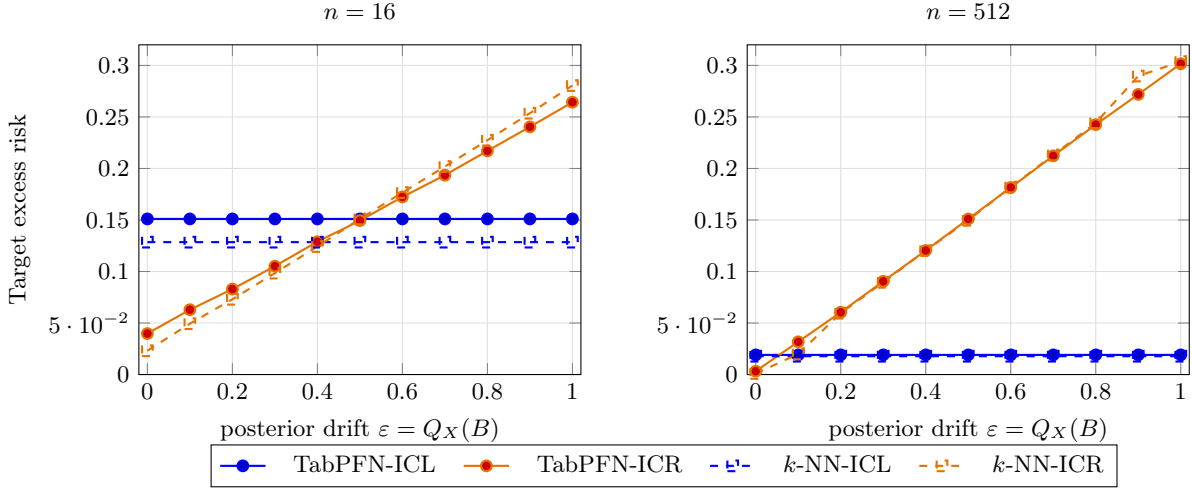
\begin{figure}[htbp]
\centering
\begin{tikzpicture}
\begin{groupplot}[
    group style={
        group size=2 by 1,
        horizontal sep=2.2cm,
        group name=expTwoEpsPdfPlots,
    },
    width=0.45\textwidth,
    height=0.36\textwidth,
    xmin=-0.02, xmax=1.02,
    ymin=0,     ymax=0.32,
    xtick={0,0.2,0.4,0.6,0.8,1.0},
    xticklabels={$0$,$0.2$,$0.4$,$0.6$,$0.8$,$1$},
    ytick={0,0.05,0.10,0.15,0.20,0.25,0.30},
    xlabel={posterior drift $\varepsilon = Q_X(B)$},
    ylabel={Target excess risk},
    grid=major,
    grid style={gray!25},
    tick label style={font=\small},
    label style={font=\small},
    title style={font=\small},
]

\nextgroupplot[title={$n = 16$},
    legend to name=expTwoEpsPdfLegend,
    legend columns=4,
    legend style={font=\small, column sep=0.8em, legend cell align=left},
]
\addplot+[blue, solid, thick, mark=*, mark size=2pt]
    table[x=epsilon, y=mean] {\nSmTabICL};
\addlegendentry{TabPFN-ICL}
\addplot+[orange!90!black, solid, thick, mark=*, mark size=2pt]
    table[x=epsilon, y=mean] {\nSmTabICR};
\addlegendentry{TabPFN-ICR}
\addplot+[blue, dashed, thick, mark=square, mark size=2pt]
    table[x=epsilon, y=mean] {\nSmKNNICL};
\addlegendentry{$k$-NN-ICL}
\addplot+[orange!90!black, dashed, thick, mark=square, mark size=2pt]
    table[x=epsilon, y=mean] {\nSmKNNICR};
\addlegendentry{$k$-NN-ICR}

\nextgroupplot[title={$n = 512$}, ylabel={}]
\addplot+[blue, solid, thick, mark=*, mark size=2pt]
    table[x=epsilon, y=mean] {\nLgTabICL};
\addplot+[orange!90!black, solid, thick, mark=*, mark size=2pt]
    table[x=epsilon, y=mean] {\nLgTabICR};
\addplot+[blue, dashed, thick, mark=square, mark size=2pt]
    table[x=epsilon, y=mean] {\nLgKNNICL};
\addplot+[orange!90!black, dashed, thick, mark=square, mark size=2pt]
    table[x=epsilon, y=mean] {\nLgKNNICR};

\end{groupplot}

\node[anchor=north, yshift=-0.8cm]
    at ($(expTwoEpsPdfPlots c1r1.south)!0.5!(expTwoEpsPdfPlots c2r1.south)$)
    {\pgfplotslegendfromname{expTwoEpsPdfLegend}};

\end{tikzpicture}
\caption{Target excess risk vs.\ posterior drift $\varepsilon = Q_X(B)$ on the Weierstrass family at $\alpha = 1.0$ (smooth $\eta$), for $n \in \{16, 512\}$ over 5 seeds. $B$ is a union of $N = 40$ periodic sub-brackets, each covering the leftmost $\varepsilon$-fraction of one width-$2/N$ interval of $[-1, 1]$, so $B$ samples every macro-cycle of $\eta_Q$ uniformly and $Q_X(B) = \varepsilon$ holds exactly. \emph{Color}: ICL (blue) vs.\ ICR (orange). \emph{Line style}: TabPFN (solid, \textbullet) vs.\ $k$-NN (dashed, $\square$). ICL curves are exactly flat because ICL uses $Q$ labels and does not see the drift; ICR curves rise almost linearly in $\varepsilon$ across the whole range, saturating near $1 - 2R_Q^\ast \approx 0.306$ at $\varepsilon = 1$. At small $n$ (left, $n=16$) ICR beats ICL up to $\varepsilon \approx 0.4$--$0.5$; at large $n$ (right, $n=512$) ICR is roughly an order of magnitude below ICL at $\varepsilon = 0$ (variance-limited regime), but ICL overtakes it already by $\varepsilon \approx 0.08$. In both regimes the ICR bias tracks $\varepsilon$ almost linearly, matching the $\varepsilon + (\lambda / \kappa) (R_P - R_P^\ast)$ reduction in \Cref{lem:fundamental-reduction}.}
\label{fig:icr_eps_sweep_weier}
\end{figure}

\begin{table}[htbp]
\caption{Experiment 2, $\alpha = 1.0$, $n = 16$: target excess risk vs.\ posterior drift $\varepsilon = Q_X(B)$, mean$\pm$std over $5$ seeds (Weierstrass family, $A = 0.24$, $b = 0.85$, $f = 2$, $J = 12$; oracle $k_\alpha = \mathrm{round\_odd}(0.30\cdot n^{2\alpha/(2\alpha+1)})$; $m_{\mathrm{corpus}} = 10^{5}$; $n_{\mathrm{test}} = 400$).  Region $B$ is the leftmost $\varepsilon$-fraction of each of $N = 40$ equal-width periodic brackets partitioning $[-1, 1]$, so $Q_X(B) = \varepsilon$ exactly and $B$ samples every macro-cycle of $\eta_Q$ uniformly; $\eta_P = 1 - \eta_Q$ on $B$; ICL uses $Q$ labels (so ICL columns are $\varepsilon$-constant).}
\label{tab:exp2-eps-n16}
\vskip 0.1in
\centering\small\setlength{\tabcolsep}{4pt}
\begin{tabular}{ccccc}
\toprule
$\varepsilon$ & TabPFN-ICL & TabPFN-ICR & $k$-NN-ICL & $k$-NN-ICR \\
\midrule
$0.0$ & \meanstd{\nSmTabICL}{0}  & \meanstd{\nSmTabICR}{0}  & \meanstd{\nSmKNNICL}{0}  & \meanstd{\nSmKNNICR}{0}  \\
$0.1$ & \meanstd{\nSmTabICL}{1}  & \meanstd{\nSmTabICR}{1}  & \meanstd{\nSmKNNICL}{1}  & \meanstd{\nSmKNNICR}{1}  \\
$0.2$ & \meanstd{\nSmTabICL}{2}  & \meanstd{\nSmTabICR}{2}  & \meanstd{\nSmKNNICL}{2}  & \meanstd{\nSmKNNICR}{2}  \\
$0.3$ & \meanstd{\nSmTabICL}{3}  & \meanstd{\nSmTabICR}{3}  & \meanstd{\nSmKNNICL}{3}  & \meanstd{\nSmKNNICR}{3}  \\
$0.4$ & \meanstd{\nSmTabICL}{4}  & \meanstd{\nSmTabICR}{4}  & \meanstd{\nSmKNNICL}{4}  & \meanstd{\nSmKNNICR}{4}  \\
$0.5$ & \meanstd{\nSmTabICL}{5}  & \meanstd{\nSmTabICR}{5}  & \meanstd{\nSmKNNICL}{5}  & \meanstd{\nSmKNNICR}{5}  \\
$0.6$ & \meanstd{\nSmTabICL}{6}  & \meanstd{\nSmTabICR}{6}  & \meanstd{\nSmKNNICL}{6}  & \meanstd{\nSmKNNICR}{6}  \\
$0.7$ & \meanstd{\nSmTabICL}{7}  & \meanstd{\nSmTabICR}{7}  & \meanstd{\nSmKNNICL}{7}  & \meanstd{\nSmKNNICR}{7}  \\
$0.8$ & \meanstd{\nSmTabICL}{8}  & \meanstd{\nSmTabICR}{8}  & \meanstd{\nSmKNNICL}{8}  & \meanstd{\nSmKNNICR}{8}  \\
$0.9$ & \meanstd{\nSmTabICL}{9}  & \meanstd{\nSmTabICR}{9}  & \meanstd{\nSmKNNICL}{9}  & \meanstd{\nSmKNNICR}{9}  \\
$1.0$ & \meanstd{\nSmTabICL}{10} & \meanstd{\nSmTabICR}{10} & \meanstd{\nSmKNNICL}{10} & \meanstd{\nSmKNNICR}{10} \\
\bottomrule
\end{tabular}
\end{table}

\begin{table}[htbp]
\caption{Experiment 2, $\alpha = 1.0$, $n = 512$: target excess risk vs.\ posterior drift $\varepsilon = Q_X(B)$, mean$\pm$std over $5$ seeds (same setup as \Cref{tab:exp2-eps-n16}).}
\label{tab:exp2-eps-n512}
\vskip 0.1in
\centering\small\setlength{\tabcolsep}{4pt}
\begin{tabular}{ccccc}
\toprule
$\varepsilon$ & TabPFN-ICL & TabPFN-ICR & $k$-NN-ICL & $k$-NN-ICR \\
\midrule
$0.0$ & \meanstd{\nLgTabICL}{0}  & \meanstd{\nLgTabICR}{0}  & \meanstd{\nLgKNNICL}{0}  & \meanstd{\nLgKNNICR}{0}  \\
$0.1$ & \meanstd{\nLgTabICL}{1}  & \meanstd{\nLgTabICR}{1}  & \meanstd{\nLgKNNICL}{1}  & \meanstd{\nLgKNNICR}{1}  \\
$0.2$ & \meanstd{\nLgTabICL}{2}  & \meanstd{\nLgTabICR}{2}  & \meanstd{\nLgKNNICL}{2}  & \meanstd{\nLgKNNICR}{2}  \\
$0.3$ & \meanstd{\nLgTabICL}{3}  & \meanstd{\nLgTabICR}{3}  & \meanstd{\nLgKNNICL}{3}  & \meanstd{\nLgKNNICR}{3}  \\
$0.4$ & \meanstd{\nLgTabICL}{4}  & \meanstd{\nLgTabICR}{4}  & \meanstd{\nLgKNNICL}{4}  & \meanstd{\nLgKNNICR}{4}  \\
$0.5$ & \meanstd{\nLgTabICL}{5}  & \meanstd{\nLgTabICR}{5}  & \meanstd{\nLgKNNICL}{5}  & \meanstd{\nLgKNNICR}{5}  \\
$0.6$ & \meanstd{\nLgTabICL}{6}  & \meanstd{\nLgTabICR}{6}  & \meanstd{\nLgKNNICL}{6}  & \meanstd{\nLgKNNICR}{6}  \\
$0.7$ & \meanstd{\nLgTabICL}{7}  & \meanstd{\nLgTabICR}{7}  & \meanstd{\nLgKNNICL}{7}  & \meanstd{\nLgKNNICR}{7}  \\
$0.8$ & \meanstd{\nLgTabICL}{8}  & \meanstd{\nLgTabICR}{8}  & \meanstd{\nLgKNNICL}{8}  & \meanstd{\nLgKNNICR}{8}  \\
$0.9$ & \meanstd{\nLgTabICL}{9}  & \meanstd{\nLgTabICR}{9}  & \meanstd{\nLgKNNICL}{9}  & \meanstd{\nLgKNNICR}{9}  \\
$1.0$ & \meanstd{\nLgTabICL}{10} & \meanstd{\nLgTabICR}{10} & \meanstd{\nLgKNNICL}{10} & \meanstd{\nLgKNNICR}{10} \\
\bottomrule
\end{tabular}
\end{table}

\subsubsection{Experiment 3: varying smoothness $\alpha$ (via the Weierstrass function)}
\label{app:exp-smoothness}

\paragraph{Motivation.}
\Cref{thm:icl-bound,thm:icr-bound} predicts that ICL pays a smoothness penalty at rate $n^{-\alpha(\beta+1)/(2\alpha+1)}$, while ICR attains the $\alpha$-independent rate $n^{-(\beta+1)/2}$ in the large-corpus limit.
To isolate this effect, we vary $\alpha$ while holding the margin exponent $\beta$, the Bayes risk $R^*_Q$, and the Bayes classifier fixed by construction.

\pgfplotstableread[col sep=space]{
alpha   r1      r2      r3      r4      r5
1.000   0.0219  0.0144  0.0244  0.0054  0.0360
0.100   0.0760  0.0592  0.0903  0.0897  0.0841
0.010   0.0907  0.0793  0.1015  0.0811  0.0918
0.001   0.1144  0.1113  0.1035  0.0945  0.0972
}\nNbKNNICL
\pgfplotstablecreatecol[create col/expr={(\thisrow{r1}+\thisrow{r2}+\thisrow{r3}+\thisrow{r4}+\thisrow{r5})/5}]{mean}{\nNbKNNICL}

\pgfplotstableread[col sep=space]{
alpha   r1      r2      r3      r4      r5
1.000   0.0030  0.0023  0.0022  0.0014  0.0022
0.100   0.0012  0.0008  0.0012  0.0008  0.0007
0.010   0.0007  0.0016  0.0016  0.0016  0.0017
0.001   0.0012  0.0015  0.0012  0.0006  0.0020
}\nNbKNNICR
\pgfplotstablecreatecol[create col/expr={(\thisrow{r1}+\thisrow{r2}+\thisrow{r3}+\thisrow{r4}+\thisrow{r5})/5}]{mean}{\nNbKNNICR}

\pgfplotstableread[col sep=space]{
alpha   r1      r2      r3      r4      r5
1.000   0.0423  0.0165  0.0443  0.0145  0.0623
0.100   0.0465  0.0183  0.0655  0.0275  0.0307
0.010   0.0763  0.0066  0.0709  0.0544  0.0194
0.001   0.0995  0.1030  0.0821  0.0407  0.0765
}\nNbTabICL
\pgfplotstablecreatecol[create col/expr={(\thisrow{r1}+\thisrow{r2}+\thisrow{r3}+\thisrow{r4}+\thisrow{r5})/5}]{mean}{\nNbTabICL}

\pgfplotstableread[col sep=space]{
alpha   r1      r2      r3      r4      r5
1.000   0.0059  0.0079  0.0066  0.0072  0.0080
0.100   0.0047  0.0022  0.0047  0.0038  0.0051
0.010   0.0073  0.0018  0.0053  0.0063  0.0064
0.001   0.0018  0.0055  0.0049  0.0022  0.0068
}\nNbTabICR
\pgfplotstablecreatecol[create col/expr={(\thisrow{r1}+\thisrow{r2}+\thisrow{r3}+\thisrow{r4}+\thisrow{r5})/5}]{mean}{\nNbTabICR}

\addstd{\nNbKNNICL}\addstd{\nNbKNNICR}\addstd{\nNbTabICL}\addstd{\nNbTabICR}

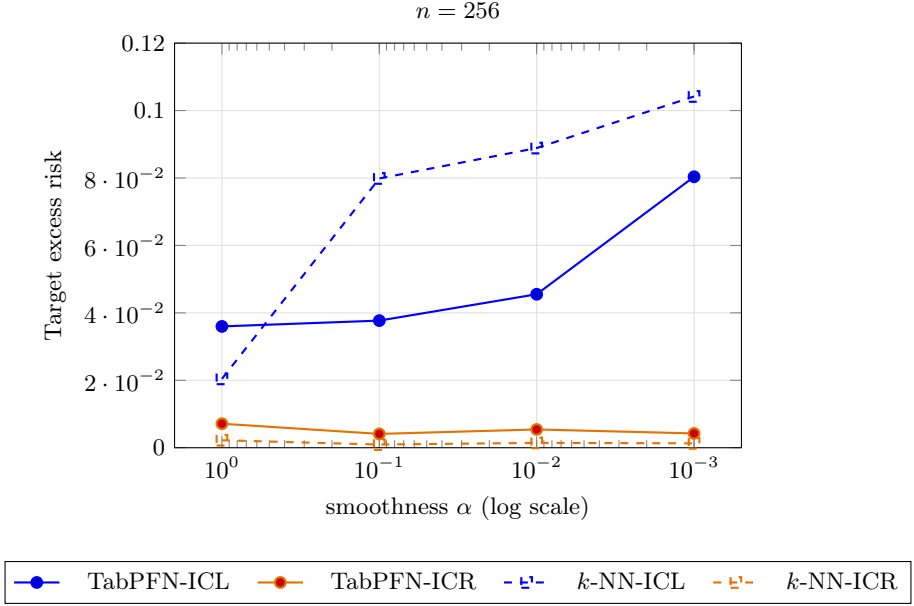
\begin{figure}[htbp]
\centering
\begin{tikzpicture}
\begin{axis}[
    width=0.55\textwidth,
    height=0.42\textwidth,
    xmode=log,
    xmin=0.0005, xmax=2,
    ymin=0, ymax=0.12,
    xtick={0.001,0.01,0.1,1},
    xticklabels={$10^{-3}$,$10^{-2}$,$10^{-1}$,$10^{0}$},
    x dir=reverse,                       %
    xlabel={smoothness $\alpha$ (log scale)},
    ylabel={Target excess risk},
    title={$n = 256$},
    grid=major,
    grid style={gray!25},
    tick label style={font=\small},
    label style={font=\small},
    title style={font=\small},
    legend columns=4,
    legend style={
        at={(0.5,-0.28)}, anchor=north,
        font=\small, column sep=0.8em, legend cell align=left,
    },
]
\addplot+[blue, solid, thick, mark=*, mark size=2pt]
    table[x=alpha, y=mean] {\nNbTabICL};
\addlegendentry{TabPFN-ICL}
\addplot+[orange!90!black, solid, thick, mark=*, mark size=2pt]
    table[x=alpha, y=mean] {\nNbTabICR};
\addlegendentry{TabPFN-ICR}
\addplot+[blue, dashed, thick, mark=square, mark size=2pt]
    table[x=alpha, y=mean] {\nNbKNNICL};
\addlegendentry{$k$-NN-ICL}
\addplot+[orange!90!black, dashed, thick, mark=square, mark size=2pt]
    table[x=alpha, y=mean] {\nNbKNNICR};
\addlegendentry{$k$-NN-ICR}
\end{axis}
\end{tikzpicture}
\caption{Target excess risk vs.\ smoothness exponent $\alpha$ on the Weierstrass family at $n = 256$, with $\alpha \in \{1,\,10^{-1},\,10^{-2},\,10^{-3}\}$ (log-spaced, 5 seeds). \emph{Color}: ICL (blue) vs.\ ICR (orange). \emph{Line style}: TabPFN (solid, \textbullet) vs.\ $k$-NN (dashed, $\square$). The horizontal axis is logarithmic; the $\alpha$-axis is reversed so rougher $\eta_\alpha$ is on the right. Both $k$-NN-ICL and TabPFN-ICL trend upward as $\alpha$ decreases, matching the $n^{-2\alpha/(2\alpha+1)}$ prediction of \Cref{thm:icl-bound}; the $k$-NN-ICL curve saturates near $\alpha \to 0$ at the finite-$n$ noise floor.
ICR sits about an order of magnitude below ICL and is essentially flat in $\alpha$, consistent with the $\alpha$-independent $n^{-(\beta+1)/2}$ rate of \Cref{thm:icr-bound}; the ICL/ICR gap widens as $\alpha \to 0$.}
\label{fig:icl_icr_alpha_eta_pdf}
\end{figure}

\begin{table}[htbp]
\caption{Experiment 3, $n = 256$: target excess risk vs.\ smoothness $\alpha$, mean$\pm$std over $5$ seeds (Weierstrass family, $A = 0.24$, $b = 0.85$, $f = 2$, $J = 12$; oracle $k_\alpha = \mathrm{round\_odd}(0.30\cdot n^{2\alpha/(2\alpha+1)})$; $m_{\mathrm{corpus}} = 10^{5}$; $n_{\mathrm{test}} = 400$).}
\label{tab:exp3-eta-n256}
\vskip 0.1in
\centering\small\setlength{\tabcolsep}{4pt}
\begin{tabular}{ccccc}
\toprule
$\alpha$ & TabPFN-ICL & TabPFN-ICR & $k$-NN-ICL & $k$-NN-ICR \\
\midrule
$1$        & \meanstd{\nNbTabICL}{0} & \meanstd{\nNbTabICR}{0} & \meanstd{\nNbKNNICL}{0} & \meanstd{\nNbKNNICR}{0} \\
$10^{-1}$  & \meanstd{\nNbTabICL}{1} & \meanstd{\nNbTabICR}{1} & \meanstd{\nNbKNNICL}{1} & \meanstd{\nNbKNNICR}{1} \\
$10^{-2}$  & \meanstd{\nNbTabICL}{2} & \meanstd{\nNbTabICR}{2} & \meanstd{\nNbKNNICL}{2} & \meanstd{\nNbKNNICR}{2} \\
$10^{-3}$  & \meanstd{\nNbTabICL}{3} & \meanstd{\nNbTabICR}{3} & \meanstd{\nNbKNNICL}{3} & \meanstd{\nNbKNNICR}{3} \\
\bottomrule
\end{tabular}
\end{table}

\subsubsection{Experiment 4: varying corpus size $m$}

\paragraph{Prediction from the finite-corpus bound.}
The ICR bound in~\eqref{eq:icr-finite-m-main} is governed by
\[
    \max\!\left\{
        2L_P(2n/m)^{\alpha_P},
        \sqrt{8(\beta+2)/n}
    \right\},
\]
whose two terms represent localization bias and label variance, respectively.
The bound enters its variance-limited regime when
\[
    m\geq m_\star(n)
    =C_\star n^{1+1/(2\alpha_P)}
\]
as given exactly in~\eqref{eq:icr-saturation}. Thus increasing $m$ improves
the bias-dependent part of the guarantee below this scale, whereas the bound
no longer improves with $m$ above it. This is a sufficient, order-level
saturation condition rather than an exact crossover between ICR and ICL:
the multiplicative constant depends on $(L_P,\beta)$, and the upper bound does
not imply that ICL must be better below the threshold.

\paragraph{Setup and results.}
We use the smooth Weierstrass family from Experiment~3 with no shift
($P=Q$), fix $\alpha_P=1$, and vary the corpus size over powers of two for
$n\in\{64,128\}$. Figure~\ref{fig:m_sweep} marks the unit-constant reference
$m=n^{3/2}$: $512$ for $n=64$ and approximately $1448$ for $n=128$.
For both $k$-NN and TabPFN, ICR improves sharply as the corpus becomes larger
than the context and then reaches a low-error plateau, while ICL is unaffected
by corpus size. The transition is sharper for $k$-NN and more gradual for
TabPFN; the small nonmonotone fluctuations across adjacent corpus sizes are
consistent with finite-trial variation. Table~\ref{tab:m_sweep_full} reports
all configurations and seeds rather than selecting individual runs.

\pgfplotstableread[col sep=space]{
r1      r2      r3      r4      r5
0.1010  0.0586  0.0828  0.0675  0.0852
}\mSaKNNICL
\pgfplotstablecreatecol[create col/expr={(\thisrow{r1}+\thisrow{r2}+\thisrow{r3}+\thisrow{r4}+\thisrow{r5})/5}]{mean}{\mSaKNNICL}
\pgfplotstablecreatecol[create col/expr={
    sqrt( ((\thisrow{r1}-\thisrow{mean})^2 + (\thisrow{r2}-\thisrow{mean})^2
         + (\thisrow{r3}-\thisrow{mean})^2 + (\thisrow{r4}-\thisrow{mean})^2
         + (\thisrow{r5}-\thisrow{mean})^2) / 4 )
}]{std}{\mSaKNNICL}

\pgfplotstableread[col sep=space]{
r1      r2      r3      r4      r5
0.1139  0.0698  0.1456  0.0694  0.0841
}\mSaTabICL
\pgfplotstablecreatecol[create col/expr={(\thisrow{r1}+\thisrow{r2}+\thisrow{r3}+\thisrow{r4}+\thisrow{r5})/5}]{mean}{\mSaTabICL}
\pgfplotstablecreatecol[create col/expr={
    sqrt( ((\thisrow{r1}-\thisrow{mean})^2 + (\thisrow{r2}-\thisrow{mean})^2
         + (\thisrow{r3}-\thisrow{mean})^2 + (\thisrow{r4}-\thisrow{mean})^2
         + (\thisrow{r5}-\thisrow{mean})^2) / 4 )
}]{std}{\mSaTabICL}

\pgfplotstableread[col sep=space]{
m       r1      r2      r3      r4      r5
64      0.1907  0.1142  0.1142  0.1142  0.1907
128     0.1718  0.1114  0.1142  0.1641  0.1364
256     0.0896  0.0943  0.1098  0.1644  0.0865
512     0.0075  0.0099  0.0273  0.0256  0.0224
1024    0.0116  0.0038  0.0086  0.0112  0.0077
2048    0.0049  0.0051  0.0066  0.0079  0.0151
4096    0.0109  0.0040  0.0051  0.0110  0.0094
8192    0.0068  0.0076  0.0085  0.0050  0.0078
16384   0.0094  0.0093  0.0089  0.0090  0.0043
}\mSaKNNICR
\pgfplotstablecreatecol[create col/expr={(\thisrow{r1}+\thisrow{r2}+\thisrow{r3}+\thisrow{r4}+\thisrow{r5})/5}]{mean}{\mSaKNNICR}
\pgfplotstablecreatecol[create col/expr={
    sqrt( ((\thisrow{r1}-\thisrow{mean})^2 + (\thisrow{r2}-\thisrow{mean})^2
         + (\thisrow{r3}-\thisrow{mean})^2 + (\thisrow{r4}-\thisrow{mean})^2
         + (\thisrow{r5}-\thisrow{mean})^2) / 4 )
}]{std}{\mSaKNNICR}

\pgfplotstableread[col sep=space]{
m       r1      r2      r3      r4      r5
64      0.1552  0.1198  0.1099  0.1343  0.1473
128     0.0706  0.0813  0.0674  0.1034  0.0433
256     0.0450  0.0293  0.0526  0.0348  0.0279
512     0.0184  0.0188  0.0291  0.0193  0.0160
1024    0.0257  0.0115  0.0245  0.0224  0.0190
2048    0.0271  0.0178  0.0179  0.0148  0.0270
4096    0.0241  0.0117  0.0148  0.0157  0.0204
8192    0.0193  0.0142  0.0178  0.0154  0.0157
16384   0.0187  0.0144  0.0167  0.0232  0.0159
}\mSaTabICR
\pgfplotstablecreatecol[create col/expr={(\thisrow{r1}+\thisrow{r2}+\thisrow{r3}+\thisrow{r4}+\thisrow{r5})/5}]{mean}{\mSaTabICR}
\pgfplotstablecreatecol[create col/expr={
    sqrt( ((\thisrow{r1}-\thisrow{mean})^2 + (\thisrow{r2}-\thisrow{mean})^2
         + (\thisrow{r3}-\thisrow{mean})^2 + (\thisrow{r4}-\thisrow{mean})^2
         + (\thisrow{r5}-\thisrow{mean})^2) / 4 )
}]{std}{\mSaTabICR}

\pgfplotstableread[col sep=space]{
r1      r2      r3      r4      r5
0.0422  0.0995  0.0542  0.0678  0.0593
}\mSbKNNICL
\pgfplotstablecreatecol[create col/expr={(\thisrow{r1}+\thisrow{r2}+\thisrow{r3}+\thisrow{r4}+\thisrow{r5})/5}]{mean}{\mSbKNNICL}
\pgfplotstablecreatecol[create col/expr={
    sqrt( ((\thisrow{r1}-\thisrow{mean})^2 + (\thisrow{r2}-\thisrow{mean})^2
         + (\thisrow{r3}-\thisrow{mean})^2 + (\thisrow{r4}-\thisrow{mean})^2
         + (\thisrow{r5}-\thisrow{mean})^2) / 4 )
}]{std}{\mSbKNNICL}

\pgfplotstableread[col sep=space]{
r1      r2      r3      r4      r5
0.0624  0.1317  0.0904  0.0891  0.0832
}\mSbTabICL
\pgfplotstablecreatecol[create col/expr={(\thisrow{r1}+\thisrow{r2}+\thisrow{r3}+\thisrow{r4}+\thisrow{r5})/5}]{mean}{\mSbTabICL}
\pgfplotstablecreatecol[create col/expr={
    sqrt( ((\thisrow{r1}-\thisrow{mean})^2 + (\thisrow{r2}-\thisrow{mean})^2
         + (\thisrow{r3}-\thisrow{mean})^2 + (\thisrow{r4}-\thisrow{mean})^2
         + (\thisrow{r5}-\thisrow{mean})^2) / 4 )
}]{std}{\mSbTabICL}

\pgfplotstableread[col sep=space]{
m       r1      r2      r3      r4      r5
128     0.1142  0.1907  0.1142  0.1142  0.1142
256     0.2118  0.1441  0.1142  0.1595  0.0893
512     0.1061  0.1326  0.0894  0.1033  0.1057
1024    0.0094  0.0127  0.0050  0.0038  0.0099
2048    0.0015  0.0109  0.0083  0.0006  0.0015
4096    0.0041  0.0053  0.0029  0.0033  0.0040
8192    0.0038  0.0037  0.0051  0.0030  0.0060
16384   0.0036  0.0063  0.0060  0.0050  0.0043
32768   0.0038  0.0043  0.0043  0.0028  0.0049
}\mSbKNNICR
\pgfplotstablecreatecol[create col/expr={(\thisrow{r1}+\thisrow{r2}+\thisrow{r3}+\thisrow{r4}+\thisrow{r5})/5}]{mean}{\mSbKNNICR}
\pgfplotstablecreatecol[create col/expr={
    sqrt( ((\thisrow{r1}-\thisrow{mean})^2 + (\thisrow{r2}-\thisrow{mean})^2
         + (\thisrow{r3}-\thisrow{mean})^2 + (\thisrow{r4}-\thisrow{mean})^2
         + (\thisrow{r5}-\thisrow{mean})^2) / 4 )
}]{std}{\mSbKNNICR}

\pgfplotstableread[col sep=space]{
m       r1      r2      r3      r4      r5
128     0.0304  0.0786  0.0951  0.0820  0.0911
256     0.0354  0.0559  0.0513  0.0500  0.0439
512     0.0207  0.0442  0.0218  0.0129  0.0280
1024    0.0131  0.0153  0.0146  0.0066  0.0175
2048    0.0103  0.0165  0.0084  0.0068  0.0181
4096    0.0107  0.0129  0.0081  0.0083  0.0156
8192    0.0124  0.0093  0.0114  0.0104  0.0117
16384   0.0121  0.0096  0.0119  0.0122  0.0078
32768   0.0110  0.0113  0.0117  0.0118  0.0100
}\mSbTabICR
\pgfplotstablecreatecol[create col/expr={(\thisrow{r1}+\thisrow{r2}+\thisrow{r3}+\thisrow{r4}+\thisrow{r5})/5}]{mean}{\mSbTabICR}
\pgfplotstablecreatecol[create col/expr={
    sqrt( ((\thisrow{r1}-\thisrow{mean})^2 + (\thisrow{r2}-\thisrow{mean})^2
         + (\thisrow{r3}-\thisrow{mean})^2 + (\thisrow{r4}-\thisrow{mean})^2
         + (\thisrow{r5}-\thisrow{mean})^2) / 4 )
}]{std}{\mSbTabICR}

\pgfplotstablegetelem{0}{mean}\of\mSaKNNICL
\pgfmathsetmacro\mSaKNNICLFlat{\pgfplotsretval}

\pgfplotstablegetelem{0}{mean}\of\mSaTabICL
\pgfmathsetmacro\mSaTabICLFlat{\pgfplotsretval}

\pgfplotstablegetelem{0}{mean}\of\mSbKNNICL
\pgfmathsetmacro\mSbKNNICLFlat{\pgfplotsretval}

\pgfplotstablegetelem{0}{mean}\of\mSbTabICL
\pgfmathsetmacro\mSbTabICLFlat{\pgfplotsretval}

\begin{figure}[htbp]
\centering
\begin{tikzpicture}
\begin{groupplot}[
    group style={
        group size=2 by 1,
        horizontal sep=2.2cm,
        group name=mSweepPlots,
    },
    width=0.45\textwidth,
    height=0.36\textwidth,
    xmode=log,
    log basis x=2,
    xmin=48, xmax=45000,
    ymin=-0.005, ymax=0.22,
    ytick={0,0.05,0.10,0.15,0.20},
    xlabel={corpus size $m$},
    ylabel={Excess error on $Q$},
    tick label style={font=\small},
    label style={font=\small},
    title style={font=\small},
]

\nextgroupplot[title={$n = 64$, $\alpha_P = 1$ \;(5 seeds)},
    legend to name=mSweepLegend,
    legend columns=4,
    legend style={font=\small, column sep=0.8em, legend cell align=left},
]
\addplot[blue, solid, thick, mark=none]
    coordinates {(48,\mSaTabICLFlat)(45000,\mSaTabICLFlat)};
\addlegendentry{TabPFN-ICL}
\addplot[orange!90!black, solid, thick, mark=*, mark size=2pt]
    table[x=m, y=mean] {\mSaTabICR};
\addlegendentry{TabPFN-ICR}
\addplot[blue, dashed, thick, mark=none]
    coordinates {(48,\mSaKNNICLFlat)(45000,\mSaKNNICLFlat)};
\addlegendentry{$k$-NN-ICL}
\addplot[orange!90!black, dashed, thick, mark=square, mark size=2pt]
    table[x=m, y=mean] {\mSaKNNICR};
\addlegendentry{$k$-NN-ICR}
\draw[gray!70, dotted, very thick]
    (axis cs:512,-0.005) -- (axis cs:512,0.22);
\node[gray!70, font=\scriptsize, anchor=south east]
    at (axis cs:512,0.20) {$n^{3/2}{=}512$};

\nextgroupplot[title={$n = 128$, $\alpha_P = 1$ \;(5 seeds)}, ylabel={}]
\addplot[blue, solid, thick, mark=none]
    coordinates {(48,\mSbTabICLFlat)(45000,\mSbTabICLFlat)};
\addplot[orange!90!black, solid, thick, mark=*, mark size=2pt]
    table[x=m, y=mean] {\mSbTabICR};
\addplot[blue, dashed, thick, mark=none]
    coordinates {(48,\mSbKNNICLFlat)(45000,\mSbKNNICLFlat)};
\addplot[orange!90!black, dashed, thick, mark=square, mark size=2pt]
    table[x=m, y=mean] {\mSbKNNICR};
\draw[gray!70, dotted, very thick]
    (axis cs:1448,-0.005) -- (axis cs:1448,0.22);
\node[gray!70, font=\scriptsize, anchor=south east]
    at (axis cs:1448,0.20) {$n^{3/2}{\approx}1448$};

\end{groupplot}

\node[anchor=north, yshift=-0.8cm]
    at ($(mSweepPlots c1r1.south)!0.5!(mSweepPlots c2r1.south)$)
    {\pgfplotslegendfromname{mSweepLegend}};

\end{tikzpicture}
\caption{Excess error on $Q$ versus corpus size $m$, with $n$ fixed and
$\alpha_P = 1$.  \emph{Color}: ICL (blue) versus ICR (orange).
\emph{Line style}: TabPFN (solid, \textbullet) vs.\ $k$-NN (dashed,
$\square$).  Per $(n, \text{seed})$ we draw one corpus of size $m_{\max}$
and use nested prefixes $\mathcal{D}_m = \mathcal{D}_{m_{\max}}[:m]$; the
ICL demo set is drawn once per $n$ from that corpus and reused across
every $m$.  Consequently ICL is exactly constant along the $m$ axis for
each seed, and the horizontal ICL baselines are the across-seed means of
those per-seed constants.  Both ICR curves collapse sharply around the reference scale and
remain near their asymptotic value thereafter.  The dotted grey line marks the
unit-constant reference $n^{1+1/(2\alpha_P)}=n^{3/2}$ ($512$ for $n=64$,
$\approx 1448$ for $n=128$); the actual saturation threshold is
$C_\star\,n^{1+1/(2\alpha_P)}$ for some unknown constant.
All 5 seeds (42--46).  Setup: 1-D Weierstrass $\eta$-family with
$A = 0.24$, $b = 0.85$, $f = 2$, $J = 12$, oracle
$k_\alpha = \mathrm{round\_odd}(0.30 \cdot n^{2\alpha/(2\alpha+1)})$,
linspace $n_\text{test} = 400$, and joint demonstrations.}
\label{fig:m_sweep}
\end{figure}
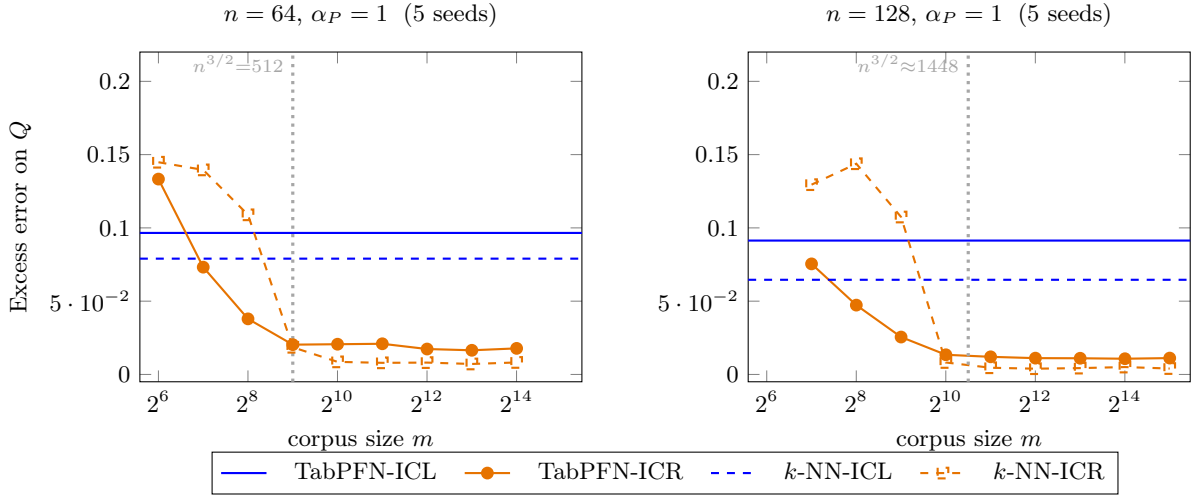

\begin{table}[htbp]
\caption{Corpus-size sweep: excess error on $Q$, mean $\pm$ sample
standard deviation over five seeds.  ICL entries are constant along the
$m$ axis by construction (see Fig.~\ref{fig:m_sweep} caption) and are shown
once per $n$; ICR entries vary per $m$.}
\label{tab:m_sweep_full}
\centering
\scriptsize
\setlength{\tabcolsep}{4pt}
\begin{tabular}{rrcccc}
\toprule
$n$ & $m$ & TabPFN-ICL & TabPFN-ICR & $k$-NN-ICL & $k$-NN-ICR \\
\midrule
\multirow{9}{*}{64}
& 64    & \multirow{9}{*}{\meanstd{\mSaTabICL}{0}} & \meanstd{\mSaTabICR}{0}
        & \multirow{9}{*}{\meanstd{\mSaKNNICL}{0}} & \meanstd{\mSaKNNICR}{0} \\
& 128   & & \meanstd{\mSaTabICR}{1} & & \meanstd{\mSaKNNICR}{1} \\
& 256   & & \meanstd{\mSaTabICR}{2} & & \meanstd{\mSaKNNICR}{2} \\
& 512   & & \meanstd{\mSaTabICR}{3} & & \meanstd{\mSaKNNICR}{3} \\
& 1024  & & \meanstd{\mSaTabICR}{4} & & \meanstd{\mSaKNNICR}{4} \\
& 2048  & & \meanstd{\mSaTabICR}{5} & & \meanstd{\mSaKNNICR}{5} \\
& 4096  & & \meanstd{\mSaTabICR}{6} & & \meanstd{\mSaKNNICR}{6} \\
& 8192  & & \meanstd{\mSaTabICR}{7} & & \meanstd{\mSaKNNICR}{7} \\
& 16384 & & \meanstd{\mSaTabICR}{8} & & \meanstd{\mSaKNNICR}{8} \\
\midrule
\multirow{9}{*}{128}
& 128   & \multirow{9}{*}{\meanstd{\mSbTabICL}{0}} & \meanstd{\mSbTabICR}{0}
        & \multirow{9}{*}{\meanstd{\mSbKNNICL}{0}} & \meanstd{\mSbKNNICR}{0} \\
& 256   & & \meanstd{\mSbTabICR}{1} & & \meanstd{\mSbKNNICR}{1} \\
& 512   & & \meanstd{\mSbTabICR}{2} & & \meanstd{\mSbKNNICR}{2} \\
& 1024  & & \meanstd{\mSbTabICR}{3} & & \meanstd{\mSbKNNICR}{3} \\
& 2048  & & \meanstd{\mSbTabICR}{4} & & \meanstd{\mSbKNNICR}{4} \\
& 4096  & & \meanstd{\mSbTabICR}{5} & & \meanstd{\mSbKNNICR}{5} \\
& 8192  & & \meanstd{\mSbTabICR}{6} & & \meanstd{\mSbKNNICR}{6} \\
& 16384 & & \meanstd{\mSbTabICR}{7} & & \meanstd{\mSbKNNICR}{7} \\
& 32768 & & \meanstd{\mSbTabICR}{8} & & \meanstd{\mSbKNNICR}{8} \\
\bottomrule
\end{tabular}
\end{table}

\newpage

\subsection{Sentiment classification: full setup}

We use the sentiment classification dataset of \citet{diaz2018addressing}, a language benchmark with substantial annotator-induced posterior drift: each sentence is labeled by 4--5 individuals with $\approx 40\%$ label disagreement. The training split has 59{,}235 annotated examples from 1{,}481 annotators (up to 40 per annotator); the test split has 1{,}419 annotated examples from 878 annotators. This five-class classification ($\mc{Y} = \{\text{very negative, somewhat negative, neutral, somewhat positive, very positive}\}$) experiment qualitatively tests the implications of our binary classification theory. Each test annotator is treated as a target distribution $Q$. Sentences are encoded with BAAI/bge-base-en-v1.5~\citep{bge_embedding} and retrieved by Euclidean distance; the in-context learner is Qwen3.5-4B~\citep{qwen3.5}. ICL demonstrations are sampled uniformly without replacement from the annotator's own training data. Errors are averaged over test annotators and over 5 independent retrieval/sampling runs.

\paragraph{Posterior drift depends on the choice of source corpus.}
The same dataset can yield very different source distributions $P$ depending on how the retrieval corpus is assembled, and this choice directly controls the drift relative to each test annotator's $Q$:
\emph{(1) self} --- only the test annotator's own training data ($P=Q$, no drift, but the corpus is small, $\le 40$ examples);
\emph{(2) self $+$ 100 random others} --- mild drift, moderately larger corpus;
\emph{(3) all 1{,}481 annotators} --- the entire training set, severe drift but the largest corpus;
\emph{(4) others only} --- the test annotator deliberately excluded, extreme drift.
This sweep makes explicit a tradeoff that practitioners face whenever labels come from heterogeneous sources: a tightly curated corpus minimizes posterior drift at the cost of size, while pooling all available data maximizes corpus size at the cost of drift. Notably, much of the existing ICR literature treats the entire training set as the retrieval corpus by default, with no explicit accounting for posterior drift---which, as our results show, is exactly the regime where ICR can fail catastrophically.

\paragraph{Results.}
ICR strongly outperforms ICL when $P=Q$ and remains preferable under mild posterior drift for $n\le 16$. Under severe drift the ICR error stays near $50\%$ while ICL keeps improving with $n$; under extreme drift, ICR error rises to $\approx 63\%$, well above the ICL baseline. This finding not only verifies our theory, but also has a practical takeaway: na\"ively pooling all available data is not a safe default---a smaller, well-aligned corpus can dominate a larger, drifted one.

\paragraph{Effect of model scale.}
\Cref{fig:exp-sentiment-scale} repeats the experiment with the stronger Qwen3.5-27B model alongside Qwen3.5-4B. The 27B model is markedly better in the few-shot regime where the corpus is aligned with the target (e.g., $0.12$ vs.\ $0.26$ at $n=4$ under $P=Q$), yet its error under severe and extreme drift is essentially unchanged from the 4B model: model capacity does not remove the irreducible drift-induced bias, consistent with Proposition~\ref{prop:icr-lower-bound-simple}.

\begin{figure}[htbp]
\centering
\begin{tikzpicture}
\begin{groupplot}[
    group style={
        group size=2 by 1,
        horizontal sep=1.6cm,
        group name=sentpanels,
    },
    width=0.42\textwidth, height=0.3\textwidth,
    xmode=log, log basis x=2,
    xtick={4,8,16,32}, xticklabels={4,8,16,32},
    ymin=0, ymax=0.7,
    ytick={0,0.2,0.4,0.6},
    xlabel={$n$ (demonstrations)},
    ylabel={Test error},
    grid=major, grid style={gray!25},
    tick label style={font=\small}, label style={font=\small},
]
\nextgroupplot[title={Qwen3.5-4B}, title style={font=\small},
    legend to name=appSentLegend,
    legend columns=3,
    legend style={font=\small, column sep=0.8em, legend cell align=left, draw=none},
]
\addplot[blue, solid, thick, mark=*, mark size=1.8pt]
    coordinates {(4,0.5383) (8,0.4806) (16,0.3480) (32,0.1168)};
\addlegendentry{ICL}
\addplot[orange!50!yellow, solid, thick, mark=square*, mark size=1.8pt]
    coordinates {(4,0.2630) (8,0.1693) (16,0.0985) (32,0.0724)};
\addlegendentry{ICR, no drift ($P=Q$)}
\addplot[orange!90!black, solid, thick, mark=triangle*, mark size=2.2pt]
    coordinates {(4,0.2653) (8,0.2250) (16,0.1622) (32,0.1420)};
\addlegendentry{ICR, mild drift}
\addplot[red!80!black, solid, thick, mark=diamond*, mark size=2.2pt]
    coordinates {(4,0.5347) (8,0.5039) (16,0.4910) (32,0.4800)};
\addlegendentry{ICR, severe drift}
\addplot[red!40!black, solid, thick, mark=pentagon*, mark size=2.2pt]
    coordinates {(4,0.6334) (8,0.6277) (16,0.6149) (32,0.6190)};
\addlegendentry{ICR, extreme drift}
\nextgroupplot[title={Qwen3.5-27B}, title style={font=\small}, ylabel={}]
\addplot[blue, solid, thick, mark=*, mark size=1.8pt]
    coordinates {(4,0.5140) (8,0.4593) (16,0.3414) (32,0.1319)};
\addplot[orange!50!yellow, solid, thick, mark=square*, mark size=1.8pt]
    coordinates {(4,0.1179) (8,0.0939) (16,0.0833) (32,0.0974)};
\addplot[orange!90!black, solid, thick, mark=triangle*, mark size=2.2pt]
    coordinates {(4,0.1647) (8,0.1543) (16,0.1548) (32,0.1632)};
\addplot[red!80!black, solid, thick, mark=diamond*, mark size=2.2pt]
    coordinates {(4,0.5223) (8,0.5067) (16,0.5058) (32,0.5022)};
\addplot[red!40!black, solid, thick, mark=pentagon*, mark size=2.2pt]
    coordinates {(4,0.6189) (8,0.6198) (16,0.6064) (32,0.6121)};
\end{groupplot}
\node[anchor=north, yshift=-0.9cm]
    at ($(sentpanels c1r1.south)!0.5!(sentpanels c2r1.south)$)
    {\pgfplotslegendfromname{appSentLegend}};
\end{tikzpicture}
\caption{Test error vs.\ number of demonstrations $n$ on the age-biased sentiment dataset (878 test annotators, mean over 5 seeds), for Qwen3.5-4B (left) and Qwen3.5-27B (right). \emph{Blue}: ICL with demonstrations sampled uniformly without replacement from the annotator's own data. \emph{Yellow $\to$ red}: ICR under four corpus designs with increasing posterior drift. The stronger model improves ICR only where the corpus is aligned with the target; under severe and extreme drift its error is essentially unchanged---the drift-induced bias is irreducible.
Per-cell mean values are in \Cref{tab:exp3-sentiment-4b,tab:exp3-sentiment-27b}.
}
\label{fig:exp-sentiment-scale}
\end{figure}
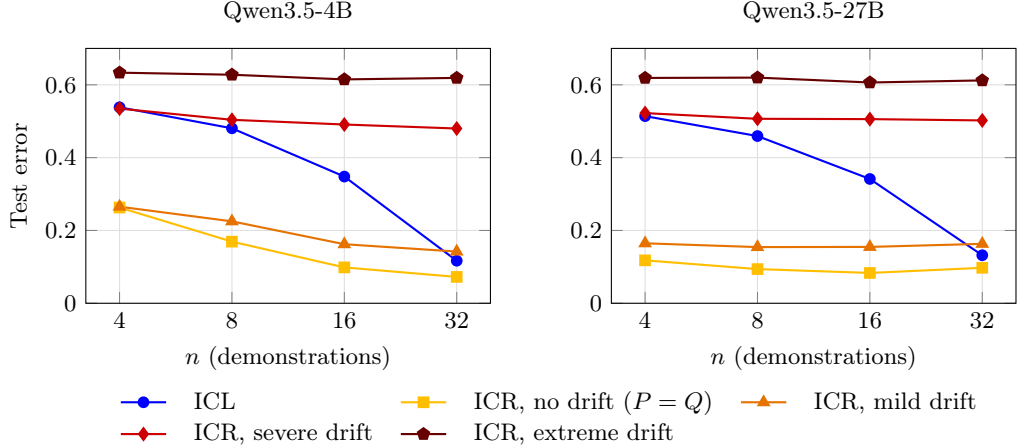

\paragraph{Reproducibility.}
\emph{Data.} We use the age-biased sentiment re-annotation of Sentiment140 by \citet{diaz2018addressing}: 59{,}235 training rows (1{,}481 annotators, $\le 40$ each) and 1{,}419 test rows (878 annotators, 1--4 each). Each rating is an ordinal label in $\{1,\dots,5\}$, verbalized as \emph{Very negative}, \emph{Somewhat negative}, \emph{Neutral}, \emph{Somewhat positive}, \emph{Very positive}; test error is $1-{}$accuracy (exact match). Every test annotator with at least $5$ of their own training examples defines a target $Q$; results are averaged over all 878 test annotators and 5 seeds (42--46).

{\sloppy \emph{In-context learner.} Qwen3.5-4B and Qwen3.5-27B with greedy decoding (\texttt{do\_sample=False}), \texttt{max\_new\_tokens}${}=10$, and thinking disabled (\texttt{enable\_thinking=False}), using the system prompt: ``You are a sentiment classifier. Output only the label (Very negative, Somewhat negative, Neutral, Somewhat positive, Very positive) and nothing else. Do not explain, do not add punctuation, do not repeat the input.'' Each demonstration is formatted as \texttt{Input: \{sentence\}} $\backslash$n \texttt{Output: \{label\}}; demonstrations are concatenated, followed by the query \texttt{Input: \{sentence\}} $\backslash$n \texttt{Output:}. The completion is mapped to a rating by case-insensitive substring matching, longest label first.\par}

\emph{Retrieval.} Sentences are embedded with BAAI/bge-base-en-v1.5~\citep{bge_embedding} (unit-normalized); the $n$ nearest neighbors are retrieved by Euclidean distance (equivalently, cosine distance on unit vectors). ICL samples $n$ demonstrations uniformly without replacement from the annotator's own training data, fixed across that annotator's queries; ICR retrieves $n$ per-query nearest neighbors from the source corpus.

\paragraph{Per-cell mean (sentiment).}
We report means (averaged across 878 test annotators and 5 independent runs):

\pgfplotstableread[col sep=space]{
n     r1      r2      r3      r4      r5
4     0.5486  0.5371  0.5312  0.5326  0.5421
8     0.4755  0.4876  0.4838  0.4866  0.4693
16    0.3423  0.3495  0.3615  0.3337  0.3528
32    0.1190  0.1198  0.1122  0.1196  0.1132
}\ICLFourB
\pgfplotstablecreatecol[create col/expr={(\thisrow{r1}+\thisrow{r2}+\thisrow{r3}+\thisrow{r4}+\thisrow{r5})/5}]{mean}{\ICLFourB}

\pgfplotstableread[col sep=space]{
n     r1      r2      r3      r4      r5
4     0.2626  0.2632  0.2632  0.2626  0.2632
8     0.1720  0.1686  0.1686  0.1686  0.1686
16    0.1020  0.0975  0.0975  0.0980  0.0975
32    0.0735  0.0723  0.0723  0.0718  0.0723
}\ICRSelfFourB
\pgfplotstablecreatecol[create col/expr={(\thisrow{r1}+\thisrow{r2}+\thisrow{r3}+\thisrow{r4}+\thisrow{r5})/5}]{mean}{\ICRSelfFourB}

\pgfplotstableread[col sep=space]{
n     r1      r2      r3      r4      r5
4     0.2681  0.2756  0.2566  0.2561  0.2700
8     0.2204  0.2340  0.2128  0.2240  0.2340
16    0.1685  0.1636  0.1493  0.1520  0.1776
32    0.1471  0.1463  0.1276  0.1353  0.1538
}\ICRMildFourB
\pgfplotstablecreatecol[create col/expr={(\thisrow{r1}+\thisrow{r2}+\thisrow{r3}+\thisrow{r4}+\thisrow{r5})/5}]{mean}{\ICRMildFourB}

\pgfplotstableread[col sep=space]{
n     r1      r2      r3      r4      r5
4     0.5372  0.5322  0.5299  0.5393  0.5349
8     0.5065  0.5026  0.5033  0.5009  0.5064
16    0.4905  0.4953  0.4921  0.4885  0.4888
32    0.4810  0.4790  0.4776  0.4805  0.4821
}\ICRSevereFourB
\pgfplotstablecreatecol[create col/expr={(\thisrow{r1}+\thisrow{r2}+\thisrow{r3}+\thisrow{r4}+\thisrow{r5})/5}]{mean}{\ICRSevereFourB}

\pgfplotstableread[col sep=space]{
n     r1      r2      r3      r4      r5
4     0.6382  0.6277  0.6357  0.6344  0.6308
8     0.6315  0.6194  0.6300  0.6352  0.6223
16    0.6100  0.6105  0.6122  0.6218  0.6201
32    0.6290  0.6112  0.6107  0.6253  0.6187
}\ICRExtremeFourB
\pgfplotstablecreatecol[create col/expr={(\thisrow{r1}+\thisrow{r2}+\thisrow{r3}+\thisrow{r4}+\thisrow{r5})/5}]{mean}{\ICRExtremeFourB}

\addstd{\ICLFourB}\addstd{\ICRSelfFourB}\addstd{\ICRMildFourB}\addstd{\ICRSevereFourB}\addstd{\ICRExtremeFourB}

\begin{figure}[htbp]
\centering
\begin{tikzpicture}
\begin{axis}[
    width=0.42\textwidth,
    height=0.4\textwidth,
    xmode=log,
    log basis x=2,
    xtick={4,8,16,32},
    xticklabels={4,8,16,32},
    ymin=0, ymax=0.8,
    ytick={0,0.2,0.4,0.6,0.8},
    xlabel={$n$ (demonstrations)},
    ylabel={Test error},
    grid=major,
    grid style={gray!25},
    tick label style={font=\small},
    label style={font=\small},
    legend style={font=\small, at={(1.02,0.5)}, anchor=west, draw=none},
    legend cell align=left,
]

\addplot[blue, solid, thick, mark=*, mark size=2pt]
    table[x=n, y=mean] {\ICLFourB};
\addlegendentry{Qwen-ICL (baseline)}

\addplot[orange!50!yellow, solid, thick, mark=square*, mark size=2pt]
    table[x=n, y=mean] {\ICRSelfFourB};
\addlegendentry{Qwen-ICR, no drift ($P=Q$, self)}

\addplot[orange!90!black, solid, thick, mark=triangle*, mark size=2.5pt]
    table[x=n, y=mean] {\ICRMildFourB};
\addlegendentry{Qwen-ICR, mild drift}

\addplot[red!80!black, solid, thick, mark=diamond*, mark size=2.5pt]
    table[x=n, y=mean] {\ICRSevereFourB};
\addlegendentry{Qwen-ICR, severe drift}

\addplot[red!40!black, solid, thick, mark=pentagon*, mark size=2.5pt]
    table[x=n, y=mean] {\ICRExtremeFourB};
\addlegendentry{Qwen-ICR, extreme drift}

\end{axis}
\end{tikzpicture}
\caption{Test error vs.\ number of demonstrations $n$ on the age-biased
sentiment dataset (878 test annotators, Qwen3.5-4B). Each point is the
mean over $5$ independent runs (seeds 42--46). \emph{Blue}: Qwen-ICL with
demonstrations sampled uniformly without replacement from the annotator's own training data.
\emph{Yellow $\to$ red}: Qwen-ICR under four levels of posterior drift,
as a result of expanding the source corpus: no drift ($P=Q$, self only),
mild (self $+$ 100 random others), severe (all 1{,}481 annotators),
and extreme (all other annotators, self excluded). ICR sharply outperforms
ICL when the source corpus is well-aligned with the target distribution
and degrades as posterior drift increases. Uncertainty estimates are
reported in \Cref{tab:exp3-sentiment-4b}.}
\label{fig:icl_icr_sentiment_4b}
\end{figure}
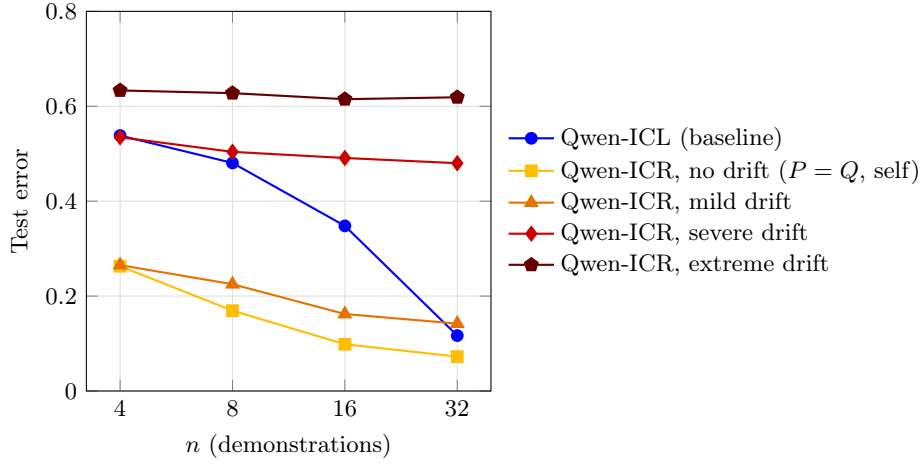

\begin{table}[htbp]
\caption{Sentiment classification with Qwen3.5-4B: test error (mean $\pm$ std over $5$ seeds, $878$ test annotators) vs.\ number of demonstrations $n$ across the four corpus designs.}
\label{tab:exp3-sentiment-4b}
\centering\small\setlength{\tabcolsep}{4pt}
\begin{tabular}{rccccc}
\toprule
$n$ & Qwen-ICL & ICR ($P=Q$) & ICR (mild) & ICR (severe) & ICR (extreme) \\
\midrule
4   & \meanstd{\ICLFourB}{0} & \meanstd{\ICRSelfFourB}{0} & \meanstd{\ICRMildFourB}{0} & \meanstd{\ICRSevereFourB}{0} & \meanstd{\ICRExtremeFourB}{0} \\
8   & \meanstd{\ICLFourB}{1} & \meanstd{\ICRSelfFourB}{1} & \meanstd{\ICRMildFourB}{1} & \meanstd{\ICRSevereFourB}{1} & \meanstd{\ICRExtremeFourB}{1} \\
16  & \meanstd{\ICLFourB}{2} & \meanstd{\ICRSelfFourB}{2} & \meanstd{\ICRMildFourB}{2} & \meanstd{\ICRSevereFourB}{2} & \meanstd{\ICRExtremeFourB}{2} \\
32  & \meanstd{\ICLFourB}{3} & \meanstd{\ICRSelfFourB}{3} & \meanstd{\ICRMildFourB}{3} & \meanstd{\ICRSevereFourB}{3} & \meanstd{\ICRExtremeFourB}{3} \\
\bottomrule
\end{tabular}
\end{table}

\pgfplotstableread[col sep=space]{
n     r1      r2      r3      r4      r5
4     0.5072  0.5108  0.5185  0.5166  0.5169
8     0.4656  0.4513  0.4599  0.4577  0.4621
16    0.3430  0.3434  0.3372  0.3339  0.3497
32    0.1346  0.1340  0.1281  0.1399  0.1227
}\ICLBigB
\pgfplotstablecreatecol[create col/expr={(\thisrow{r1}+\thisrow{r2}+\thisrow{r3}+\thisrow{r4}+\thisrow{r5})/5}]{mean}{\ICLBigB}

\pgfplotstableread[col sep=space]{
n     r1      r2      r3      r4      r5
4     0.1169  0.1181  0.1181  0.1181  0.1181
8     0.0939  0.0939  0.0939  0.0939  0.0939
16    0.0835  0.0831  0.0835  0.0835  0.0831
32    0.0978  0.0974  0.0966  0.0979  0.0972
}\ICRSelfBigB
\pgfplotstablecreatecol[create col/expr={(\thisrow{r1}+\thisrow{r2}+\thisrow{r3}+\thisrow{r4}+\thisrow{r5})/5}]{mean}{\ICRSelfBigB}

\pgfplotstableread[col sep=space]{
n     r1      r2      r3      r4      r5
4     0.1669  0.1641  0.1600  0.1644  0.1682
8     0.1505  0.1567  0.1460  0.1556  0.1627
16    0.1549  0.1550  0.1554  0.1479  0.1608
32    0.1661  0.1661  0.1547  0.1665  0.1628
}\ICRMildBigB
\pgfplotstablecreatecol[create col/expr={(\thisrow{r1}+\thisrow{r2}+\thisrow{r3}+\thisrow{r4}+\thisrow{r5})/5}]{mean}{\ICRMildBigB}

\pgfplotstableread[col sep=space]{
n     r1      r2      r3      r4      r5
4     0.5248  0.5184  0.5345  0.5215  0.5124
8     0.5027  0.5038  0.5074  0.5088  0.5109
16    0.5079  0.5049  0.5030  0.5092  0.5041
32    0.4974  0.5032  0.5019  0.5065  0.5018
}\ICRSevereBigB
\pgfplotstablecreatecol[create col/expr={(\thisrow{r1}+\thisrow{r2}+\thisrow{r3}+\thisrow{r4}+\thisrow{r5})/5}]{mean}{\ICRSevereBigB}

\pgfplotstableread[col sep=space]{
n     r1      r2      r3      r4      r5
4     0.6241  0.6223  0.6218  0.6152  0.6111
8     0.6158  0.6232  0.6239  0.6227  0.6135
16    0.5987  0.5979  0.6159  0.6117  0.6076
32    0.6096  0.6019  0.6197  0.6185  0.6109
}\ICRExtremeBigB
\pgfplotstablecreatecol[create col/expr={(\thisrow{r1}+\thisrow{r2}+\thisrow{r3}+\thisrow{r4}+\thisrow{r5})/5}]{mean}{\ICRExtremeBigB}

\addstd{\ICLBigB}\addstd{\ICRSelfBigB}\addstd{\ICRMildBigB}\addstd{\ICRSevereBigB}\addstd{\ICRExtremeBigB}

\begin{figure}[htbp]
\centering
\begin{tikzpicture}
\begin{axis}[
    width=0.42\textwidth,
    height=0.4\textwidth,
    xmode=log,
    log basis x=2,
    xtick={4,8,16,32},
    xticklabels={4,8,16,32},
    ymin=0, ymax=0.8,
    ytick={0,0.2,0.4,0.6,0.8},
    xlabel={$n$ (demonstrations)},
    ylabel={Test error},
    grid=major,
    grid style={gray!25},
    tick label style={font=\small},
    label style={font=\small},
    legend style={font=\small, at={(1.02,0.5)}, anchor=west, draw=none},
    legend cell align=left,
]

\addplot[blue, solid, thick, mark=*, mark size=2pt]
    table[x=n, y=mean] {\ICLBigB};
\addlegendentry{Qwen-ICL (baseline)}

\addplot[orange!50!yellow, solid, thick, mark=square*, mark size=2pt]
    table[x=n, y=mean] {\ICRSelfBigB};
\addlegendentry{Qwen-ICR, no drift ($P=Q$, self)}

\addplot[orange!90!black, solid, thick, mark=triangle*, mark size=2.5pt]
    table[x=n, y=mean] {\ICRMildBigB};
\addlegendentry{Qwen-ICR, mild drift}

\addplot[red!80!black, solid, thick, mark=diamond*, mark size=2.5pt]
    table[x=n, y=mean] {\ICRSevereBigB};
\addlegendentry{Qwen-ICR, severe drift}

\addplot[red!40!black, solid, thick, mark=pentagon*, mark size=2.5pt]
    table[x=n, y=mean] {\ICRExtremeBigB};
\addlegendentry{Qwen-ICR, extreme drift}

\end{axis}
\end{tikzpicture}
\caption{Test error vs.\ number of demonstrations $n$ on the age-biased
sentiment dataset (878 test annotators, Qwen3.5-27B). Each point is the
mean over $5$ independent runs (seeds 42--46). \emph{Blue}: Qwen-ICL with
demonstrations sampled uniformly without replacement from the annotator's own training data.
\emph{Yellow $\to$ red}: Qwen-ICR under four levels of posterior drift,
as a result of expanding the source corpus: no drift ($P=Q$, self only),
mild (self $+$ 100 random others), severe (all 1{,}481 annotators),
and extreme (all other annotators, self excluded). Compared to Qwen3.5-4B
(\Cref{fig:icl_icr_sentiment_4b}), the stronger model lowers the ICR
error under no/mild drift but does not remove the irreducible bias under
severe or extreme drift, matching the theory. Uncertainty estimates are
reported in \Cref{tab:exp3-sentiment-27b}.}
\label{fig:icl_icr_sentiment_27b}
\end{figure}
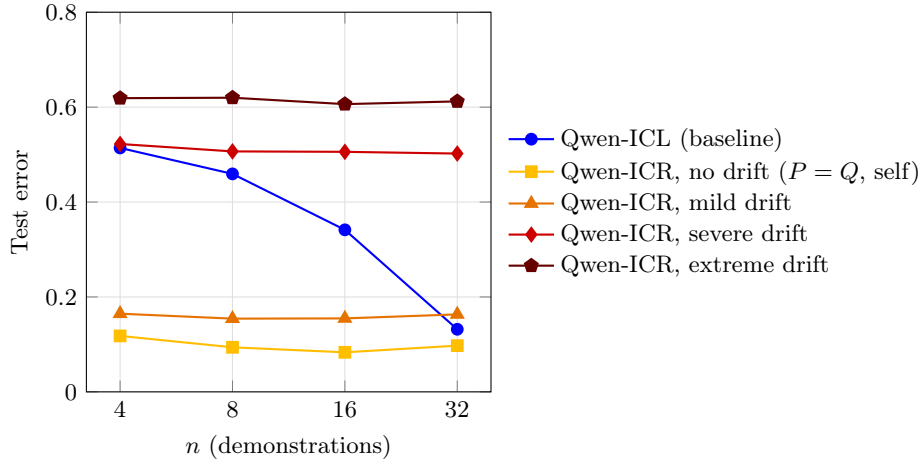

\begin{table}[htbp]
\caption{Sentiment classification with Qwen3.5-27B: test error (mean $\pm$ std over $5$ seeds, $878$ test annotators) vs.\ number of demonstrations $n$ across the four corpus designs.}
\label{tab:exp3-sentiment-27b}
\centering\small\setlength{\tabcolsep}{4pt}
\begin{tabular}{rccccc}
\toprule
$n$ & Qwen-ICL & ICR ($P=Q$) & ICR (mild) & ICR (severe) & ICR (extreme) \\
\midrule
4   & \meanstd{\ICLBigB}{0} & \meanstd{\ICRSelfBigB}{0} & \meanstd{\ICRMildBigB}{0} & \meanstd{\ICRSevereBigB}{0} & \meanstd{\ICRExtremeBigB}{0} \\
8   & \meanstd{\ICLBigB}{1} & \meanstd{\ICRSelfBigB}{1} & \meanstd{\ICRMildBigB}{1} & \meanstd{\ICRSevereBigB}{1} & \meanstd{\ICRExtremeBigB}{1} \\
16  & \meanstd{\ICLBigB}{2} & \meanstd{\ICRSelfBigB}{2} & \meanstd{\ICRMildBigB}{2} & \meanstd{\ICRSevereBigB}{2} & \meanstd{\ICRExtremeBigB}{2} \\
32  & \meanstd{\ICLBigB}{3} & \meanstd{\ICRSelfBigB}{3} & \meanstd{\ICRMildBigB}{3} & \meanstd{\ICRSevereBigB}{3} & \meanstd{\ICRExtremeBigB}{3} \\
\bottomrule
\end{tabular}
\end{table}

\end{document}